\documentclass[10pt]{article}

\usepackage[preprint]{tmlr}

\usepackage[utf8]{inputenc}
\usepackage[T1]{fontenc}
\usepackage{microtype}
\usepackage{graphicx}
\usepackage{subfigure}
\usepackage{placeins}
\usepackage{booktabs}
\usepackage{amsmath}
\usepackage{amssymb}
\usepackage{mathtools}
\usepackage{amsthm}
\usepackage{thmtools}
\usepackage{thm-restate}
\usepackage[ruled,vlined,linesnumbered]{algorithm2e}
\usepackage{xstring}
\usepackage{tikz}
\usepackage{bm}
\usepackage{nicefrac}
\usepackage{multirow}
\usepackage{makecell}
\usepackage{wrapfig}
\usepackage{xcolor}

\definecolor{Blue}{RGB}{23, 78, 166}
\definecolor{Red}{RGB}{165, 14, 14}
\definecolor{Orange}{RGB}{227, 116, 0}
\definecolor{Green}{RGB}{13, 101, 45}
\definecolor{MediumBlue}{RGB}{66, 103, 210}
\definecolor{MediumRed}{RGB}{234, 67, 53}
\definecolor{Yellow}{RGB}{251, 188, 4}
\definecolor{MediumGreen}{RGB}{52, 168, 83}
\definecolor{LightBlue}{RGB}{210, 227, 252}
\definecolor{LightRed}{RGB}{250, 210, 207}
\definecolor{LightYellow}{RGB}{254, 239, 195}
\definecolor{LightGreen}{RGB}{206, 234, 214}
\definecolor{LightGrey}{RGB}{241, 243, 244}
\definecolor{Grey}{RGB}{154, 160, 166}
\definecolor{Black}{RGB}{32, 33, 36}

\usepackage{hyperref}
\usepackage{url}
\hypersetup{
  colorlinks=true,
  linkcolor=Green,
  urlcolor=Orange,
  citecolor=Blue,
  breaklinks,
  hypertexnames=false
}
\usepackage[capitalize,noabbrev]{cleveref}

\newcommand{\method}{DIDA}
\newcommand{\methodfull}{Diffusion-Inspired Dual Annealing}

\DeclareMathOperator*{\argmin}{arg\,min}

\DeclareTextFontCommand{\boldemph}{\bfseries\itshape}

\declaretheorem[numberwithin=section]{theorem}

\declaretheorem[numberlike=theorem]{proposition}
\declaretheorem[numberlike=theorem]{assumption}
\declaretheorem[numberlike=theorem]{lemma}
\declaretheorem[numberlike=theorem, style=remark]{remark}

\crefname{theorem}{Theorem}{Theorems}
\crefname{definition}{Definition}{Definitions}
\crefname{corollary}{Corollary}{Corollaries}
\crefname{proposition}{Proposition}{Propositions}
\crefname{assumption}{Assumption}{Assumptions}
\crefname{lemma}{Lemma}{Lemmas}
\crefname{remark}{Remark}{Remarks}
\crefname{proof}{proof}{proofs}
\crefname{equation}{Equation}{Equations}

\newcommand{\raisedtarget}[1]{%
  \raisebox{\fontcharht\font`P}[0pt][0pt]{\hypertarget{#1}{}}%
}

\newcommand{\SG}{\mathcal{SG}}
\NewDocumentCommand{\BallTau}{o}{%
  \IfValueTF{#1}%
  {B^{#1}_{\tau}}%
  {B_{\tau}}%
}
\newcommand{\BallTauComp}[1]{B_{\tau}^c}
\newcommand{\Pin}[1]{P_{\text{in}}}
\newcommand{\Ppost}[1]{P_{0|#1}}
\newcommand{\Pout}[1]{P_{\text{out}}}
\newcommand{\PbtIn}[2]{P_{\text{in}[#1|#2]}}
\newcommand{\PbtOut}[2]{P_{\text{out}[#1|#2]}}

\newcommand{\EbtOutnorm}[2]{\mathbb{E}_{[#1|#2]}[\|y - x^*\| \mid x_t, y \notin \BallTau{}]}
\newcommand{\EbtOutTnorm}[2]{\mathbb{E}_{[#1|#2]}[\|y - x^*\|^2 \mid x_t, y \notin \BallTau{}]}

\newcommand{\BtauDist}[1]{\min_{y \notin \BallTau} \|y - #1\|}
\newcommand{\BinDist}[1]{\max_{y \in \BallTau} \|y - #1\|}

\newcommand{\momout}{M_{\mathrm{out}}}
\newcommand{\Expec}[1]{\mathbb{E}\left[#1\right]}
\newcommand{\Var}[1]{\mathrm{Var}\left[#1\right]}
\newcommand{\CovbtAll}[2]{\mathrm{Cov}_{[#1|#2]}[y \mid x_t]}
\newcommand{\CovbtIn}[2]{\mathrm{Cov}_{[#1|#2]}[y \mid x_t, y \in \BallTau{}]}
\newcommand{\CovbtOut}[2]{\mathrm{Cov}_{[#1|#2]}[y \mid x_t, y \notin \BallTau{}]}

\newcommand{\pzeromax}{\tilde{p}_0^{\max}}
\newcommand{\pzeromin}{\tilde{p}_0^{\min}}
\newcommand{\SGRad}[2]{\mathcal{SG}\left(r|#1, #2\right)}
\NewDocumentCommand{\poptrpost}{m o o}{%
\IfValueTF{#3}{%
\tilde{p}^{R_{#3}}_{#1|#2}%
}{%
\IfValueTF{#2}{%
\tilde{p}^R_{#1|#2}%
}{%
\tilde{p}^R_{#1}%
}%
}%
}

\NewDocumentCommand{\Poptpost}{m o}{%
  \IfValueTF{#2}%
  {\tilde{p}_{#1|#2}}%
  {\tilde{p}_{#1}}%
}

\title{Global Convergence of Sampling-Based Nonconvex Optimization through Diffusion-Style Smoothing}

\author{
\name Zeji Yi\textsuperscript{*} \email zejiy@andrew.cmu.edu \\
\name Chaoyi Pan\textsuperscript{*} \email chaoyip@andrew.cmu.edu \\
\name Guanya Shi \email guanyas@andrew.cmu.edu \\
\name Guannan Qu \email gqu@andrew.cmu.edu \\
\addr Carnegie Mellon University \\
\addr \textsuperscript{*}Equal contribution.
}

\begin{document}

\maketitle

\begin{abstract}
    Sampling-based optimization (SBO), like cross-entropy method and evolutionary algorithms, has achieved many successes in solving non-convex problems without gradients, yet its convergence is poorly understood.
    In this paper, we establish a non-asymptotic convergence analysis for SBO through the lens of \emph{smoothing}. Specifically, we recast SBO as gradient descent on a smoothed objective, mirroring noise‑conditioned score ascent in diffusion models.
    Our first contribution is a landscape analysis of the smoothed objective, demonstrating how smoothing helps escape local minima and uncovering a fundamental \emph{coverage–optimality trade-off}: smoothing renders the landscape more benign by enlarging the locally convex region around the global minimizer, but at the cost of introducing an optimality gap.
    Building on this insight, we establish non-asymptotic convergence guarantees for SBO algorithms to a neighborhood of the global minimizer.
    Furthermore, we propose an annealed SBO algorithm, \underline{D}iffusion-\underline{I}nspired \underline{D}ual-\underline{A}nnealing (DIDA), which is provably convergent to the global optimum.
    We conduct extensive numerical experiments to verify our landscape results and also demonstrate the compelling performance of DIDA compared to other gradient-free optimization methods. Lastly, we discuss implications of our results for diffusion models.
\end{abstract}

\section{Introduction}
\begin{wrapfigure}{r}{0.4\textwidth}
    \centering
    \vspace{-1em}
    \includegraphics[width=0.4\textwidth]{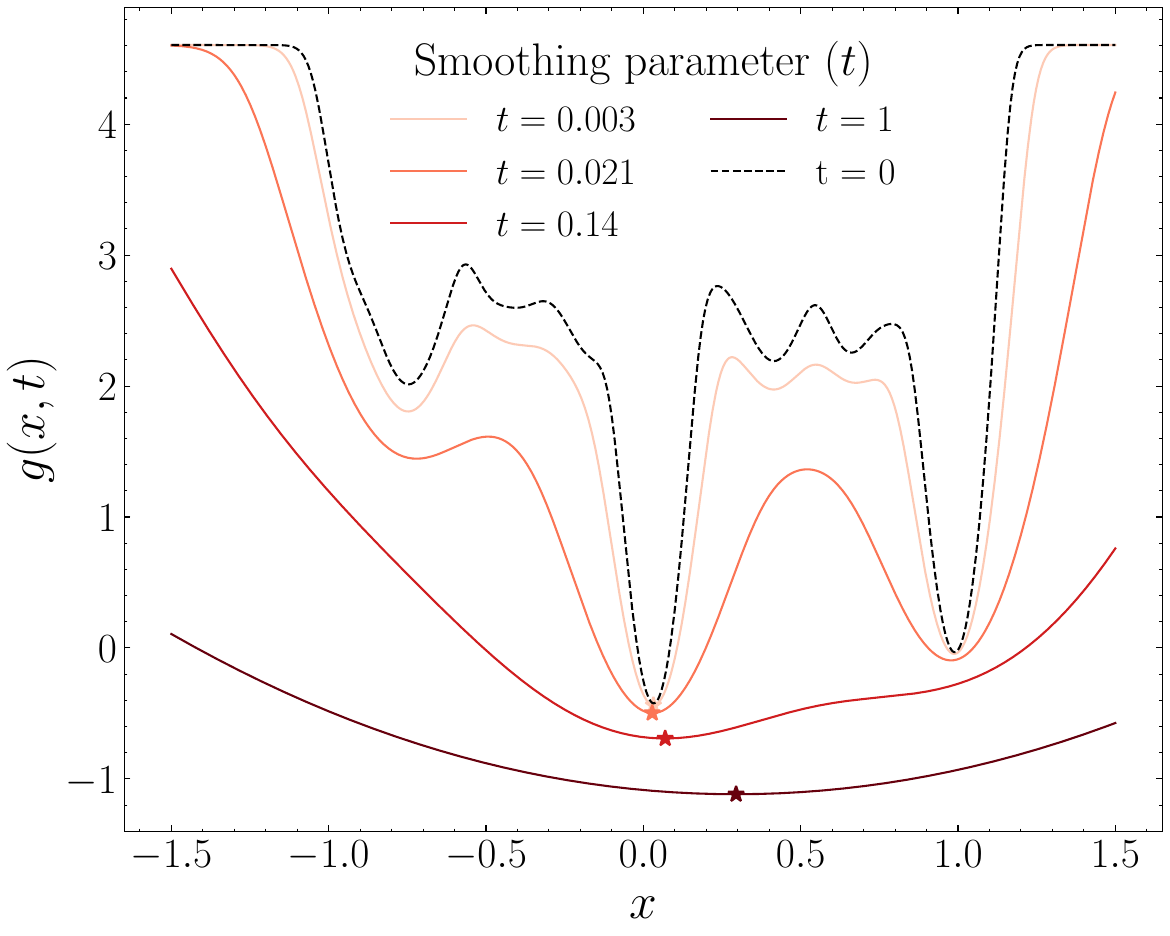}
    \caption{Illustration of the smoothing over different $t$ on the GMM landscape. As $t$ increases, the landscape is smoothed out and the optimum point $x_t^*$ is shifted away from global optimum $x^*$.}
    \label{fig:landscape_tradeoff}
    \vspace{-1em}
\end{wrapfigure}
Many real-world optimization problems are highly nonconvex or even discontinuous, e.g., in optimal control problems in contact-rich robotics~\citep{graesdalTightConvexRelaxations2024,liDROPDexterousReorientation2024}, computer vision~\citep{broxLargeDisplacementOptical2011,hrubyLearningSolveHard2021} and machine learning~\citep{bengioLearningDeepArchitectures,jainNonconvexOptimizationMachine2017,chaudhariEntropySGDBiasingGradient2017,gargianiTransferringOptimalityData2019}. 
While significant research in optimization has yielded methods like interior-point algorithms~\citep{freundInteriorPointMethods2000}, sequential convex programming~\citep{dinhLocalConvergenceSequential2010}, and sum-of-squares techniques~\citep{powellMethodMinimizingSum1965} to tackle nonconvexity, they may be computationally intensive, only guarantee convergence to local minima, or struggle with non-smoothness or discontinuities.

Recently, Sampling-based Optimization (SBO) \cite{CrossentropyMethodPower,maSamplingCanBe2019,hansenCMAEvolutionStrategy2023,williamsAggressiveDrivingModel2016} has gained considerable popularity as a promising alternative. Their appeal stems from the ability to handle highly nonconvex or even discontinuous objective functions effectively without requiring explicit gradients.
Furthermore, they evaluate many candidate solutions in parallel, allowing for massive parallelization on Graphics Processing Units (GPUs).
As a result, SBO techniques, such as Model Predictive Path Integral Control (MPPI)~\citep{williamsInformationTheoreticModel2017} and cross-entropy method (CEM)~\citep{CrossentropyMethodPower},
have found successful applications in path planning~\citep{liDROPDexterousReorientation2024,panModelBasedDiffusionTrajectory2024}, image processing~\citep{josephCrossEntropyBased2018} and black-box optimization~\citep{deboerTutorialCrossEntropyMethod2005}.

Despite this growing empirical success and adoption, several critical gaps remain in the understanding and application of the SBOC. The theoretical side, particularly concerning convergence properties (e.g., guarantees of convergence to global or local optima, convergence rates), are still underdeveloped compared to more traditional optimization methods. Additionally, there is a lack of a systematic framework for designing SBOC algorithms, including principled methods for hyperparameter tuning, adaptive sampling strategies, and variance reduction.
In light of the gap, we answer the following:

\emph{
    Why can SBO effectively find global optima in highly non-convex landscapes?
    How can hyperparameters be designed to achieve optimal convergence?
} 

To answer this question, we leverage a recent discovery: many SBO methods managed to find the optimal $x^*$ of the objective function not by optimizing directly on $f(x)$, but rather a smoothed version $g(x;t)$ (defined in \cref{eq:homotopy_dist}). The parameter $t$ controls this smoothing and corresponds to the sampling variance in SBO.
When $t = 0$, $g(x;t) = f(x)$ and when $t$ increases, $g(x;t)$ becomes smoothed as demonstrated in~\cref{fig:landscape_tradeoff}. Therefore, the convergence of SBO boils down to the properties of the function $g(x;t)$.

\textbf{Contribution.}
Our first contribution is that we show the landscape of $g(x;t)$ demonstrates a \emph{``coverage-optimality''} tradeoff: suppose the global optimizer of $x^*$ lies in a locally convex region of $f(x)$, then as $t$ increases, the convex region expands, covering a larger region that allows for easier convergence; in the mean while, a larger $t$ also will increase the optimality gap as the optimizer of $g(x;t)$ shifts away from the true optimizer $x^*$.
Based on this landscape property of $g(x;t)$, we show for the first time the global convergence of a class of SBO algorithms to a neighborhood of the global minimizer.
Lastly, the tradeoff also suggests an annealing strategy for the parameter $t$, based on which we propose a new algorithm: \method{} that optimally schedules the smoothing parameter and achieve global convergence to the exact global minimizer. In experiments, we verify the coverage-optimality tradeoff, and also show that the proposed \method{} algorithm achieves state-of-the-art over benchmarks.
Lastly, we show that the SBO has deep connections to the ODE method in diffusion model. Specifically, our results have implications for guidance's role in improving sample quality.

\section{Related Work}
\begin{wrapfigure}{r}{0.6\textwidth}
  \vspace{-1.5em}
  \begin{minipage}{\linewidth}
    \begin{algorithm}[H]
      \caption{General Sampling-Based Optimization Algorithm}
      \label{alg:general_sampling_optimization}
      \KwIn{initial guess $x_0$, sampling parameters $\theta$, sample size $N$, total iteration $K$, objective $f$, proposal distribution $P(\cdot | x, \theta)$, parameter update rule $U_P$,  solution update rule $U_S$}
      \For{$m = 1$ to $K$}{
      Draw samples $\{y_i\}_{i=1}^{N} \stackrel{i.i.d.}{\sim} P(y|x_m, \theta_{m})$ \tcp*[r]{\textcolor{Blue}{Generate candidate solutions}} \label{alg:line:sample}
      Evaluate function values $\{f_i = f(y_i)\}_{i=1}^{N}$\;
      $x_{m+1} = U_S(x_m, \{y_i\}_{i=1}^{N}, \{f_i\}_{i=1}^{N}, \theta_{m})$\tcp*[r]{\textcolor{Blue}{Calculate new candidate}} \label{alg:line:update_solution}
      $\theta_{m+1} = U_P(\theta_m, m, x_m)$   \tcp*[r]{\textcolor{Blue}{Adapt sampling strategy}} \label{alg:line:update_parameters}
      }
      \Return{$x_K$}
    \end{algorithm}
  \end{minipage}
  \vspace{-1em}
\end{wrapfigure}
\textbf{Zeroth-order gradient estimation.}
A large class of gradient-free optimization uses function evaluations to estimate the gradient via one-point/multi-point estimators and conducts stochastic gradient descent (SGD) \cite{liuZerothOrderStochasticVariance2018,balasubramanianZerothorderNonconvexStochastic2019,wangStochasticZerothorderOptimization2018}.
It has been studied in bandit online optimization settings~\citep{flaxmanOnlineConvexOptimization2004,bachHighlySmoothZerothOrder} for convex functions, showing its dimension-dependent global convergence rate \citep{nesterovRandomGradientFreeMinimization2017,shamirOptimalAlgorithmBandit,duchiOptimalRatesZeroorder2014}. For non-convex functions, related algorithms have also been proposed in~\citep{belloniEscapingLocalMinima2015} with SGD-style updates.
However, these works mainly show local convergence in non-convex settings. In contrast, our work studies the landscape of the smoothed objective which enables global convergence analysis.

\textbf{(Cross-entropy-style) sampling-based optimization.}
Our work is closely related to a class of methods that cast optimization as a sampling problem from a target distribution (cf. \eqref{eq:target_dist}). For example,
\citet{deboerTutorialCrossEntropyMethod2005,rubinsteinCrossEntropyMethod2004} propose a gradient-free method that minimizes the KL divergence between the target distribution and the sampling distribution.
It has been applied to planning~\citep{pinneriSampleefficientCrossEntropyMethod2021,chuaDeepReinforcementLearning2018}, nonlinear programming~\citep{kothariOptimalGenerationExpansion2009}, and power systems~\citep{CrossentropyMethodPower} and has been proved to asymptotically converge to the target distribution~\citep{margolinConvergenceCrossEntropyMethod2005,josephCrossEntropyBased2018}.
Monte Carlo Markov Chain (MCMC) methods~\citep{metropolisEquationStateCalculations1953,hastingsMonteCarloSampling1970} like Metropolis-Hastings~\citep{greenMetropolisMethodsGaussian1992}, Simulated Annealing~\citep{bertsimasSimulatedAnnealing1993}, and CMA-ES~\citep{akimotoTheoreticalFoundationCMAES2012} are proposed to draw samples from the target distribution by constructing a Markov chain that converges to the target distribution.
While sharing the same proposal distribution and update rules, our work differs in two ways.
First, instead of minimizing the KL divergence to the target distribution, we aim to find the optimal solution for the original problem.
Second, a major technical contribution is the landscape analysis for the smoothed distribution (cf. \eqref{eq:homotopy_dist}) that enables us to show global non-asymptotic convergence of SBO, which to the best of our knowledge is not present in this literature.

\textbf{Homotopy optimization.} Our results on the varying smoothing parameter (\Cref{sec:multistage_convergence}) are closely related to homotopy optimization \cite{blakeVisualReconstruction1987,yuilleEnergyFunctionsEarly1989,allowerPLContinuationMethods1990,allgowerPLHomotopyAlgorithms1990,dunlavyHomotopyOptimizationMethods2005}. With wide applications \cite{yuilleEnergyFunctionsEarly1989,terzopoulosComputationVisiblesurfaceRepresentations1988,goldLearningPreknowledgeClustering1994,broxLargeDisplacementOptical2011,hrubyLearningSolveHard2021,bengioLearningDeepArchitectures,jainNonconvexOptimizationMachine2017,chaudhariEntropySGDBiasingGradient2017,gargianiTransferringOptimalityData2019,panNewFractionalHomotopy2019,bergmanCombiningHomotopyMethods2017,panModelBasedDiffusionTrajectory2024,xueFullOrderSamplingBasedMPC2024}, homotopy optimization works by solving a sequence of smoothed surrogate problems that gradually transform from an easy-to-solve objective to the original complex objective~\citep{linContinuationPathLearning2023}, and the aim is to find (near) global optima for the original objective.
Among various homotopy methods, the Gaussian smoothing (or continuation) method~\citep{loogBehaviorSpatialCritical2001,mobahiLinkGaussianHomotopy2015} has gained significant attention due to its simplicity and effectiveness. Extensive theoretical analyses have been conducted for Gaussian smoothing: its transport optimality is examined in~\citep{mobahiLinkGaussianHomotopy2015}, and bounds on the endpoint based on the function's optimization complexity are derived in~\citep{mobahiTheoreticalAnalysisOptimization2015}. Regarding convergence, a double-loop convergence is proved in~\citep{hazanGraduatedOptimizationStochastic2015} under the assumption that the function is $\sigma$-nice, while a single-loop convergence is established in~\citep{iwakiriSingleLoopGaussian2022,linContinuationPathLearning2023}.
However, the theoretical understanding of homotopy methods often rely on strong assumptions on the landscape of the smoothed objective. In contrast, our work directly characterizes the benign properties of the landscape.

\section{Problem Formulation and Preliminaries}
\label{sec:problem_formulation}
We focus on a general unconstrained optimization problem:
\begin{equation} \label{eq:original_problem}
    \min_{x \in \mathbb{R}^d} f(x),
    \vspace{-5pt}
\end{equation}
\begin{wraptable}{r}{0.58\textwidth}
    \begin{minipage}{\linewidth}
        \centering
        \footnotesize
        \setlength{\tabcolsep}{2pt}
        \begin{tabular}{@{}lll@{}}
            \toprule
            \textbf{Algorithm}                                           & \makecell{\textbf{Parameters}                                                 \\ \cref{alg:line:update_parameters}} & \makecell{\textbf{Solution} \\ \cref{alg:line:update_solution}} \\
            \midrule
            MPPI~\citep{williamsInformationTheoreticModelPredictive2018} & \textcolor{Blue}{$t$}, \textcolor{Blue}{$\lambda$}                  & Softmax \\
            SA~\citep{colemanIsotropicEffectiveEnergy1993}               & \textcolor{Blue}{$t$}, \textcolor{Red}{$\lambda$}                   & Softmax \\
            CEM~\citep{rubinsteinCrossEntropyMethod2004}                 & \textcolor{Blue}{$t$}, \textcolor{Blue}{$\lambda$}                  & Top-k   \\
            CMA-ES~\citep{akimotoTheoreticalFoundationCMAES2012}         & \textcolor{Yellow}{$t$}$^{\text{cov}}$, \textcolor{Blue}{$\lambda$} & Softmax \\
            MBD~\citep{panModelBasedDiffusionTrajectory2024}             & \textcolor{Red}{$t$}, \textcolor{Blue}{$\lambda$}                   & Softmax \\
            \midrule
            \method{} (Ours)                                             & \textcolor{Red}{$t$}, \textcolor{Red}{$\lambda$}                    & Softmax \\
            \bottomrule
        \end{tabular}
    \end{minipage}
    \caption[]{Comparison of different SBO algorithms with Gaussian as proposal distribution whose kernel parameterized by smoothing parameter $t$ and involve temperature parameter $\lambda$.\footnotemark\label{tab:baseline_comparison}}
    \vspace{-8pt}
\end{wraptable}
\footnotetext{\textcolor{Blue}{Blue} indicates fixed parameters, \textcolor{Red}{Red} indicates annealed parameters. $\textcolor{Yellow}{t}^{\text{cov}}$ for CMA-ES denotes $t$ adapted via covariance matrix.}

where the objective function $f: \mathbb{R}^d \to \mathbb{R}$ may be non-convex, and we denote a global optimizer as $x^*$ with minimum $f^* = f(x^*)$.
We consider a broad class of \emph{sampling-based optimization methods} (SBO). SBO (\cref{alg:general_sampling_optimization}) is composed of three steps at each iteration $m$:
(a) Update parameters (\cref{alg:line:update_parameters}): update the sampling parameters $\theta_m$ based on the previous iterate $x_m$ and the sampling parameters.
(b) Sample (\cref{alg:line:sample}): at each step $m$, propose $N$ candidate points $\{y_i\}_{i=1}^N$ from a sampling distribution $P(y|x_m, \theta_m)$ and evaluate the corresponding function values $\{f_i = f(y_i)\}_{i=1}^N$.
One common choice of $P(y|x, \theta)$ is Gaussian distribution $k(y-x, t) = (2 \pi t)^{-d/2} e^{-\|y-x\|^2 / 2t}$ with variance $t$ and mean $x$, which is the focus of this paper.
(c) Update solution (\cref{alg:line:update_solution}): update the current solution estimate $x$ following certain update rules.
Common practice include softmax update rule:
\begin{align}
    x_{m+1} & = \sum_{i=1}^N w_i y_i \quad \text{where} \quad w_i = \frac{e^{-\frac{1}{\lambda} f(y_i)}}{\sum_{j=1}^N e^{-\frac{1}{\lambda} f(y_j)}}
    \label{eq:softmax_update}
\end{align}
and top-k update rule: $x_{m+1} = \frac{1}{N} \sum_{j=1}^N y_{i_j} \label{eq:topk_update}$.
Crucially, the updates rely solely on the function $f(y_i)$ evaluation without requiring direct computation of gradients $\nabla f$, making it a \emph{zero-order optimization} method.
Based on how the sampling distribution $P(y|x, \theta)$ is updated and the update rules they use, \cref{tab:baseline_comparison} summarizes the most popular SBO algorithms.

This work proposes that the success of sampling based optimization lies in an implicit connection to gradient-based optimization, but on a modified landscape. Specifically, we show that sampling-based methods effectively perform gradient descent on a \emph{smoothed version} of the original objective function. This implicit smoothing mitigates the challenges posed by non-convexity, facilitating convergence towards better solutions, potentially global minima. In the subsequent subsection, we formally define the smoothed objective function and show its connection to sampling-based optimization.

\textbf{Roadmap for the rest of the paper.} \Cref{sec:preliminaries} defines the smoothed objective $g(x;t)$ and its connection to SBO. Then, in \Cref{sec:result}, we will leverage this smoothing framework to analyze the convergence properties of sampling-based optimization methods. Specifically, in~\cref{sec:tradeoff}, we will provide a landscape analysis of the smoothed objective $g(x;t)$. Based on the landscape result,  \Cref{sec:single_stage_convergence} we will provide the convergence analysis of SBO.
Our analysis also leads to several algorithmic insights on better design of SBO algorithms, which are discussed in \Cref{sec:multistage_convergence}. Finally, implications of our results in diffusion models will be discussed in \Cref{sec:diffusion_implications}.

\subsection{Preliminaries: Sampling-based optimization is ZO-GD on smoothed objective}
\label{sec:preliminaries}

\textbf{Smoothed objective.}
We consider the Gibbs-Boltzmann distribution
\begin{equation} \label{eq:target_dist}
    p_0(x) \propto e^{-f(x) / \lambda},
\end{equation}
where $\lambda > 0$ is a temperature parameter. Minimizing $f(x)$ is equivalent to finding the mode of $p_0(x)$. This distribution is widely used as in Langevin dynamics based algorithms for optimization~\citep{maSamplingCanBe2019,wibisonoSamplingOptimizationSpace2018,xuGlobalConvergenceLangevin2018}.
As discussed earlier, the effectiveness of these methods is deeply connected to an implicit \textbf{landscape smoothing} effect, i.e, their update mechanisms effectively approximate gradient steps on a smoothed version of the original landscape~\citep{panModelBasedDiffusionTrajectory2024}. To define this smoothed landscape, let the Gaussian kernel $k(z; t) = (2 \pi t)^{-d/2} e^{-\|z\|^2 / 2t}$ be the zero-mean Gaussian kernel with variance $t > 0$. We define the smoothed distribution $p(x; t)$ by convoluting the original Boltzmann distribution $p_0(x)$ (Eq. \ref{eq:target_dist}) with this kernel, and the smoothed objective as
\begin{equation} \label{eq:homotopy_dist}
    p(x; t) = (p_0(\cdot) \ast k(\cdot; t))(x) = \int_{\mathbb{R}^d} p_0(y) k(x - y; t) dy, \quad g(x; t) = - \lambda \log p(x; t).
\end{equation}

The parameter $t$ controls the degree of smoothing: as $t \to 0$, $g(x; t)$ approaches $f(x)$ (up to constants), while for larger $t$, $g(x; t)$ becomes smoother, potentially convexifying the landscape even if $f(x)$ is highly non-convex. Furthermore, as $t$ increases, the smoothed density $p(\cdot; t)=p_0 \ast k(\cdot; t)$ becomes progressively more regular, and the corresponding smoothed objective
$g(\cdot; t) = -\lambda \log p(\cdot; t)$ inherits a smoother landscape than $f$.
In \Cref{theorem:SCBound} we show that, under \Cref{assum:subgaussian,assum:local_convexity}, $g(\cdot; t)$ is strongly convex in a neighborhood
$\mathcal{R}_{SC}(t)$ with curvature on the order of
$\Omega\!\left(\lambda\big/(t+\lambda/\alpha)\right)$.

\textbf{Connection between SBO and $g(x;t)$.} SBO's update step (\cref{alg:line:update_solution}) with softmax update rules (\cref{eq:softmax_update}) can be interpreted as performing \emph{zero-order gradient descent} on this smoothed objective $g(x; t)$.

\begin{proposition}[Adapted from~\citep{mobahiTheoreticalAnalysisOptimization2015,panModelBasedDiffusionTrajectory2024}]
\label{prop:sbo_is_zogd}
Let $p_{t|0}(y\mid x)=\mathcal N(x,tI)$ and recall $p_0(y)\propto \exp(-f(y)/\lambda)$ and
$g(x;t)=-\lambda\log\!\big((p_0 * k(\cdot;t))(x)\big)$.
Then the softmax update~\cref{eq:softmax_update} can be written as one step of estimated gradient descent on $g(\cdot;t)$:
\begin{align}
x^+ &= x - \frac{t}{\lambda}\,\nabla_x g^{(0)}(x;t),\\
\nabla_x g^{(0)}(x;t)
&:= \frac{\lambda}{t}\left(
x-\frac{\sum_{i=1}^N w_i\,y_i}{\sum_{i=1}^N w_i}
\right),
\qquad
\{y_i\}_{i=1}^N \stackrel{i.i.d.}{\sim} p_{t|0}(\cdot\mid x),
\label{eq:zero_order_gradient_approx}
\end{align}
where $w_i := \exp(-f(y_i)/\lambda)$ (equivalently, $w_i\propto p_0(y_i)$; the normalizing constant of $p_0$ cancels).
Moreover, $\nabla_x g^{(0)}(x;t)$ is a zeroth-order (Monte Carlo) estimator of $\nabla_x g(x;t)$ in the sense that, as $N\to\infty$,
\[
\nabla_x g^{(0)}(x;t)\ \to\ \nabla_x g(x;t)\quad\text{in probability}.
\]
\end{proposition}

\cref{prop:sbo_is_zogd} shows that SBO updates perform zeroth-order gradient descent on the smoothed objective $g(x;t)$, rather than directly on $f$.
Intuitively, smoothing via convolution yields a more benign landscape for optimization.
In the next section, we formalize this intuition by analyzing the geometry of $g(x;t)$ (e.g., local strong convexity) and deriving convergence guarantees for SBO.

\textbf{Connection between Smoothing and Diffusion.}
Gaussian kernel smoothing is also standard in the diffusion (score-based) literature, where one samples from a potentially complex target distribution $p_0(x)$ by progressively corrupting it with Gaussian noise and then approximately reversing this noising process.
For concreteness, we adopt the additive Brownian-noise forward process
\citep{karrasAnalyzingImprovingTraining2024}\footnote{
Many works instead use a \emph{variance-preserving} parameterization (e.g., an Ornstein--Uhlenbeck/VP-SDE).
These formulations are closely related to additive Gaussian smoothing through a rescaling of the state and a reparameterization of the noise level; see \citet{luSimplifyingStabilizingScaling2024} for a unified view.
}:
$$
    d x_t = d w_t, \quad t \in [0, T], \quad x_0 \sim p_0(x)
$$
where $w_t$ is a standard Brownian motion.
In short, the target distribution $p_0(x)$ is corrupted with Gaussian noise $k(x; t) = (2 \pi t)^{-d/2} e^{-\|x\|^2 / 2t}$ to generate a smoothed distribution $p(x; t) = p_0(x) \ast k(x; t)$.
The sampling procedure is done via the backward process, where the sample is guided back following the following ODE that goes in reverse time~\cite{songDenoisingDiffusionImplicit2020}:
$$
    d x_t  = - \frac{1}{2} \nabla \log p(x_t; t) dt, \quad t \in [T, 0]
$$
In the above equation, note that $ - \log p(x; t)$ is exactly the smoothed objective $g(x;t)$ we defined in~\cref{eq:homotopy_dist} up to the multiplicative factor $\frac{1}{\lambda}$. Therefore, the above ODE is equivalent to the gradient flow on $g(x; t)$ with time-varying $t$, which is similar to SBO except in diffusion model, the quantity $ \nabla\log p(x; t)$ is typically estimated by a neural network. In light of this connection between SBO and diffusion model, we will discuss the implications of our result for diffusion model in \Cref{sec:diffusion_implications}.

\section{Main results}
\label{sec:result}
As established in~\cref{sec:preliminaries}, sampling-based optimization methods can be interpreted as zeroth-order gradient descent on a smoothed objective function \(g(x;t)\).
Building on this foundation, this section presents our core theoretical contributions, where we rigorously analyze the properties of this smoothed landscape and the convergence behavior of SBO that navigates it.
Our analysis begins with the key assumptions (\cref{sec:assumptions}). In \Cref{sec:tradeoff}, we detail our main findings on the landscape of \(g(x;t)\), particularly the \emph{coverage-optimality tradeoff} inherent in the smoothing parameter \(t\). This tradeoff governs the balance between achieving a more favorable landscape structure and the optimality gap introduced by the smoothing process. Subsequently, we explore the implications of this landscape for sampling-based optimization, showing a global convergence analysis for SBO to an approximate optimizer (\cref{sec:single_stage_convergence}). Finally, the coverage-optimality tradeoff also suggests a way to jointly anneal the smoothing parameter $t$ and the temperature parameter $\lambda$ to improve convergence, which we show in (\cref{sec:multistage_convergence}) and we demonstrate their effective convergence to the true optimum. To the best of our knowledge, the necessity of co-annealing the temperature parameter \(\lambda\) in conjunction with \(t\) is not known in the literature and hence we name our new algorithm as \methodfull{}(\method{}).

\subsection{Assumptions}
\label{sec:assumptions}
To analyze the behavior of our optimization approach, we impose the following assumptions on the problem structure. Our first assumption is on the tail bound of the distribution $p_0$.
\begin{assumption}[Global Sub-Gaussian Assumption]
    There exists a constant $D_{\tau}>0$ such that the tail probability deviating from the optimal solution $x^{*}$ is bounded:
    $
        P_0(\| x - x^* \| \geq a) \leq \exp\left(-\frac{a^2}{2\tau^2}\right), \quad \forall a \geq D_{\tau},
    $
    where $P_0$ is the probability measure induced by the density $p_0$.
    We also denote the total probability of being outside the $D_\tau$-ball around $x^*$ as
    $
        P_0(\| x - x^* \| \geq D_{\tau}) = \Pout{}.
    $
    \label{assum:subgaussian}
\end{assumption}
The sub-Gaussian assumption (\cref{assum:subgaussian}) describes the target distribution \(p_0(x)\)'s concentration around the global optimum \(x^*\). Weaker than log-concavity \citep{maSamplingCanBe2019}, it constrains only tail behavior, accommodating a broader class of \textbf{non-convex} functions. This condition is met in many practical problems, like regularized machine learning losses~\citep{jainNonconvexOptimizationMachine2017}, or LQR control~\citep{panNewFractionalHomotopy2019}. Adopting it allows our analysis to cover more non-convex functions, ensuring sufficient probability mass near the optimum. \(\tau\) sets the concentration width, and \(\Pout{}\) quantifies tail probability.

Our second assumption is on the behavior of $f$ around $x^*$. For notational convenience, we define the following regions: $\BallTau{} := \left\{x | \|x - x^* \| < D_\tau\right\}$, a ball of radius $D_\tau$ around $x^*$.
\begin{assumption}[Local Convexity and Smoothness]
    We assume that the objective function \(f(x)\) is strongly convex and smooth within the ball \(\BallTau{}\) centered at the optimum \(x^*\). Specifically, there exists constants \(\alpha > 0\) and \(\beta > 0\) such that for all \(x, y \in \BallTau{}\), the function satisfies \(\alpha\)-strong convexity: \(f(y) \geq f(x) + \nabla f(x)^\top (y - x) + \frac{\alpha}{2} \| y - x \|^2\), and \(\beta\)-smoothness: \(\| \nabla f(x) - \nabla f(y) \| \leq \beta \| x - y \|\). The ratio \(\kappa_0 = \beta/\alpha\) defines the local condition number of \(f(x)\) within \(\BallTau{}\).
    \label{assum:local_convexity}
\end{assumption}

Assumptions \ref{assum:subgaussian} and \ref{assum:local_convexity} provide the necessary structure for our analysis. Specifically, \cref{assum:local_convexity} ensures local regularity (strong convexity and smoothness) near the optimum \(x^*\), while \cref{assum:subgaussian} controls the global tail decay of the target distribution $p_0$. This combined structure enables a rigorous analysis of the convergence behavior of sampling-based optimization methods within the proposed framework, which we introduce in the subsequent subsections.
\subsection{Coverage and Optimality Tradeoff: The Role of Smoothing}
\label{sec:tradeoff}
Here we show that smoothing the original objective \(f(x)\) into \(g(x;t)\) via the parameter \(t\) induces a fundamental tradeoff. This tradeoff balances the \textit{coverage} of desirable landscape properties (such as convexity) with the \textit{optimality gap} between the minimizer $x_t^*$ of \(g(x;t)\) and the true global minimizer \(x^*\) resulting from the smoothing process. Our main results demonstrate that as \(t\) increases, the convex region of \(g(x;t)\) expands. Concurrently, however, the optimality gap widens, as the minimizer of \(g(x;t)\) deviates further from the true optimum of \(f(x)\). We detail these findings below.

\textbf{Coverage: Expansion of convex region.} We now state our first main result, showing that as $t$ increases, the convex region around $x^*$ expands at least on the order of $\sqrt{t}$.
\begin{restatable}[Strongly Convex Bound]{theorem}{SCBound}
    \label{theorem:SCBound}
    Under \Cref{assum:subgaussian} and \Cref{assum:local_convexity}, let $d$ denote the dimension of the space and define $\kappa_0:=\beta/\alpha$.
    Fix any $C_\alpha\in(0,1)$ and assume $\lambda\le \lambda_{\max}$ for some fixed $\lambda_{\max}>0$.
    Then, for any $t > 0$, the function $g(x;t)$ is $\tfrac{C_\alpha\lambda}{t + \tfrac{\lambda}{\alpha}}$-strongly convex within the region
    \begin{equation}
        \mathcal{R}_{SC}(t)
        :=  \left\{x\in \mathbb{R}^d \;\middle|\;
        \|x - x^*\|
        \leq  C_E\min\!\left(
        \sqrt{\frac{t+\tfrac{\lambda}{\beta}}{\tfrac{\lambda}{\beta}}}\;,\;
        \sqrt{\frac{t+\tau^2}{\tau^2}}
        \right)D_\tau
        \right\}
        \label{eq:convex_radius}
    \end{equation}
    i.e., for all $x\in \mathcal{R}_{SC}(t)$ we have $\nabla^2 g(x;t) \succeq \tfrac{C_\alpha\lambda}{t + \tfrac{\lambda}{\alpha}} I$, provided the sufficient parameter conditions in \cref{eq:conv_suff} hold.
    Here $D_\tau$, $\tau$, and $\Pout{}$ are the sub-Gaussian/tail parameters from \cref{assum:subgaussian}, and $C_E$ is the expansion-rate factor defining $\mathcal{R}_{SC}(t)$.
\end{restatable}

As shown in \cref{theorem:SCBound}, increasing the smoothing parameter \(t\) generally improves the landscape's structure for optimization. Larger values of \(t\) tend to expand the region \(\mathcal{R}_{SC}(t)\) within which the smoothed function \(g(x;t)\) exhibits strong convexity, a phenomenon illustrated in \cref{fig:gmm_combined_plots}. The sufficient condition \cref{eq:conv_suff} makes explicit a concrete high-dimensional parameter regime in which this strong convexity guarantee holds; in particular it enforces the scalings \(D_\tau^2=\Theta(d)\), \(\Pout{}=O(1/d)\), and \(C_E^2=\Theta((\log d)/d)\), along with an explicit upper bound on \(\tau\) in terms of \((\kappa_0,C_\alpha)\).

A larger convex region provides a wider basin of attraction, making it easier for optimization algorithms (including sampling-based ones that implicitly perform gradient descent on \(g(x;t)\)) to navigate towards \(x_t^*\), which is defined as $x_t^* := \argmin_x g(x;t)$ and avoid getting trapped in potentially complex local structures inherited from \(f(x)\). In essence, large \(t\) increases the "coverage" of the well-behaved, convex-like landscape. Next, we introduce the other side of the tradeoff, that is, a larger $t$ will introduce greater optimality gap.

\textbf{Optimality gap.} Despite the larger coverage, smoothing the objective function introduces optimality gap: the minimizer \(x_t^*\) of \(g(x;t)\) generally differs from the original function's global minimizer \(x^*\). Understanding and bounding this gap is crucial, as a large deviation would mean that optimizing \(g(x;t)\) does not yield a sufficiently accurate solution to the original problem \(\eqref{eq:original_problem}\). The subsequent theorem bounds this gap, relating it to the smoothing parameter \(t\) and the properties of \(f\).

\begin{restatable}[Optimality gap]{theorem}{OptBound}
    \label{theorem:bias_bound}
    Let \cref{assum:subgaussian} and \cref{assum:local_convexity} hold. For any smoothing parameter \(t > 0\), let \(x_t^*\) denote the unique minimizer of the smoothed function \(g(x;t)\) within the strongly convex region \(\mathcal{R}_{SC}(t)\), as established in \cref{theorem:SCBound}. If \(x_t^*\) also lies within the region of local strong convexity and smoothness for the original function \(f(x)\), i.e., \(x_t^* \in \BallTau{} = \{x \mid \|x - x^*\| < D_\tau\}\), then the optimality gap \(\|x_t^* - x^*\|\) is bounded by:
    \begin{equation}
        \label{eq:bias_bound}
        \left\| x_t^* - x^* \right\| \leq \min\left\{\frac{(1-C_\alpha) t}{C_\alpha} \frac{1}{4 D_\tau}, (D_\tau
        + \tau) \frac{t+\frac{\lambda}{\alpha}}{C_\alpha t}\right\} + \frac{\kappa_0 -1}{ 2 C_\alpha} \sqrt{\tfrac{1}{2\pi(\tfrac{1}{t}+\tfrac{\alpha}{\lambda})}}.
    \end{equation}
    where $C_\alpha \in (0,1)$ is the strong convexity parameter from \cref{theorem:SCBound}, and $D_\tau$, $\tau$ are the parameters from Assumption \ref{assum:subgaussian}.
\end{restatable}

\cref{theorem:bias_bound} shows a larger \(t\) increases the optimality gap \(\|x_t^* - x^*\|\).
The bound in \eqref{eq:bias_bound} indicates that the gap arises from two primary sources. The first, captured by both $\frac{(1-C_\alpha) t}{C_\alpha} \frac{1}{4 D_\tau}$and $(D_\tau + \tau) \frac{t+\frac{\lambda}{\alpha}}{C_\alpha t}$, is largely influenced by the smoothing parameter \(t\) and the sub-Gaussian tail. The second term, $\frac{\kappa_0-1}{2C_\alpha}\sqrt{\frac{1}{2\pi\left(\frac1t+\frac{\alpha}{\lambda}\right)}}$, captures \emph{local asymmetry} around $x^*$: when the local condition number $\kappa_0=\beta/\alpha$ is close to $1$, this contribution is small, and it grows with increasing anisotropy/curvature mismatch (larger $\kappa_0$). Notably, both contributions vanish as $t\to 0$, consistent with $g(\cdot;t)\to f$ (up to constants). Moreover, the bound remains finite as $t\to\infty$, matching the empirical behavior reported in \cref{sec:exp}.

\textbf{Discussion.} The landscape coverage (\cref{theorem:SCBound}) versus optimality gap (\cref{theorem:bias_bound}) tradeoff formally underpins successful SBO methods like MPPI \citep{williamsInformationTheoreticModel2017} and diffusion-inspired approaches like Model-Based Diffusion~\citep{panModelBasedDiffusionTrajectory2024}. These methods often implicitly use smoothing (via sample averaging or explicit diffusion) to navigate complex, non-convex landscapes. Our theorems quantify this fundamental smoothing tradeoff: larger $t$ creates a more benign, larger convex region but increases the optimality gap.

The coverage-optimality tradeoff developed here forms the foundation for analyzing the convergence of SBO algorithms, which we present in \Cref{sec:single_stage_convergence}.
Beyond the fixed $t$ regime, the inherent tradeoff in $t$ motivates the use of an annealing schedule for $t$. We can start with a relatively large $t$ to leverage the enlarged convex region, allowing the algorithm to start in the convex region. Then, $t$ is gradually decreased to reduce the optimality gap, guiding the iterates towards the true global minimum $x^*$ while staying within the convex region. In Section \ref{sec:multistage_convergence}, we detail this annealing process and provide the convergence guarantee.

\textbf{Multiple Convex Regions.} While our current analysis focuses on landscapes with a dominant global minimum characterized by local convexity (\Cref{assum:subgaussian,assum:local_convexity}), our analysis techniques could potentially be extended to multi-convex-region (multi-modal) scenarios. Larger \(t\) might merge distinct basins of attraction, while annealing could help distinguish them later. Further investigation is needed to formalize the behavior in landscapes with multiple convex regions.

\subsection{\texorpdfstring{Convergence Analysis with Fixed Smoothing Parameter $t$}{Convergence Analysis with Fixed Smoothing Parameter t}}
\label{sec:single_stage_convergence}

\paragraph{From landscape analysis to convergence guarantees.}
The landscape results of \cref{sec:tradeoff} supply two key ingredients for convergence analysis.
First, the strong-convexity bound (\cref{theorem:SCBound}) guarantees that $g(x;t)$ is
$\alpha_t$-strongly convex with $\alpha_t = \tfrac{C_\alpha\lambda}{t+\lambda/\alpha}$ inside
an expanding region $\mathcal{R}_{SC}(t)$, providing a basin in which gradient-based updates
contract.
Second, the optimality-gap bound (\cref{theorem:bias_bound}) quantifies the bias
$\|x_t^*-x^*\|$ introduced by smoothing.
Together, these determine the three-term error decomposition that underlies both the single-stage
and multi-stage results below:
\emph{contraction} (governed by the $t$-dependent condition number $\kappa_t=\beta_t/\alpha_t$),
\emph{estimation noise} (captured by the gradient-estimator bias $K_t$ and variance $\sigma_t^2$
from \cref{thm:gradient-bounds}), and \emph{optimality gap} (from \cref{theorem:bias_bound}).

For a fixed smoothing parameter $t$, \cref{thm:single_stage} combines these ingredients to
establish a non-asymptotic convergence rate to a neighborhood of $x^*$ whose size is controlled by
the optimality gap.
The coverage--optimality tradeoff then motivates annealing $t$ (and $\lambda$) across stages:
each stage~$m$ inherits the landscape guarantees at its own noise level $t_m$ with
stage-indexed parameters $\alpha_m,\beta_m,\kappa_m$, progressively tightening the optimality gap
while remaining inside the contracting convex region, culminating in the global convergence
guarantee of \cref{thm:dida_convergence}.
A complete notation reference is provided in \cref{tab:notation}.

\textbf{Zeroth-order gradient estimator bounds.} Understanding the zeroth-order gradient estimator (\cref{eq:zero_order_gradient_approx}) is crucial for SBO's non-asymptotic convergence. \Cref{thm:gradient-bounds} details its bias and variance bounds: bias bounded by $\frac{\lambda}{N} \cdot \mathcal{P}_B(t^{-1/2}, t^{1/2})$ and variance by $\frac{\lambda^2}{N} \cdot \mathcal{P}_V(t^{-1}, t)$. The polynomials $\mathcal{P}_B, \mathcal{P}_V$ (e.g., $\mathcal{P}_B= c_1 x + c_0 + c_2 y$) have coefficients that does not depend on $N$. This ensures estimator consistency, as bias/variance vanish for large $N$.

\textbf{Convergence of SBO under fixed $t$.}
Combining~\cref{theorem:SCBound} and the gradient estimator bounds~\cref{thm:gradient-bounds}, we can derive the non-asymptotic convergence rate of the SBO in \cref{eq:zero_order_gradient_approx}.

\begin{restatable}[Convergence of Sampling without Annealing over $t$]{theorem}{SingleStageConvergence}
    \label{thm:single_stage}
    Given fixed noise $t$, the function $g(x;t)$ is $\alpha_t$-strongly convex, $\beta_t$-smooth, $L_t$-Lipschitz, and condition number is $\kappa_t = \frac{\beta_t}{\alpha_t}$.
    The step size is $\eta = \frac{\alpha_t}{4\beta_t^2}$.
    If initial iterate lies in original convex region: $x_0 \in \mathcal{R}_{SC}(t)$ (\cref{theorem:SCBound}), with probability at least $(1 - \delta)$, the error satisfies:
    \begin{align*}
        \| x_k - x^* \|^2 \le \|x_t^* - x^*\|^2 + (1 - \frac{1}{4 \kappa_t^2})^k \| x_0 - x^* \|^2 + \frac{4(K_t^2 + \sigma_t^2)}{\delta \ \alpha_t} (\frac{t}{2\lambda} + \frac{1}{\alpha_t})
    \end{align*}
    where $\|x_t^* - x^*\|^2$ follows \cref{theorem:bias_bound}, $K_t =\frac{\lambda}{N} \left(M_{-\frac{1}{2}} t^{-\frac{1}{2}} + M_0 \frac{L_t}{\lambda} + M_\frac{1}{2} \left(\frac{L_t}{\lambda}\right)^2 t^{\frac{1}{2}}\right)$ and $\sigma^2_t = \frac{\lambda^2}{N} \left(V_{-1} t^{-1} + V_0 \left(\frac{L_t}{\lambda}\right)^2 + V_1 \left(\frac{L_t}{\lambda}\right)^4 t\right)$.
\end{restatable}

\cref{thm:single_stage} indicates that, the SBO in \eqref{eq:zero_order_gradient_approx} enjoys \emph{linear convergence rate} with an optimality gap in the order of  $O(t)$.
In practice, we choose step size:
$\eta = \frac{1}{4 \kappa_0 t}$ when $t \gg D_\tau$, and $\eta = \frac{1}{4 \beta_0}$ otherwise.
Different from previous asymptotic analysis~\citep{iwakiriSingleLoopGaussian2022}, \cref{thm:single_stage} provides a \emph{non-asymptotic global convergence} result for SBO for the nonconvex landscape for the first time, which reveals several practical insights in algorithm design and hyperparameter selection:
Firstly, \emph{greater sample size $N$} reduces both the bias and variance of the gradient estimator and thus the final optimality gap.
Secondly, according to gradient estimator bounds~\cref{thm:gradient-bounds}, there exists a \emph{optimal temperature} $\lambda^* = (\frac{U_1}{U_0})^{\frac{1}{4}} \beta \sqrt{t}$ where the final optimality gap is minimized, which inspires a dual annealing strategy later in~\cref{alg:adaptive_annealing}.

\subsection{\texorpdfstring{Convergence Analysis with Varying Smoothing Parameter $t$}{Convergence Analysis with Varying Smoothing Parameter t}}
\label{sec:multistage_convergence}

As stated in~\cref{sec:tradeoff}, a varying smoothing parameter $t$ is able to balance between coverage and optimality of SBO.
The major challenge in extending the convergence analysis from fixed $t$ to adaptive $t$ is to find an appropriate schedule for $t$ to make sure the iterates always stay in the convex region.
In~\cref{thm:dida_convergence}, we prove that using a diffusion-style \emph{geometric time schedule} along side with annealing in temperature $\lambda$, each iteration can stay within the local convex region with a reduced sample size.
Then, leveraging the property that the optimality gap (\cref{thm:single_stage}) is shrinking at the same rate ($O(t)$) as the convex region expansion (\cref{theorem:SCBound}), we can show the convergence to $x^*$.

\begin{restatable}[Global Convergence of Dual-Level Annealing Algorithm]{theorem}{DIDAConvergence}
    \label{thm:dida_convergence}
    According to convex radius bound in \cref{theorem:SCBound}, there exists a unique time $t_{M_0}$ where the radius of the convex region is minimum.
    Consider the update rule from time $t_0$ to $t_M$ with total number of steps $M > M_0$:
    \begin{align}
        t_{m+1} = \begin{cases} \gamma t_m, & m < M_0 \\ t_F, & m \ge M_0 \end{cases}, \quad x_{t_{m+1}} = x_{t_m} - \eta \nabla \hat{g}(x_{t_m}; t_m)
    \end{align}
    where $t_F < t_{M_0}$ is the final sampling kernel, the sampling temperature is set adaptively as $\lambda_m = \beta \sqrt{t_m}$, step size is set as $\eta = \frac{\alpha_m}{4 \beta_m^2}$.
    With adaptive annealing $\lambda_m = \beta \sqrt{t_m}$, the required sample size $N$ is bounded by:
    \begin{align}
        N = \frac{3 \beta^2 M d}{2 E_0 D_\tau \beta_1^2 \delta} (V_{-1} + V_0 + V_1)
    \end{align}
    where $t_c$ is the smallest time parameter which applies to the gradient estimator bound in and $V_{-1}, V_0, V_1$ are the constants in gradient estimator bounds.
    Without adaptive annealing, the required sample size $N$ is bounded by:
    \begin{align}
        N      & = \max \{N(t_0), N(t_c)\},                                                                                                                                \\
        N(t_0) & = \frac{3 \lambda_0^2 M d}{2 E_0 D_\tau \beta_1^2 \delta} \left( V_{-1} t_0^{-1} + V_0 (\frac{\beta}{\lambda_0})^2 + V_1 (\frac{\beta}{\lambda_0})^4 t_0 \right), \\
        N(t_c) & = \frac{3 \lambda_c^2 M d}{2 E_0 D_\tau \beta_1^2 \delta} \left( V_{-1} t_c^{-1} + V_0 (\frac{\beta}{\lambda_c})^2 + V_1 (\frac{\beta}{\lambda_c})^4 t_c \right)
    \end{align}
    Then with probability at least $1-\delta$, the dual-level annealing algorithm converges to
    \begin{align}
        \| x_M - x^* \|^2 \le \|x_F^* - x^*\|^2 + (1 - \frac{1}{4 \kappa_F^2})^{M - M_0} (C_E^2 D_\tau^2 + k_g t_{M_0}) + \frac{4(K_F^2 + \sigma_F^2)}{\delta \ \alpha_F} (\frac{t_F}{2\lambda_F} + \frac{1}{\alpha_F})
    \end{align}
    where $k_g = C_E^2 \min \{\frac{\beta^2}{\lambda^2}, \frac{1}{\tau^4}\}$.
\end{restatable}

\cref{thm:dida_convergence} provides the first non-asymptotic global convergence result of SBO towards the global minimizer.
Algorithm-wise, \cref{thm:dida_convergence} unifies two annealing strategies in the literature: simulated annealing~\citep{bertsimasSimulatedAnnealing1993} for temperature $\lambda$ and diffusion annealing~\citep{panModelBasedDiffusionTrajectory2024} for noise level $t$.
Our results suggest the temperature $\lambda$ and noise level $t$ should be jointly scheduled to achieve the best convergence rate: in the high $t$ regime, larger temperature $\lambda$ is preferred to make the distribution less concentrated to encourage exploration; in the low $t$ regime, smaller temperature $\lambda$ is preferred to make the distribution more concentrated to encourage convergence to the exact minimizer.

\begin{wrapfigure}{r}{0.6\textwidth}
    \hspace{0.7em}%
    \begin{minipage}{0.96\linewidth}
        \begin{algorithm}[H]
            \small
            \caption{\methodfull{} for Zeroth-Order Optimization}
            \label{alg:adaptive_annealing}
            \KwIn{initial noise $T$, initial guess $x_T$, sample size $N$, iteration number $M$, annealing rate $\gamma$}
            Initialize $x_0 \gets x_T$, $t_0 \gets T$\;
            \For{$m = 1$ to $M$}{
            $t_{m} = \gamma \cdot t_{m-1}$ \label{line:update_t} \tcp*[r]{\textcolor{Blue}{Smoothing annealing}}
            Draw samples $\{y_i\}_{i=1}^{N} \sim p_{t|0}(y|x)$\;
            $\lambda_m = \sqrt{\Var{f(y_i)}}$ \label{line:update_lambda} \tcp*[r]{\textcolor{Blue}{Temperature annealing}}
            Estimate $\nabla_x g^{(0)}(x; t)$ with \cref{eq:zero_order_gradient_approx}\;
            Update: $x_{m+1} \gets x_m - \frac{1}{4} t_m \cdot \nabla_x g^{(0)}(x_m; t_m)$\;
            }
            \Return{$x_M$}
        \end{algorithm}
    \end{minipage}
\end{wrapfigure}
\textbf{Algorithm Design.} Inspired by \cref{thm:dida_convergence}, we propose \methodfull{} (\method{}), featuring the dual annealing strategy detailed in~\cref{alg:adaptive_annealing}: \emph{smoothing annealing} over $t$ (\cref{line:update_t}) and \emph{temperature annealing} over $\lambda$ (\cref{line:update_lambda}).
For \emph{smoothing annealing} over $t$, we approximate $t_m = \gamma^m t_0$, where $\gamma \in (0,1)$ is a hyper-parameter.
For \emph{temperature annealing} over $\lambda$, we follows $\lambda_m = \beta\,\sqrt{t_m}$ with an approximated Lipschitz constant $\beta^2 \approx \frac{\Var{f(x_t)}}{\Var{x_t}} = \frac{\Var{f(x_t)}}{t_m}$, where $\Var{f(x_t)}$ is the variance of the sampled function values.
The approximated Lipschitz leads to a simplified temperature schedule $\lambda_m = \sqrt{\Var{f(x_t)}}$.

\subsection{Implications for Diffusion Models}

\label{sec:diffusion_implications}
As shown in~\cref{sec:problem_formulation}, the ODE form of the reverse process in diffusion can be viewed as a gradient flow on log density $-\log p(x;t)$, which is exactly $g(x,t)$ up to a constant.
Therefore, the diffusion model can be viewed as  descent of learned gradients on the smoothed objective, whose convexity is improved over $t$ as stated in~\cref{sec:tradeoff}.
This optimization perspective of diffusion model enables us to understand the stability of the convergence behavior by analyzing the convexity of the landscape.

\begin{wrapfigure}{r}{0.55\textwidth}
    \centering
    \includegraphics[width=\linewidth]{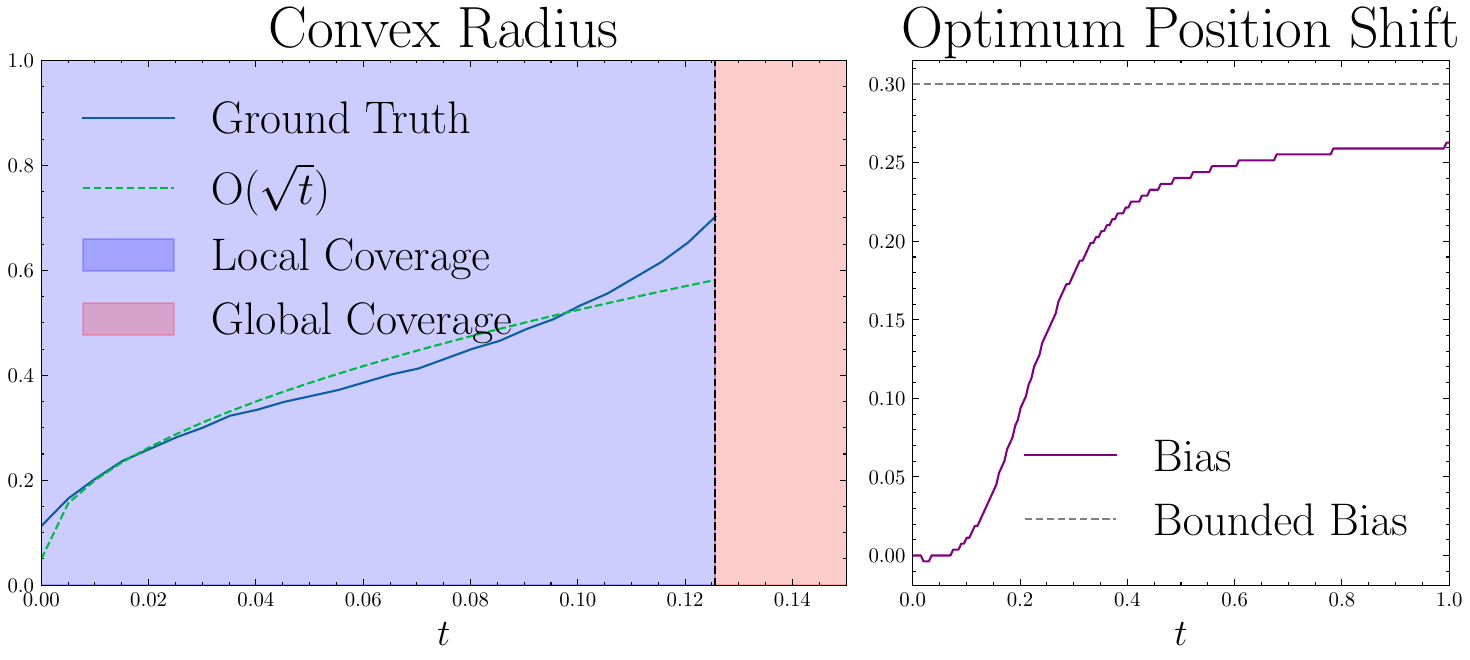}
    \caption{The coverage and optimality gap for GMM in \cref{fig:landscape_tradeoff}. Both the coverage expansion and bounded optimum shift matches with \cref{theorem:SCBound,theorem:bias_bound}.}
    \label{fig:gmm_combined_plots}
\end{wrapfigure}
\textbf{Guidance makes convex region more dominant.}
Classifier-free guidance has been widely adopted to improve the sample quality of diffusion models~\citep{jeonUnderstandingMemorizationGenerative2025,hoClassifierFreeDiffusionGuidance2022}.
\cref{theorem:SCBound} offers a theoretical insight for diffusion models: the more concentrated the initial distribution (i.e. the smaller sub-Gaussian parameter), the faster the convergence.
In classifier-free guidance, the original multi-modal distribution is concentrated to a single mode by conditioning on a class, where tuning the weight of the guidance can trade off the sample quality and sample diversity.
From optimization perspective, the more smoothed the landscape is, the larger the contraction rate would be, leading to faster convergence and better tolerance to the noise in score estimation.

\section{Experimental Results}
\label{sec:exp}
\begin{wraptable}{r}{0.7\textwidth}
    \tiny
    \centering
    \setlength{\tabcolsep}{3pt}
    \renewcommand{\arraystretch}{0.93}
    \begin{tabular}{@{}l >{\raggedright\arraybackslash}p{1.9cm} ccc@{}}
        \toprule
         & Task/Env. & DIDA(ours) & CEM~\citep{rubinsteinCrossEntropyMethod2004} & CMA-ES~\citep{akimotoTheoreticalFoundationCMAES2012} \\
        \midrule
        \multicolumn{5}{@{}l}{\emph{Blackbox optimization}} \\
         & Ackley (d=200)     & \textbf{$\mathbf{3.1 \pm \scriptstyle 0.1}$}      & $14.3 \pm \scriptstyle 0.1$      & $14.2 \pm \scriptstyle 0.1$      \\
         & Ackley (d=400)     & \textbf{$\mathbf{4.4 \pm \scriptstyle 0.2}$}      & $14.7 \pm \scriptstyle 0.1$      & $14.6 \pm \scriptstyle 0.0$      \\
         & Ackley (d=800)     & \textbf{$\mathbf{6.0 \pm \scriptstyle 0.1}$}      & $14.9 \pm \scriptstyle 0.0$      & $14.8 \pm \scriptstyle 0.0$      \\
         & Levy (d=200)       & \textbf{$\mathbf{11.8 \pm \scriptstyle 2.0}$}     & $744.3 \pm \scriptstyle 23.7$    & $744.3 \pm \scriptstyle 23.7$    \\
         & Levy (d=400)       & \textbf{$\mathbf{53.6 \pm \scriptstyle 5.0}$}     & $1567.4 \pm \scriptstyle 28.2$   & $1567.4 \pm \scriptstyle 28.2$   \\
         & Levy (d=800)       & \textbf{$\mathbf{202.5 \pm \scriptstyle 11.3}$}   & $3212.5 \pm \scriptstyle 24.8$   & $3212.5 \pm \scriptstyle 24.8$   \\
         & Rastrigin (d=200)  & \textbf{$\mathbf{1703.3 \pm \scriptstyle 65.0}$}  & $3644.2 \pm \scriptstyle 32.0$   & $3648.8 \pm \scriptstyle 43.4$   \\
         & Rastrigin (d=400)  & \textbf{$\mathbf{3782.1 \pm \scriptstyle 80.3}$}  & $7478.4 \pm \scriptstyle 76.0$   & $7478.4 \pm \scriptstyle 76.0$   \\
         & Rastrigin (d=800)  & \textbf{$\mathbf{8337.7 \pm \scriptstyle 132.9}$} & $15231.6 \pm \scriptstyle 116.9$ & $15231.6 \pm \scriptstyle 116.9$ \\
        \midrule
        \multicolumn{5}{@{}l}{\emph{Trajectory optimization}} \\
         & ant                & $0.032 \pm \scriptstyle 0.080$                    & $0.649 \pm \scriptstyle 0.101$   & $0.879 \pm \scriptstyle 0.177$   \\
         & halfcheetah        & $\mathbf{0.414 \pm \scriptstyle 0.042}$           & $0.998 \pm \scriptstyle 0.008$   & $0.995 \pm \scriptstyle 0.006$   \\
         & hopper             & $\mathbf{0.623 \pm \scriptstyle 0.007}$           & $0.861 \pm \scriptstyle 0.006$   & $0.929 \pm \scriptstyle 0.010$   \\
         & humanoidrun        & $\mathbf{0.298 \pm \scriptstyle 0.059}$           & $0.973 \pm \scriptstyle 0.008$   & $0.989 \pm \scriptstyle 0.017$   \\
         & humanoidstandup    & $\mathbf{0.781 \pm \scriptstyle 0.025}$           & $0.876 \pm \scriptstyle 0.000$   & $0.876 \pm \scriptstyle 0.000$   \\
         & humanoidtrack      & $\mathbf{0.845 \pm \scriptstyle 0.009}$           & $1.015 \pm \scriptstyle 0.002$   & $1.022 \pm \scriptstyle 0.008$   \\
         & pushT              & $\mathbf{0.834 \pm \scriptstyle 0.034}$           & $0.960 \pm \scriptstyle 0.032$   & $1.028 \pm \scriptstyle 0.031$   \\
         & walker2d           & $\mathbf{0.352 \pm \scriptstyle 0.062}$           & $0.850 \pm \scriptstyle 0.001$   & $0.848 \pm \scriptstyle 0.001$   \\
        \bottomrule
    \end{tabular}
    \caption{Summary of optimized cost comparison for DIDA, CEM, and CMA-ES. Full comparison with all baselines can be found in the Appendix (\Cref{tab:full_results_appendix}).}
\label{tab:merged_performance}
\end{wraptable}

To validate the coverage and optimality tradeoff in~\cref{theorem:SCBound,theorem:bias_bound}, we visualize the landscape and the coverage/optimality gap for a 1-d Gaussian Mixture model in \cref{fig:landscape_tradeoff,fig:gmm_combined_plots} and for the checkerboard function in \cref{fig:checkerboard_coverage}.
The coverage expansion rate and optimality gap are within the bound of our theory.

To validate the empirical performance of the proposed~\method{}, we compare its performance with several SBO methods in~\cref{tab:baseline_comparison}.
In \cref{tab:merged_performance}, we evaluate baselines and \method{} on high-dimensional black-box optimization and contact-rich trajectory optimization tasks. \method{} outperforms all baselines with a clear margin thanks to its dual annealing strategy, demonstrating the effectiveness of our theoretical prediction.

\begin{figure}[ht]
    \centering
    \includegraphics[width=\textwidth]{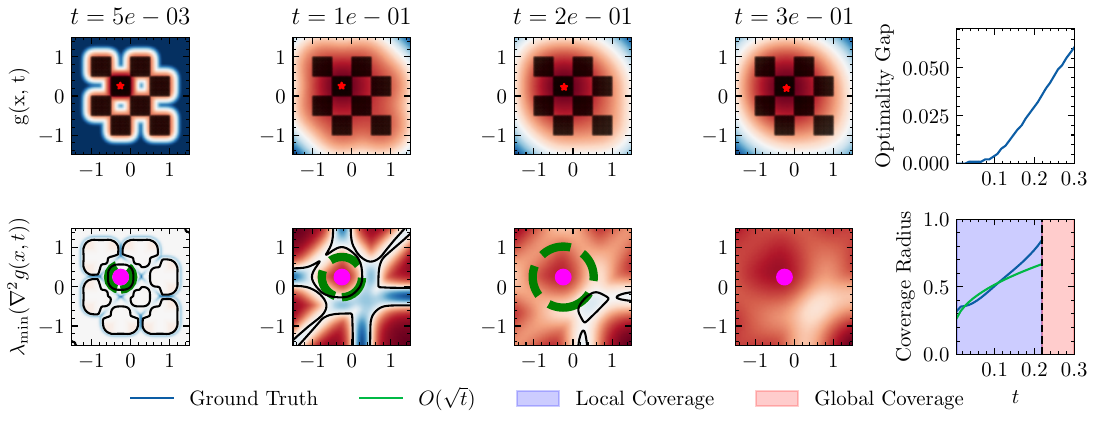}
    \caption{The smoothed landscape of checkerboard function with its coverage radius and optimality gap over different $t$. When $t$ increases, both coverage and optimality gap increase.}
    \label{fig:checkerboard_coverage}
\end{figure}

\section{Conclusion and Future Work}
This paper conducts a comprehensive study on the non-asymptotic convergence behavior of SBO algorithms through the lens of diffusion-style smoothing. 
Based on our bias-coverage tradeoff analysis, we propose a new SBO algorithm, \method{}, demonstrating strong empirical performance. 
Future work includes extending our theoretical analysis to function with multiple optima and applying our landscape analysis to diffusion models to improve sampling efficiency.

\bibliographystyle{tmlr}
\bibliography{refs_chaoyi}

\appendix
\newpage

\section*{Appendix}

\section{Notation}

\begin{table*}[htbp]
    \centering
    \begin{tabular}{lll}
        \toprule
        \textbf{Symbol}                & \textbf{Meaning}                        & \textbf{Definition}                                                       \\
        \midrule
        $x$                            & decision variable                       & $x \in \mathbb{R}^d$                                                      \\
        $f(x)$                         & original objective function             & $f: \mathbb{R}^d \to \mathbb{R}$                                          \\
        $x^*$                          & global minimizer of $f(x)$              & $f(x^*) = \min_x f(x)$                                                    \\
        $p_0(x)$                       & target distribution density             & $p_0(x) = \frac{e^{-f(x) / \lambda}}{\int e^{-f(x) / \lambda} dx}$        \\
        $P_0(\cdot)$                    & target distribution measure             & Associated with density $p_0(x)$                                          \\
        $k(\cdot;t)$                   & kernel function                         & $k(\cdot;t) = (2 \pi t)^{-d/2} e^{-\|\cdot\|^2 / 2t}$                     \\
        $p(x ; t)$                     & smoothed distribution                   & $p(x ; t) = (p_0(\cdot) \ast k(\cdot ; t))(x)$                            \\
        $g(x ; t)$                     & smoothed objective function             & $g(x ; t) = - \lambda \log p(x ; t)$                                      \\
        $x_t^*$                        & minimizer of $g(x;t)$                   & $g(x_t^*; t) = \min_x g(x; t)$                                               \\
        $p_{t \mid 0}(x \mid y)$       & forward distribution                    & $p_{t \mid 0}(x \mid y) = k(x - y ; t)$                                   \\
        $p_{0 \mid t}(y \mid x)$       & backward distribution                   & $p_{0 \mid t}(y \mid x) = \frac{p_0(y) p_{t \mid 0}(y \mid x)}{p(x ; t)}$ \\
        $\lambda$                       & temperature parameter                   & $\lambda > 0$                                                              \\
        $\alpha$                        & local strong convexity of $f(x)$      & Constant $\alpha > 0$                                                       \\
        $\beta$                            & local smoothness of $f(x)$    & Constant $\beta > 0$                                                            \\
        $\kappa_0$                      & local condition number of $f(x)$        & $\kappa_0 = \beta/\alpha$                                                       \\
        $\nabla_x g(x; t)$             & gradient of smoothed objective function &                                                                           \\
        $\nabla_x \hat{g}^{(0)}(x; t)$ & zeroth-order gradient estimator         &                                                                                                                   \\
        $\mathcal{N}(\mu, \Sigma)$      & Gaussian distribution                  &                                                                           \\
        $\mathcal{SG}(\mu, \Sigma)$     & sub-Gaussian distribution              &  $P(\| x - \mu \| \leq t) = \exp(-t^2 / \Sigma)$                                                                     \\
        $D_s$                            & determined by sub-Gaussian tail parameters & $D_s = \tau \sqrt{\log(1/\Pout{})}$                                                       \\
        \bottomrule
    \end{tabular}
    \label{tab:notation}
\end{table*}

\subsection{Important Equations}

\begin{proposition}[Smoothed Function Properties] \label{prop:smoothed_properties}
The smoothed function \(g(x;t)\) is defined as:
\[
g(x; t) = -\lambda \log \left( p_0(\cdot) \ast k(\cdot ; t) \right)(x)
\]
Its gradient and Hessian have the following forms:

 \textbf{Gradient Equations:}
        \begin{align}
            \nabla_x g(x; t)   & = \mathbb{E}_{y \sim p_{0 \mid t}} \left[ \nabla f(y) \mid x \right] = \int \nabla f(y) p_{0 \mid t}(y \mid x) dy \label{eq:grad_g_exp_f} \\
            \nabla_x g(x; t)   & = -\lambda \frac{\nabla_x p(x; t)}{p(x; t)}   = \frac{\lambda}{t} \left( x - \mathbb{E}_{y \sim p_{0|t}} [ y \mid x ] \right) \label{eq:grad_g_tweedie} \\
        \end{align}
 \textbf{Hessian Equations:}
        \begin{align}
            \nabla_x^2 g(x; t) & = \mathbb{E}_{y \sim p_{0|t}} \left[ \nabla_x^2 f(y) \mid x \right ] - \frac{1}{\lambda} \text{Cov}_{y \sim p_{0|t}} \left[ \nabla_x f(y) \mid x \right ] \label{eq:hess_g_exp_cov} \\
            \nabla_x^2 g(x; t) & = \frac{\lambda}{t^2} \left( x \mathbb{E}_{y \sim p_{0|t}} \left[ \nabla f(x) \mid y \right ] - \mathbb{E}_{y \sim p_{0|t}} \left[ x \nabla f(x) \mid y \right ] \right ) \label{eq:hess_g_other_form} \\
            \nabla_x^2 g(x; t) & = \frac{\lambda}{t^2} \left( t I  - \text{Cov}_{y \sim p_{0 \mid t}}[y \mid x]  \right) \label{eq:second_order_tweedie_prop} 
        \end{align}

\end{proposition}

    \textbf{Lower Bound via Brascamp-Lieb Inequality}

    Recall the \emph{Matrix Brascamp-Lieb inequality}~\cite{bakryPoincareInequalities2014}:
    \begin{lemma}
        \label{lem:Br-l}
        For a probability measure \( d\mu = e^{-W(y)} dy \) on \( \mathbb{R}^n \) with strictly convex potential \( W: \mathbb{R}^n \to \mathbb{R} \), we have for any smooth vector-valued function \( h: \mathbb{R}^n \to \mathbb{R}^n \):
        \[
            \operatorname{Cov}_\mu(h) \preceq \int_{\mathbb{R}^n} (\nabla^2 W(y))^{-1} \nabla h(y) \nabla h(y)^\top \, d\mu(y).
        \]
    \end{lemma}

\section{Landscape Analysis}

\subsection{Additional Corollaries for \texorpdfstring{\cref{theorem:SCBound}}{Theorem SCBound}}
Here we first recap the main result about the landscape. 
\SCBound*

We next record the explicit sufficient parameter conditions referenced in \cref{theorem:SCBound}:
\begin{equation}
    \label{eq:conv_suff}
    \begin{alignedat}{2}
        \text{(i)}\quad  & D_\tau^2 \geq b d
        &\quad & \text{for some constant }b\ge
        \max\!\left\{\frac{9}{4}\frac{\lambda_{\max}}{\beta},\;\frac{1-C_\alpha}{18}\kappa_0^2,\;\frac{81}{8e}\kappa_0^4\right\}, \\
        \text{(ii)}\quad & \Pout{}\le \frac{1}{d}, \\
        \text{(iii)}\quad& \tau \le \tau_{\max}:=\frac{2}{3\sqrt{3}}\sqrt{1-C_\alpha}\,\kappa_0, \\
        \text{(iv)}\quad & C_E^2 \le \frac{\beta}{16\lambda_{\max}b}\,\frac{\log d}{d}. \\
        \text{(v)}\quad & \lambda \le 2 \alpha
    \end{alignedat}
\end{equation}

\subsection{Proof for \texorpdfstring{\cref{theorem:SCBound}}{Theorem SCBound}}

In this subsection, we focus on the proof of \cref{theorem:SCBound}. To start, the following lemmas are instrumental in establishing \cref{theorem:SCBound}. They provide a detailed analysis of the smoothed landscape's properties, which underpins the main theorem's conclusions regarding the conditions for strong convexity and the expansion of this convex region.
\begin{lemma}[Truncated sub-Gaussian moments]
\label{lemma:truancated_subg_moment}
    Let $Y\ge 0$ be a random variable. We say that $Y$ is \emph{(one-sided) sub-Gaussian with
    parameter $\tau^2$}, and write $Y\sim \SG(0, \tau^2)$, if its tail satisfies
    \begin{equation}\label{eq:sg_tail_def}
    \mathbb{P}(Y>r)\le \exp\!\left(-\frac{r^2}{\tau^2}\right),\qquad \forall r\ge 0,
    \end{equation}
    for some $\tau>0$. Then, for any $a\ge 0$,
    \begin{align}
    \mathbb{E}[Y \mid Y>a]
    &\le a+\frac{\tau^2}{a+\sqrt{a^2+\frac{4\tau^2}{\pi}}}, \label{eq:subg_first_moment}\\
    \mathbb{E}[Y^2 \mid Y>a]
    &\le a^2+\tau^2. \label{eq:subg_second_moment}
    \end{align}
    Furthermore, if the exact tail probability at level $a$ is known,
    \[
    p_a := \mathbb{P}(Y>a),
    \]
    it is convenient to introduce the \emph{effective tail radius}
    \[
    D := \tau\sqrt{\log\!\Big(\tfrac{1}{p_a}\Big)} \qquad (\text{so that } \ e^{-D^2/\tau^2}=p_a),
    \]
    which satisfies $D \ge a$ by \eqref{eq:sg_tail_def}. In terms of $D$, we have
    \begin{align}
    \mathbb{E}[Y \mid Y>a]
    &\le D+\frac{\tau^2}{D+\sqrt{D^2+\frac{4\tau^2}{\pi}}}
    \;\le\; D+\tau, \label{eq:subg_first_moment_general}\\
    \mathbb{E}[Y^2 \mid Y>a]
    &\le D^2+\tau^2. \label{eq:subg_second_moment_general}
    \end{align}
\end{lemma}
\begin{proof}
    We start from the tail-integral identities for nonnegative random variables. For any $Z\ge 0$,
    Tonelli's theorem yields
    \[
    \Expec{Z} =\int_0^\infty P(Z>r)\,dr,
    \qquad
    \Expec{Z^2}=\int_0^\infty 2r\,P(Z>r)\,dr.
    \]
    Applying these to $(Y-a)_+$ gives the standard representations
    \begin{align}
    \Expec{Y\mid Y>a}
    &= a+\frac{\int_a^\infty P(Y>r)\,dr}{P(Y>a)}, \label{eq:tail_rep_first}\\
    \Expec{Y^2\mid Y>a}
    &= a^2+\frac{\int_a^\infty 2r\,P(Y>r)\,dr}{P(Y>a)}. \label{eq:tail_rep_second}
    \end{align}
    
    \paragraph{Step 1: Bounds in terms of the tail at level $a$.}
    Assume $Y\sim \SG(0, \tau^2)$ in the sense of \eqref{eq:sg_tail_def}. Then for all $r\ge a$,
    $P(Y>r)\le e^{-r^2/\tau^2}$, and therefore
    \[
    \int_a^\infty P(Y>r)\,dr \le \int_a^\infty e^{-r^2/\tau^2}\,dr,
    \qquad
    \int_a^\infty 2r\,P(Y>r)\,dr \le \int_a^\infty 2r\,e^{-r^2/\tau^2}\,dr.
    \]
    For the second integral we have the exact evaluation
    \[
    \int_a^\infty 2r\,e^{-r^2/\tau^2}\,dr = \tau^2 e^{-a^2/\tau^2}.
    \]
    For the first integral we use a Gaussian tail bound (Mills' ratio; e.g.\ \cite[Lemma~1.1]{laurent2000adaptive}):
    \begin{equation}\label{eq:mills_upper}
    \int_a^\infty e^{-r^2/\tau^2}\,dr
    \le
    \frac{\tau^2 e^{-a^2/\tau^2}}{a+\sqrt{a^2+\frac{4\tau^2}{\pi}}}.
    \end{equation}
    Substituting these bounds into \eqref{eq:tail_rep_first}--\eqref{eq:tail_rep_second} yields
    \[
    \Expec{Y\mid Y>a}
    \le
    a+\frac{\tau^2 e^{-a^2/\tau^2}}{\bigl(a+\sqrt{a^2+\frac{4\tau^2}{\pi}}\bigr)\,P(Y>a)},
    \qquad
    \Expec{Y^2\mid Y>a}
    \le
    a^2+\frac{\tau^2 e^{-a^2/\tau^2}}{P(Y>a)}.
    \]
    In the special case where the sub-Gaussian bound is tight at level $a$, namely $P(Y>a)=e^{-a^2/\tau^2}$, these simplify to
    \eqref{eq:subg_first_moment}--\eqref{eq:subg_second_moment}.
    
    \paragraph{Step 2: General case via the effective tail radius.}
    Let $p_a:=P(Y>a)$ and define $D:=\tau\sqrt{\log(1/p_a)}$, so that $e^{-D^2/\tau^2}=p_a$.
    By \eqref{eq:sg_tail_def}, $p_a\le e^{-a^2/\tau^2}$, hence $D\ge a$.
    Moreover, for all $r\ge a$,
    \[
    P(Y>r)\le \min\{p_a,\,e^{-r^2/\tau^2}\}
    =
    \begin{cases}
    p_a, & a\le r\le D,\\
    e^{-r^2/\tau^2}, & r\ge D.
    \end{cases}
    \]
    
    \emph{First moment.}
    Using \eqref{eq:tail_rep_first} and splitting the numerator at $D$ gives
    \begin{align*}
    \int_a^\infty P(Y>r)\,dr
    &\le \int_a^{D} p_a\,dr + \int_{D}^\infty e^{-r^2/\tau^2}\,dr\\
    &= p_a(D-a) + \int_{D}^\infty e^{-r^2/\tau^2}\,dr.
    \end{align*}
    Dividing by $p_a$ and applying \eqref{eq:mills_upper} at $D$ (noting $e^{-D^2/\tau^2}=p_a$) yields
    \[
    \Expec{Y\mid Y>a}
    \le a+(D-a) + \frac{1}{p_a}\int_{D}^\infty e^{-r^2/\tau^2}\,dr
    \le D+\frac{\tau^2}{D+\sqrt{D^2+\frac{4\tau^2}{\pi}}},
    \]
    which is \eqref{eq:subg_first_moment_general}. Finally, since $\sqrt{D^2+\frac{4\tau^2}{\pi}}\ge \frac{2\tau}{\sqrt{\pi}}$,
    we have $D+\sqrt{D^2+\frac{4\tau^2}{\pi}}\ge \frac{2\tau}{\sqrt{\pi}}$, and hence
    \[
    \frac{\tau^2}{D+\sqrt{D^2+\frac{4\tau^2}{\pi}}}
    \le \frac{\tau^2}{2\tau/\sqrt{\pi}}
    = \frac{\sqrt{\pi}}{2}\tau
    \le \tau,
    \]
    so $\Expec{Y\mid Y>a}\le D+\tau$.
    
    \emph{Second moment.}
    Similarly, from \eqref{eq:tail_rep_second},
    \begin{align*}
    \int_a^\infty 2r\, P(Y>r)\,dr
    &\le \int_a^{D} 2r\,p_a\,dr + \int_{D}^\infty 2r\,e^{-r^2/\tau^2}\,dr\\
    &= p_a(D^2-a^2) + \tau^2 e^{-D^2/\tau^2}
    = p_a(D^2-a^2+\tau^2).
    \end{align*}
    Dividing by $p_a$ and plugging into \eqref{eq:tail_rep_second} gives
    \[
    \Expec{Y^2\mid Y>a}\le a^2 + (D^2-a^2+\tau^2)=D^2+\tau^2,
    \]
    which proves \eqref{eq:subg_second_moment_general}.
\end{proof}

Based on the above lemma, we now investigate bound for the posterior distribution $\Ppost{t}$, which the basis of our analysis on the landscape.

\begin{lemma}[Convex region covariance bound]
    \label{lem:convex_region_covariance_bound}
    The conditional covariance of $y$ given $y \in B_{\tau}$ and $x_t$ is upper bounded by
    \[
        \CovbtIn{0}{t} \preceq \frac{t\,(\lambda/\alpha)}{\,t + (\lambda/\alpha)\!}\; \mathbf{I}
    \]
\end{lemma}
\begin{proof}
    This comes from the convexity of the region as in \cref{assum:local_convexity} and \cref{lem:Br-l}.
\end{proof}
\begin{lemma}[Sub-Gaussian region moment bound]
    \label{lem:sub_gauss_region_moment_bound}
    Recall the setup and notation in \cref{assum:subgaussian}. The conditional first and second moments of the radius outside the sub-Gaussian region satisfy
    \[
        \EbtOutTnorm{0}{t} \leq  \max \left((D_s+ 2 r_t')^2,D_s^2 + 3 \tau'^2 + 4 \tau' r_t' \right)
    \]
    and
    \[
        \EbtOutnorm{0}{t} \leq \max \left(D_s+ 2 r_t',D_s + 3\tau'  \right)
    \]
    where $D_s := \tau\sqrt{\log(\frac{1}{\Pout{}})}$ , $\tau'^2 = \frac{2 \tau^2 t}{\tau^2 + 2 t}$ and $ r_t' = \frac{\tau^2 \|x_t - x^*\|}{\tau^2 + 2 t} $
\end{lemma}
\begin{proof} \textbf{Sketch}:
 Consider the following functional:
    \begin{align}
        \xi_1(q_0, q_{t|0})& : = \frac{\int_{y\notin B_{\tau}} \|y - x^*\| q_0(y)q_{t|0}(y|x_t) \, dy}{\int_{y\notin B_{\tau}} q_{0}(y) q_{t|0}(y|x_t) \, dy}. \label{eq:func_xi1} \\
        \xi_2(q_0, q_{t|0})& : = \frac{\int_{y\notin B_{\tau}} \|y - x^*\|^2 q_0(y)q_{t|0}(y|x_t) \, dy}{\int_{y\notin B_{\tau}} q_{0}(y) q_{t|0}(y|x_t) \, dy}. \label{eq:func_xi2}
    \end{align}
    By definition, $\xi_1(p_0,p_{t|0}) = \EbtOutnorm{0}{t}$ and $\xi_2(p_0,p_{t|0}) = \EbtOutTnorm{0}{t}$.
    In the following steps, we aim to show that 
    \begin{align}
        \sup_{q_0\in \mathcal{Q_0},q_{t|0}\in\mathcal{Q}_{t|}}\xi_1(q_0, q_{t|0}) &\leq RHS_1 \\
        \sup_{q_0\in \mathcal{Q_0},q_{t|0}\in\mathcal{Q}_{t|0}}\xi_2(q_0, q_{t|0}) &\leq RHS_2 
    \end{align}
    where $RHS_1$ is the RHS for \cref{eq:func_xi1} and $RHS_2$ is the RHS for \cref{eq:func_xi2}, and $p_0 \in \mathcal{Q}$ and $p_{t|0} \in \mathcal{Q}_{t|0}$. Here, $\mathcal{Q}$ and $\mathcal{Q}_{t|0}$ are sets of feasible probability measures that we will define later.
    
\noindent\textbf{Step 1: Reduction to a one-dimensional integral.}\;
Fix any $p_0 \in \mathcal{Q}_0$. In this step, we explain how to bound
\[
    \sup_{q_{t|0}\in\mathcal{Q}_{t|0}}\xi_1(p_0, q_{t|0})
    \quad\text{and}\quad
    \sup_{q_{t|0}\in\mathcal{Q}_{t|0}}\xi_2(p_0, q_{t|0}),
\]
by reducing the optimization over $q_{t|0}$ to a one-dimensional problem in $\|y-x^*\|$.
We view $\xi_1$ and $\xi_2$ as functionals of $(q_0,q_{t|0})$, where $q_0$ is a probability density on $\mathbb{R}^d$ and, for each fixed $x_t$, $q_{t|0}(\cdot|x_t)$ is a conditional density of $Y$ given $X_t=x_t$.
In particular, the conditional first moment that we want to bound is
$ \xi_1(p_0,p_{t|0})$.
Our first step is to transform this conditional expectation into a one-dimensional integral over the radial variable $\Vert y-x^*\Vert$.
To this end, we will ``replace'' $p_{t|0}(\cdot)$ and $p_0(\cdot)$ by simpler radial densities that depend only on $\Vert y-x^*\Vert$; the precise meaning of this replacement will be made clear below.

\emph{Replacing $p_{t|0}(\cdot)$. }The forward transition density $p_{t|0}$ satisfies the following bound:
\begin{align}
    (2\pi t)^{-d/2}\exp\left(-\frac{(\|y-x^*\|+\|x_t-x^*\|)^2}{2t}\right) \leq p_{t|0}(x_t|y) \leq (2\pi t)^{-d/2}\exp\left(-\frac{(\|y-x^*\|-\|x_t-x^*\|)^2}{2t}\right) \label{eq:t0func_bound}
\end{align}
Motivated by this, we define the admissible class of conditional densities as follows:
\[
    \mathcal{Q}_{t|0}
    := \Bigl\{ q_{t|0}(\cdot|x_t) \,\Bigm|\,
    \begin{aligned}
        & q_{t|0}(\cdot|x_t) \text{ is a Borel-measurable density on }\mathbb{R}^d \text{ for each } x_t, \\
        & \int_{\mathbb{R}^d} q_{t|0}(y|x_t)\,dy = 1, \\
        & (2\pi t)^{-d/2}\exp\Bigl(-\tfrac{(\|y-x^*\|+\|x_t-x^*\|)^2}{2t}\Bigr)
            \le q_{t|0}(y|x_t) \\
        &\hspace{2em}\le (2\pi t)^{-d/2}\exp\Bigl(-\tfrac{(\|y-x^*\|-\|x_t-x^*\|)^2}{2t}\Bigr),
        \quad \forall\, y\notin B_\tau
    \end{aligned}
    \Bigr\}.
\]
By construction, $p_{t|0}(\cdot|x_t)\in \mathcal{Q}_{t|0}$ for every $x_t$, and therefore
\[
    \xi_1(p_0, p_{t|0}) \leq  \sup_{q_{t|0}\in \mathcal{Q}_{t|0}}\xi_1(p_0, q_{t|0}).
\]
We now justify the worst-case choice of $q_{t|0}(\cdot|x_t)$ within the admissible envelope \eqref{eq:t0func_bound}.
Fix $q_0=p_0$ and view
\[
\xi_1(p_0,q_{t|0})
=\frac{\int_{y\notin B_\tau} \|y-x^*\|\,p_0(y)\,q_{t|0}(y|x_t)\,dy}{\int_{y\notin B_\tau} p_0(y)\,q_{t|0}(y|x_t)\,dy}
\]
as a functional of $q_{t|0}$ under the pointwise constraints $q_-(y)\le q_{t|0}(y|x_t)\le q_+(y)$ on $B_\tau^c$, where

\begin{align*}
q_-(y)&:=(2\pi t)^{-d/2}
    \exp\!\left(-\frac{(\|y-x^*\|+\|x_t-x^*\|)^2}{2t}\right),\\
q_+(y)&:=(2\pi t)^{-d/2}
    \exp\!\left(-\frac{(\|y-x^*\|-\|x_t-x^*\|)^2}{2t}\right).
\end{align*}

We next compute the functional derivative of $\xi_1$ with respect to $q_{t|0}(y|x_t)$:
\begin{align}
    \frac{\partial \xi_1(q_0, q_{t|0} )}{\partial q_{t|0}(y|x_t)} & = \frac{q_{0}(y)}{\int_{y\notin B_{\tau}} q_{t|0}(y|x_t)q_0(y) \, dy}(\|y-x^*\| - \xi_1(q_0, q_{t|0} )) \label{eq:E_yout_functopt_pt}
\end{align}

Since $\xi_1$ is a ratio of two linear functionals in $q_{t|0}$, the extremizer over such box constraints is attained at an extreme point (a ``bang--bang'' choice): there exists a threshold $r^\star$ such that $q_{t|0}$ takes the lower envelope $q_-$ on $\{\|y-x^*\|<r^\star\}$ and the upper envelope $q_+$ on $\{\|y-x^*\|>r^\star\}$ (ties on the null set $\{\|y-x^*\|=r^\star\}$ are irrelevant). We therefore define
\begin{align}
\Poptpost{t}[0](y|x_t)
:=\begin{cases}
q_-(y), & D_{\tau} \leq \|y-x^*\| < r^\star,\\[2pt]
q_+(y), & \|y-x^*\| \geq r^\star,
\end{cases}
\label{eq:funcopt_pt}
\end{align}
and set
\begin{equation}
\xi_1^{\star} := \xi_1\bigl(p_0,\Poptpost{t}[0]\bigr)
    = \sup_{q_{t|0}\in\mathcal{Q}_{t|0}}\xi_1(p_0,q_{t|0}).
\label{eq:definition_of_b1}
\end{equation}
By construction, $p_{t|0}(\cdot|x_t)\in\mathcal{Q}_{t|0}$, hence
$\xi_1(p_0,p_{t|0})\le \xi_1^{\star}$.

    \emph{Replacing $p_0(\cdot)$:} We now study the following functional for $q_0$:
    \[
        J[q_0]\;=\; \xi_1(q_0,\Poptpost{t}[0]) =  \frac{\int_{y\notin B_{\tau}} \|y-x^*\| q_0(y)\Poptpost{t}[0](y|x_t) \, dy}{\int_{y\notin B_{\tau}} q_{0}(y) \Poptpost{t}[0](y|x_t) dy}.
    \]
    We maximize this functional over $q_0$ subject to the following constraint set:
    \[
        \mathcal A
        =\Bigl\{q_0:[a,\infty)\!\to\!\mathbb R_{\ge 0}\,\Bigm|\,
        \begin{aligned}
             & q_0 \text{ is continuous and a.e.\ differentiable},                       \\
             & \int_{y \notin B_\tau} q_0(y) dy= \Pout{},\quad q_0(y)>0,                 \\
             & \int_{\|y- x^*\| >r} q_0(y) dy \le \exp(-r^2/\tau^2) ,\quad \forall r>D_\tau
        \end{aligned}
        \Bigr\},
    \]
    to obtain the optimal $\Poptpost{0}$. This yields the upper bound
   \[
        \xi_1^* \leq \xi_1(p_0,\Poptpost{t}[0])\leq \xi_1(\Poptpost{0},\Poptpost{t}[0]) = J[\Poptpost{0}].
    \]
    To maximize $J[q_0]$, we formulate the problem using the method of Lagrange multipliers, treating $J$ as an auxiliary parameter. 
    We introduce the following multipliers:
    \begin{itemize}
        \item $\nu(y)\ge 0$ for the sub-Gaussian bound $\int_{\|y - x^*\|>r} q_0(y) dy \le \exp(-r^2/\tau^2)$,
        \item a scalar $\eta$ for the total mass constraint $\int_{y \notin B_\tau}q_0(y)=\Pout{}$,
        \item and $\theta(y)\ge 0$ for the positivity constraint.
    \end{itemize}
    The Lagrangian is given by:
              \begin{align*}
                  \mathcal L[q_0] & \;=\;
                  \int_{\|y-x^*\|>r}
                  \Bigl[(\|y-x^*\| - J)\,\Poptpost{t}[0](y|x_t)\, q_0(y)
                      +\nu(y)\int_{\|y'-x^*\|>r} \bigl(\exp(-r^2/\tau^2)-q_0(y') \bigr)dy'\Bigr]dy\\
                  &\quad
           +\eta\!\int_{y \notin B_\tau}q_0(y)\,dy
                  \;+\;
                  \!\int_{y \notin B_\tau} q_0(y) \theta(y)\,dy.
              \end{align*}

    For the stationary condition, 
              we consider a variation \(q_0\mapsto q_0+\varepsilon h\) with \(\int h=0\) and
              \(\int_{y \notin B_\tau}h=0\),
              \[
                  0=\frac{d}{d\varepsilon}\mathcal L[q_0+\varepsilon h]\Big|_{\varepsilon=0}
                  =\int_{\|y-x^*\|>a}
                  \bigl[(\|y-x^*\| - J)\Poptpost{t}[0](y|x_t)- \int_{\|t-x^*\|\le\|y-x^*\|}\nu(t)
                      +\eta\,+\theta(y)\bigr]\,h(y)\,dy,
              \]
              hence
              \[
                  (\|y-x^*\| - J)\Poptpost{t}[0](y|x_t)-\int_{\|y'-x^*\|\le\|y-x^*\|}\nu(y')dy'+\eta\,+\theta (y)
                  =0
              \]
              Since \(\Poptpost{t}[0](y|x_t)>0\) and the functional derivative changes sign, the stationary condition forces a \emph{bang--bang} structure:
              \begin{align}
                  \begin{cases}
                      \nu(y) = 0,\quad \theta(y)>0, & \|y-x^*\|<\rho,    \\[4pt]
                      \nu(y) > 0,\quad \theta(y)=0  & \|y-x^*\|\ge \rho,
                  \end{cases} \label{eq:functopt_result}
              \end{align}

        Complementary slackness dictates that \(\nu(y)\) and \(\theta(y)\) cannot both be zero, as \((\|y-x^*\| - J)\Poptpost{t}[0](y|x_t)\) is strictly increasing in $\|y-x^*\|$. Thus, there exists a threshold $\rho$ such that \(\nu(y) = 0, \theta(y) > 0\) when $\|y-x^*\| < \rho$, and \(\nu(y) > 0, \theta(y) = 0\) when $\|y-x^*\| \ge \rho$. This corresponds precisely to the bang--bang structure required by the sub-Gaussian constraint. By further examining the constraint $\int_{D_\tau}^{\infty} q_0(y) dy=\Pout{}$, we obtain $\rho = D_s$, where $D_s := \tau\sqrt{\log\!\Bigl(\tfrac{1}{\Pout{}}\Bigr)}$.
              
   This optimal configuration is illustrated in Figure~\ref{fig:modified_density}, where the left segment ($D_{\tau} \leq \|y-x^*\| < r^*$) has smaller $\|y-x^*\|$ values and lower density (minimum $\Poptpost{t}[0]$), while the right segment ($\|y-x^*\| \geq r^*$) has larger $\|y-x^*\|$ values and higher density (maximum $\Poptpost{t}[0]$). Therefore, after self-normalization, the region with larger $\|y-x^*\|$ values has larger density.

        In conclusion, we have the following result for $\Poptpost{0}$ (here we give its CDF $\Tilde{P}_{0}(\|y-x^*\| \geq r)$) and $\Poptpost{t}[0]$ that achieve the maximum $\xi_1^{\star \star}$:
        \begin{align}
            \Tilde{P}_{0}(\|y-x^*\| \geq r)&  = \begin{cases}
            \Pout{} & \text{if }  D_\tau \leq r \leq D_s \\
            \exp(-a^2/\tau^2) & \text{if } r > D_s
            \end{cases}, \\
            \Poptpost{t}[0](y|x_t) &\propto \begin{cases}
               \exp(-(\|y-x^*\|  + \|x_t-x^*\|)^2/2t) & \text{if } D_\tau \leq \|y - x^*\| \leq r^* \\
                \exp(-(\|y-x^*\| - \|x_t-x^*\|)^2/2t)  & \text{if } \|y - x^*\| > r^*
            \end{cases}
            \label{eq:Poptpost_t_0}
        \end{align}
        where
        \begin{align}
            \xi_1^\star\leq \xi_1^{\star \star}:=  \sup_{q_0\in \mathcal{Q_0},q_{t|0}\in\mathcal{Q}_{t|}}\xi_1(q_0, q_{t|0}) =  \xi_1\left(\Poptpost{0},\Poptpost{t}[0]\right) 
        \end{align}
        With the above selected $\Poptpost{0}, \Poptpost{t}[0]$, we can transform $\xi_1$ as follows:
        \begin{align}
            \xi_1\left(\Poptpost{0},\Poptpost{t}[0]\right) & =\frac{\int_{y \notin B_\tau}\left\|y-x^*\right\| \Poptpost{0}(y) \Poptpost{t}[0]\left(y \mid x_t\right) d y}{\int_{y \notin B_\tau} \Poptpost{0}(y) \Poptpost{t}[0]\left(y \mid x_t\right) d y} . \nonumber \\
            & =\frac{\int_{r > D_\tau}\|r\| \int_{\left\|y-x^*\right\|=r} \Poptpost{0}(y) \Poptpost{t}[0]\left(y \mid x_t\right) d y}{\int_{r>D_\tau} \int_{\left\|y-x^*\right\|=r} \Poptpost{0}(y) \Poptpost{t}[0]\left(y \mid x_t\right) d y} . \label{eq:definition_of_xi_1_}  \\
            & :=\frac{\int_{r > D_\tau}\|r\| \poptrpost{t}[0][1]\left(r \mid x_t\right) \poptrpost{0}(r) d r}{\int_{r > D_\tau} \poptrpost{t}[0][1]\left(r \mid x_t\right) \poptrpost{0}(r) d r} . \label{eq:definition_of_xi_1_4}
        \end{align}
where   in \cref{eq:definition_of_xi_1_}, we rewrite the integral using radial coordinates centered at $x^*$. The integration proceeds by first integrating over the spherical surface $\{y : \|y-x^*\| = r \}$ for a fixed radius $r$, and then integrating over $r$. In \cref{eq:definition_of_xi_1_4}, we leverage the property that the optimized transition density $\Poptpost{t}[0](y|x_t)$ is assumed to depend only on the radius $r = \|y-x^*\|$, allowing it to be factored out of the surface integral. The remaining surface integral of the optimized prior $\Poptpost{0}(y)$ defines the radial prior density $\poptrpost{0}(r)$. Consequently, \cref{eq:definition_of_xi_1_4} presents the final result as a simplified one-dimensional integral involving only these radial densities, $\poptrpost{0}(r)$ and $\poptrpost{t}[0][1](r)$ as defined below. 

\begin{align}
          \poptrpost{0}(r) &:= \int_{\|y-x^*\|=r} \Poptpost{0}(y) d y = -\frac{d}{dr} \Tilde{P}_{0}(\|y-x^*\| \geq r) = \left\{ \begin{array}{ll} \frac{2r}{\tau^2} \exp(-r^2/\tau^2) &  \text{ if } r \geq D_s\\
        0 & \text{ if } D_\tau \leq r \leq D_s \end{array}\right., \nonumber \\
          \poptrpost{t}[0][1](r,x_t) &:\propto  \begin{cases}
            \exp(-(r  + \|x_t-x^*\|)^2/2t) & \text{if } D_\tau \leq r \leq r^* \\
             \exp(-(r - \|x_t-x^*\|)^2/2t)  & \text{if } r \geq r^*
         \end{cases}
         \label{eq:definition_of_pr1}
\end{align}

with $r = \|y-x^*\|$. To reach the maximum value, we must have $r^* = \xi_1^{\star \star}$, as the switching point is where $\xi_1 - \| y-x^*\|$ changes sign in the first-order condition. Thus, the analysis in Step 1 successfully reduces the multi-dimensional integrals for the first moment $\xi_1$ to one-dimensional integrals involving the radial densities $\poptrpost{0}(r)$ and $\poptrpost{t}[0][1](r)$, as given in \cref{eq:definition_of_xi_1_4} and the equations above.
              \begin{figure}[ht]
                  \centering
                  \includegraphics[width=0.8\linewidth]{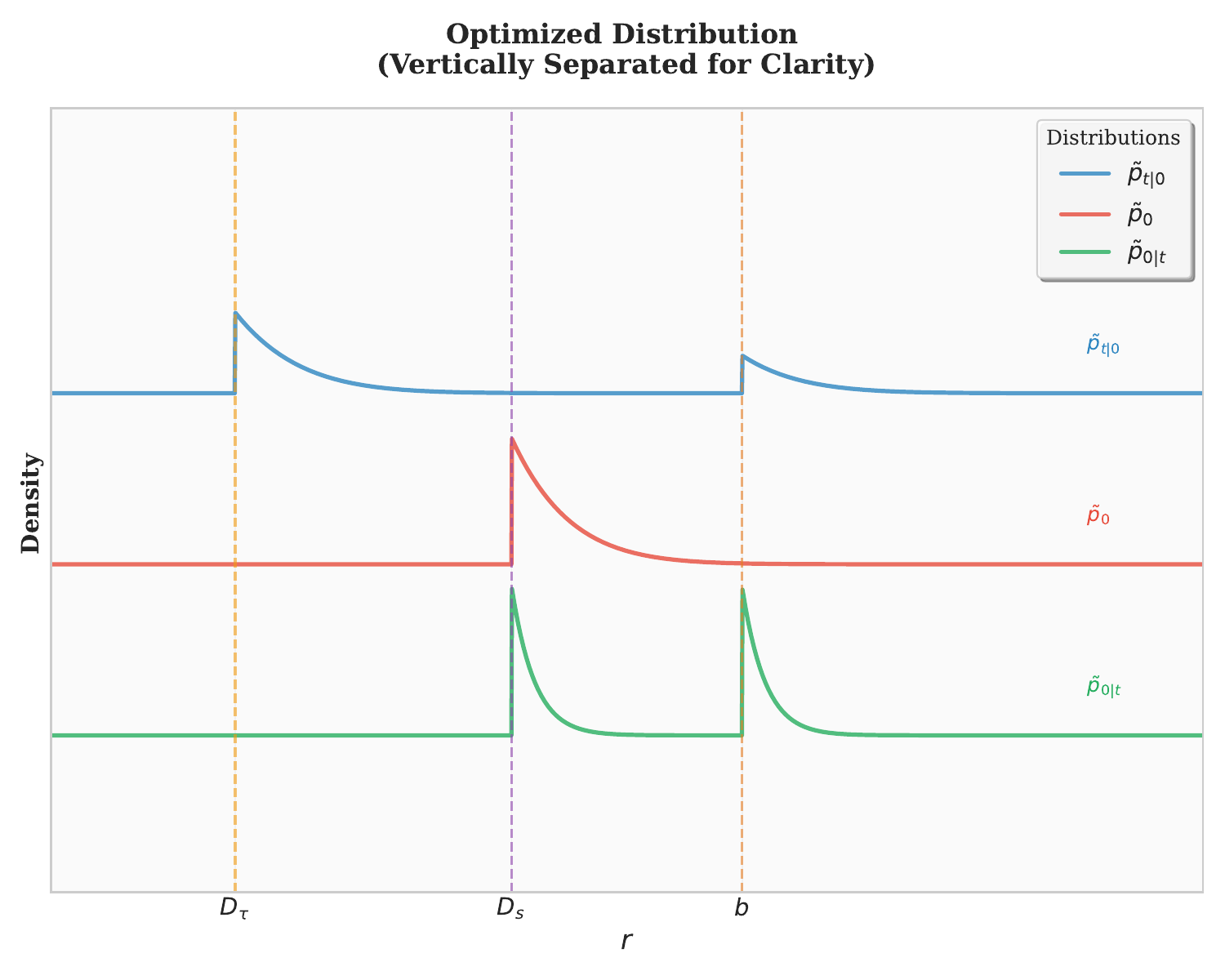}
                  \caption{Modified 1-d density function illustration when functional takes maximum. }
                  \label{fig:modified_density}
              \end{figure}

The same functional optimization framework applies to the second moment as well. While the specific optimal distributions $\Poptpost{0}$ and $\Poptpost{t}[0]$ will differ (as $r^*$ depend on the threshold $\xi_2^{**}$ rather than $\xi_1^{**}$), the overall methodology yields a similar bounding inequality for
\begin{align}
    (\xi_2^{**})^2  &:=  \sup_{q_{t|0}\in\mathcal{Q}_{t|0}}\xi_2(p_0,q_{t|0}) \\
   &= \frac{\int_{r >D_\tau} \|r\|^2 \poptrpost{t}[0][2](r) \poptrpost{0}(r) \, dr}{\int_{r >D_\tau} \poptrpost{t}[0][2](r) \poptrpost{0}(r)\, dr}
\end{align}
where $ \poptrpost{0}(r)$ remains the same and 
\begin{align}
     \poptrpost{t}[0][2](r) :\propto  \begin{cases}
        \exp(-(r  + \|x_t-x^*\|)^2/2t) & \text{if } D_s \leq r \leq r^{\star \star} \\
         \exp(-(r - \|x_t-x^*\|)^2/2t)  & \text{if } r \geq r^{\star\star}
     \end{cases}
     \label{eq:definition_pr2}
\end{align}
To reach the supremum, we must have $r^{\star\star} = \xi_2^{\star\star}$.
This shows that the second-order moment can also be reduced to a one-dimensional integral involving the radial densities $\poptrpost{0}(r)$ and $\poptrpost{t}[0][2](r)$.

\textbf{Step 2: Calculating $1$-d integral. }In this step, we bound $\xi_1^{**}$ and $(\xi_2^{**})^2$ under the polar coordinate setting as outlined in \cref{eq:definition_of_xi_1_4,eq:definition_pr2}. Recall that $\xi_1^{**} \geq \EbtOutnorm{0}{t}$ and $(\xi_2^{**})^2 \geq \EbtOutTnorm{0}{t}$. Based on the results in Step 1, we can upper bound $\xi_1^{**}$ and $\xi_2^{**}$ as
    \begin{align*}
        \xi_1^{**}  \int_{D_s}^{\infty} \poptrpost{t}[0][1](r) \poptrpost{0}(r) \, dr    & \leq \int_{D_s}^{\xi_1^{**}} r \poptrpost{t}[0][1](r) \poptrpost{0}(r) \, dr + \int_{\xi_1^{**}}^{\infty} r \poptrpost{t}[0][1](r) \poptrpost{0}(r) \, dr     \\
        (\xi_2^{**})^2 \int_{D_s}^{\infty} \poptrpost{t}[0][2](r) \poptrpost{0}(r) \, dr & \leq \int_{D_s}^{\xi_2^{**}} r^2 \poptrpost{t}[0][2](r) \poptrpost{0}(r) \, dr + \int_{\xi_2^{**}}^{\infty} r^2 \poptrpost{t}[0][2](r) \poptrpost{0}(r) \, dr 
    \end{align*}
    \paragraph{2.1: Piecewise characterization of $\poptrpost{0}[t][1](r)$ and $\poptrpost{0}[t][2](r)$.}
    Since the definitions of $\xi_1^{**}$ and $(\xi_2^{**})^2$ involve ratios of integrals, the integrands $\poptrpost{0}[t][1](r) \propto \poptrpost{t}[0][1](r) \poptrpost{0}(r)$ and $\poptrpost{0}[t][2](r) \propto \poptrpost{t}[0][2](r) \poptrpost{0}(r)$ only need to be determined up to a constant normalization factor. We therefore have \cref{eq:E_yout_functopt_1,eq:E_yout_functopt_2}:
    \begin{align}
        \xi_1^{**}  \int_{D_s}^{\infty} \poptrpost{0}[t][1](r) \, dr    & \leq \int_{D_s}^{\xi_1^{**}} r \poptrpost{0}[t][1](r)  \, dr + \int_{\xi_1^{**}}^{\infty} r \poptrpost{0}[t][1](r) \, dr \label{eq:E_yout_functopt_1}     \\
        (\xi_2^{**})^2 \int_{D_s}^{\infty} \poptrpost{0}[t][2](r)  \, dr & \leq \int_{D_s}^{\xi_2^{**}} r^2 \poptrpost{0}[t][2](r)  \, dr + \int_{\xi_2^{**}}^{\infty} r^2 \poptrpost{0}[t][2](r) \, dr \label{eq:E_yout_functopt_2}
    \end{align}
    
    Recall that $\poptrpost{0}(r)$ is constructed such that its tail probability matches the sub-Gaussian bound $\exp(-r^2/\tau^2)$ exactly for $r > D_s$, and $\poptrpost{t}[0][1](r)$ and $\poptrpost{t}[0][2](r)$ are piecewise defined based on Gaussian functions (see~\cref{eq:Poptpost_t_0}). Therefore, their product can be analyzed by adapting standard results for products of Gaussian densities. As an example, for $R>\xi_1^{**}$, we consider $\tilde{p}^{R_1}$:
    \begin{align*}
    \tilde{P}^{R_1}_{0|t}(r>R) \propto &\int_{R}^{\infty} r \exp(-\frac{r^2}{\tau^2}) \exp(-\frac{(r-\Vert x_t - x^*\Vert )^2}{2t}) dr\\
    \propto & \int_{R}^{\infty} r \exp(-\frac{(r - r_t')^2}{\tau'^2})  dr \\
    \propto & (1 - \Phi(\frac{R - r_t'}{\tau'})) r_t' + \tau' \phi(\frac{R - r_t'}{\tau'})
    \end{align*}
   Here $\phi$ and $\Phi$ are the standard normal distribution's pdf and cdf, respectively, as detailed in \cref{lem:trunc_prod_gaussians}, and
   \[
       r_t' \;=\; \frac{\tau^2}{\tau^2 + 2t}\,\|x_t - x^*\|
       \quad\text{and}\quad
       \tau'^2 \;=\; \frac{2\tau^2 t}{\tau^2 + 2t}.
   \]
   Without loss of generality, we assume $x^*=0$ for simplicity in the following analysis, in which case $r_t' = \frac{\tau^2}{\tau^2 + 2t}\,\|x_t\|$. Furthermore, the above CDF satisfies the following property: for any $R' \geq R$,
    \[\frac{(1 - \Phi(\frac{R' - r_t'}{\tau'})) r_t' + \tau' \phi(\frac{R' - r_t'}{\tau'})}{(1 - \Phi(\frac{R - r_t'}{\tau'})) r_t' + \tau' \phi(\frac{R - r_t'}{\tau'})}\leq \frac{\exp(-\frac{(R' - r_t')^2}{\tau'^2})}{\exp(-\frac{(R - r_t')^2}{\tau'^2})}\]
     This analysis reveals that the tail of the resulting effective density $\poptrpost{0}[t][1](r)$ for $R>\xi_1^{**}$ is bounded by piecewise sub-Gaussian densities. We use $\SGRad{\mu}{\Sigma}$ to denote a density proportional to a sub-Gaussian with parameters $\mu$ and $\Sigma$. More precisely, the tail of $\poptrpost{0}[t][1](r)$ is bounded by the tail of $\SGRad{r_t'}{\tau'^2}$ for $R>\xi_1^{**}$.
     For $R<\xi_1^{**}$, we directly compare the density $\poptrpost{0}[t][1](r)$ at two different values $r$ and $r'$, with $r'>r$.
    \begin{equation}
    \frac{\poptrpost{0}[t][1](r')}{\poptrpost{0}[t][1](r)} \propto \frac{r' \exp(-\frac{(r'+r_t')^2}{\tau'^2})}{r \exp(-\frac{(r+r_t')^2}{\tau'^2})} \leq \frac{r' \exp(-\frac{(r')^2}{\tau'^2})}{r \exp(-\frac{(r)^2}{\tau'^2})}
    \label{eq:sub_g_tail}
    \end{equation}
     
\paragraph{2.2: Bounding Expectation with a heavy-tailed density.}
From the above tail bound, we can see that the tail of $\poptrpost{0}[t][1](r)$ is lighter-tailed than $\SGRad{0}{\tau'^2}$ for $R<\xi_1^{**}$. One can show that if $p_1(x)/p_1(y) \leq p_2(x)/p_2(y)$ holds, then $\mathbb{E}_{p_1}(x) < \mathbb{E}_{p_2}(x)$ holds.
     Consequently, we can bound $\mathbb{E}_{\poptrpost{0}[t][1]}{[r]}$ using the truncated moments of sub-Gaussian random variables derived in \cref{lemma:truancated_subg_moment}. 
     
     Define the truncated normalizing constants
     \[
     Z_1 := \int_{D_s}^{\xi_1^{**}} \SGRad{0}{\tau'^2}(u)\,du,
     \qquad
     Z_2 := \int_{\xi_1^{**}}^{\infty} \SGRad{r_t'}{\tau'^2}(u)\,du.
     \]
     Let $\tilde P_{0|t}^{R_1}$ denote the reference probability induced by $\poptrpost{0}[t][1]$.
     We define the comparison density $\hat p$ on $[D_s,\infty)$ as
     \begin{equation}\label{eq:hatp_def}
     \hat p_1(r)
     :=
     \tilde P_{0|t}^{R_1}\!\bigl(D_s \le r \le \xi_1^{**}\bigr)\,
     \frac{\SGRad{0}{\tau'^2}(r)\,\mathbf{1}\{D_s\le r\le \xi_1^{**}\}}{Z_1}
     \;+\;
     \tilde P_{0|t}^{R_1s}\!\bigl(r \ge \xi_1^{**}\bigr)\,
     \frac{\SGRad{r_t'}{\tau'^2}(r)\,\mathbf{1}\{r>\xi_1^{**}\}}{Z_2}.
     \end{equation}

    Therefore, the expectation of $\poptrpost{0}[t][1]$ can be bounded by the expectation of \cref{eq:hatp_def} since \cref{eq:sub_g_tail} is satisfied. Likewise, for $\poptrpost{0}[t][2]$ we can bound the expectation using
    \begin{equation}\label{eq:hatp_def2}
        \hat p_2(r)
        :=
        \tilde P_{0|t}^{R_2}\!\bigl(D_s \le r \le \xi_2^{**}\bigr)\,
        \frac{\SGRad{0}{\tau'^2}(r)\,\mathbf{1}\{D_s\le r\le \xi_2^{**}\}}{Z_1}
        \;+\;
        \tilde P_{0|t}^{R_2}\!\bigl(r \ge \xi_2^{**}\bigr)\,
        \frac{\SGRad{r_t'}{\tau'^2}(r)\,\mathbf{1}\{r>\xi_2^{**}\}}{Z_2}.
        \end{equation}

    This piecewise characterization constructs a comparison density $\hat p$ on $[D_s,\infty)$ from sub-Gaussian densities and ensures that it is more heavy-tailed than $\poptrpost{0}[t][1](r)$ and $\poptrpost{0}[t][2](r)$. We can then compute the required moments using \cref{lemma:truancated_subg_moment}. 
     \paragraph{2.3: Calculating the truncated moments.}
     With the above characterization, we now apply \cref{lemma:truancated_subg_moment}, which states that for a sub-Gaussian random variable $r$ with density $\SGRad{\mu}{\tau^2}$, we have: 
     \begin{align*}
        \mathbb{E}[(r - \mu)^2 | r > a] & \leq (a - \mu)^2 + \tau^2 \\
        \mathbb{E}[r - \mu | r > a]   & \leq a - \mu + \tau
    \end{align*}
    Expanding the left-hand side and rearranging terms, we obtain
    \begin{align}
        \mathbb{E}[r^2 | r > a] & \leq a ^2 + \tau^2 + 2 \mu \tau \label{eq:E_yout_non_central_1} \\
        \mathbb{E}[r | r > a]   & \leq a  + \tau \label{eq:E_yout_non_central_2}
    \end{align}
    Applying \cref{eq:E_yout_non_central_1,eq:E_yout_non_central_2} to \cref{eq:E_yout_functopt_1,eq:E_yout_functopt_2}, we obtain
    \begin{align}
        \xi_1^{**}(\int_{D_s}^{\infty} \poptrpost{0}[t][1](r) dr) \leq (\tau' + D_s) \int_{D_s}^{\xi_1^{**}} \poptrpost{0}[t][1](r) dr + (\tau' + \xi_1^{**}) \int_{\xi_1^{**}}^{\infty} \poptrpost{0}[t][1](r) dr \\
        (\xi_2^{**})^2(\int_{D_s}^{\infty} \poptrpost{0}[t][2](r) dr) \leq (D_s^2 + \tau'^2) \int_{D_s}^{\xi_2^{**}} \poptrpost{0}[t][2](r) dr + ((\xi_2^{**})^2 + \tau'^2 + 2 \tau' r_t') \int_{\xi_2^{**}}^{\infty} \poptrpost{0}[t][2](r) dr 
    \end{align}
    Dividing the first inequality by $\int_{D_s}^{\infty} \poptrpost{0}[t][1](r) dr$ and the second by $\int_{D_s}^{\infty} \poptrpost{0}[t][2](r) dr$, we obtain 
    \begin{align}
        (\xi_2^{**})^2 & \leq (D_s^2 + \tau'^2) + (\tau'^2 + 2 \tau' r_t') \frac{\int_{\xi_2^{**}}^{\infty} \poptrpost{0}[t][2](r) dr}{\int_{D_s}^{\xi_2^{**}} \poptrpost{0}[t][2](r) dr} \label{eq:non_symmetric_bound} \\
        \xi_1^{**}   & \leq \tau' + D_s+\tau' \frac{\int_{\xi_1^{**}}^{\infty} \poptrpost{0}[t][1](r) dr}{\int_{D_s}^{\xi_1^{**}} \poptrpost{0}[t][1](r) dr} \label{eq:non_symmetric_bound_2}
    \end{align}
    From \cref{eq:non_symmetric_bound,eq:non_symmetric_bound_2}, we see that the bounds on the truncated moments $\xi_1^{**}$ and $(\xi_2^{**})^2$ depend on the fractions involving integrals of the effective radial density $\poptrpost{0}[t][1](r)$ and $\poptrpost{0}[t][2](r)$:
    \[
        F_1 = \frac{\int_{\xi_1^{**}}^{\infty} \poptrpost{0}[t][1](r) dr}{\int_{D_s}^{\xi_1^{**}} \poptrpost{0}[t][1](r) dr}
        \quad \text{and} \quad
        F_2 = \frac{\int_{\xi_2^{**}}^{\infty} \poptrpost{0}[t][2](r) dr}{\int_{D_s}^{\xi_2^{**}} \poptrpost{0}[t][2](r) dr}.
    \]
    \paragraph{2.4: Bounding the fractions $F_1$ and $F_2$.}
    We now bound these fractions by analyzing cases. We first focus on the second moment, $(\xi_2^{**})^2 = \EbtOutTnorm{0}{t}$. We consider two cases based on the relationship between $\xi_2^{**} - r_t'$ and $D_s + r_t'$.

    \textbf{Case 1: $\xi_2^{**} - r_t' \geq D_s + r_t' + \tau'$.} 
    In this case, we bound the fraction $F_2$. The analysis relies on bounding the numerator and denominator integrals, assuming the integrand $\poptrpost{0}[t](r)$ behaves similarly to $r \exp(-\frac{(r+r_t')^2}{\tau'^2})$ in the relevant ranges.
    The denominator integral is bounded below:
    \begin{align*}
        \int_{D_s}^{\xi_2^{**}}\poptrpost{0}[t](r) dr
        \geq C_{\mathrm{lower}} \int_{D_s}^{\infty} r \exp(-\frac{(r+r_t')^2}{\tau'^2}) dr
        = C_{\mathrm{lower}} (\tau' \phi(\frac{D_s + r_t'}{\tau'}) - r_t' (1 - \Phi(\frac{D_s + r_t'}{\tau'}))),
    \end{align*}
    The existence of a positive constant factor $C_{\mathrm{lower}}\geq (1-1/e)$ is guaranteed because the integration interval $[D_s, \xi_2^{**}]$ contains at least $[D_s, D_s + \tau']$ (due to the properties of sub-Gaussian tails). This ensures that a non-vanishing fraction of the integral from $D_s$ to infinity is captured.
    The numerator integral is calculated using results similar to those in \cref{lem:trunc_prod_gaussians}:
    \begin{align*}
        \int_{\xi_2^{**}}^{\infty}\poptrpost{0}[t](r) dr 
        \propto \tau' \phi(\frac{\xi_2^{**} - r_t'}{\tau'}) + r_t' (1 - \Phi(\frac{\xi_2^{**} - r_t'}{\tau'})).
    \end{align*}
   Considering the case $r_t' \leq \tau'$, we can then bound $F_2$ as follows: 
    \begin{align}
        F_2 &\leq \frac{\tau' \phi(\frac{\xi_2^{**} - r_t'}{\tau'}) + \tau' (1 - \Phi(\frac{\xi_2^{**} - r_t'}{\tau'}))}{C_{\mathrm{lower}} (\tau' \phi(\frac{D_s + r_t'}{\tau'}) - \tau' (1 - \Phi(\frac{D_s + r_t'}{\tau'})))} \nonumber \\
        &=\frac{ \phi(\frac{\xi_2^{**} - r_t'}{\tau'}) +  (1 - \Phi(\frac{\xi_2^{**} - r_t'}{\tau'}))}{C_{\mathrm{lower}} ( \phi(\frac{D_s + r_t'}{\tau'}) - (1 - \Phi(\frac{D_s + r_t'}{\tau'})))} \nonumber \\
        & \leq \frac{1/e}{1-1/e}\frac{ \phi(\frac{D_s + r_t'}{\tau'}) +  (1 - \Phi(\frac{D_s + r_t'}{\tau'}))}{( \phi(\frac{D_s + r_t'}{\tau'}) - (1 - \Phi(\frac{D_s + r_t'}{\tau'})))} \label{eq:F_2_bound}
    \end{align}
    This inequality is derived from the condition $\xi_2^{**} - r_t' \geq D_s + r_t' + \tau'$, which implies $\frac{\xi_2^{**} - r_t'}{\tau'} \ge \frac{D_s + r_t'}{\tau'} + 1$. We leverage the decay properties of the Gaussian tail function $\phi(z)$ for arguments separated by at least 1, along with the established lower bound $C_{\mathrm{lower}}\geq (1-1/e)$. One can verify that $\frac{\phi(z)+\Phi_C(z)}{\phi(z)-\Phi_C(z)}$ is monotonically decreasing when $z>1$. Through numerical calculation, we obtain
    \[
         F_2 = \frac{\int_{\xi_2^{**}}^{\infty}\poptrpost{0}[t](r) dr}{\int_{D_s}^{\xi_2^{**}}\poptrpost{0}[t](r) dr} \leq \frac{3}{2}.
    \]
    Substituting this bound into \cref{eq:non_symmetric_bound}, we get:
    \[
        (\xi_2^{**})^2 \leq (D_s^2 + \tau'^2)+ (\tau'^2 + 2 \tau' r_t')\frac{3}{2} = D_s^2 + \frac{5}{2} \tau'^2 + 3 \tau' r_t'.
    \]

    \textbf{Case 2: $\xi_2^{**} - r_t' < D_s + r_t' + \tau'$.} 
    In this case, we have $\xi_2^{**} \leq D_s + 2r_t' + \tau'$. This yields a direct bound on the squared moment:
    \[
        (\xi_2^{**})^2 \leq (D_s + 2r_t' + \tau')^2.
    \]
    We now perform a similar case analysis for the first moment, $\xi_1^{**} = \EbtOutnorm{0}{t}$, using \cref{eq:non_symmetric_bound_2}. 

    \textbf{Case 1: $\xi_1^{**} - r_t' < D_s + r_t' + \tau'$.} 
    In this case, we have $\EbtOutnorm{0}{t} \leq D_s + 2 r_t' + \tau'$. 
    
    \textbf{Case 2: $\xi_1^{**} - r_t' \geq D_s + r_t' + \tau'$.} 
    In this case, we have $\EbtOutnorm{0}{t} \leq D_s + 3\tau'$ as $F_1$ can be caculated in a similar way as in \cref{eq:F_2_bound}.

    \textbf{Step 3: Combined Bound for Moments.} 
    Combining the results from both cases for the second moment yields:
    \begin{equation}
        \EbtOutTnorm{0}{t} = (\xi_2^{**})^2 \leq \max \left( (D_s+ 2 r_t' + \tau')^2, \; D_s^2 + \tfrac{5}{2} \tau'^2 + 3 \tau' r_t' \right).
    \end{equation}

    Combining the results from both cases for the first moment gives:
    \begin{equation}
        \EbtOutnorm{0}{t} = \xi_1^{**} \leq \max \left( D_s + 2 r_t' + \tau', \; D_s + 3\tau' \right).
    \end{equation}
This concludes the derivation of the moment bounds stated in Lemma~\ref{lem:sub_gauss_region_moment_bound}.

\end{proof}
Based on the above lemmas, we are ready to prove \Cref{theorem:SCBound} (restated below). 
\SCBound*
\begin{proof}
    Recall the Tweedie-form Hessian identity (see \cref{eq:second_order_tweedie_prop}):
    \[
        \nabla_x^2 g(x;t)
        = \frac{\lambda}{t^2}\Bigl[t\,\mathbf{I}-\CovbtAll{0}{t}\Bigr],
        \qquad
        \CovbtAll{0}{t}:=\mathrm{Cov}_{[0|t]}[y\mid x_t=x].
    \]
    Hence, to prove \(\tfrac{C_\alpha\lambda}{t+\lambda/\alpha}\)-strong convexity of \(g(\cdot;t)\) at \(x\), it suffices to show
    \begin{equation}
        t\,\mathbf{I}-\CovbtAll{0}{t}
        \succeq \frac{C_\alpha t^2}{t+\lambda/\alpha}\,\mathbf{I}.
        \label{eq:second_order_alpha_strong}
    \end{equation}

    \noindent\textbf{Step 1: Variance decomposition for the conditional distribution.}
    Consider the conditional distribution \(Y \sim p_{0|t}(\cdot \mid x_t)\) and the event \(\{Y \in \BallTau{}\}\). Applying the law of total variance (see \cref{lemma:variance_fraction}) to this conditional distribution, we have
    \begin{align*}
        \CovbtAll{0}{t}
        =\;&\PbtIn{0}{t}\,\CovbtIn{0}{t}
           +\PbtOut{0}{t}\,\CovbtOut{0}{t} \\
         &+\PbtIn{0}{t}\PbtOut{0}{t}
           \bigl(\mu_{\mathrm{out}}-\mu_{\mathrm{in}}\bigr)\bigl(\mu_{\mathrm{out}}-\mu_{\mathrm{in}}\bigr)^\top,
    \end{align*}
    where
    \[
        \PbtIn{0}{t}:=P_{0|t}(y \in \BallTau{}\mid x_t),\quad
        \PbtOut{0}{t}:=P_{0|t}(y \notin \BallTau{}\mid x_t),
    \]
    and \(\mu_{\mathrm{in}}:=\mathbb{E}_{[0|t]}[y \mid x_t, y \in \BallTau{}]\), \(\mu_{\mathrm{out}}:=\mathbb{E}_{[0|t]}[y \mid x_t, y \notin \BallTau{}]\) are the conditional means.

    \medskip
    \noindent\textbf{Step 2: Upper bounding the contribution from the non-convex region.}
    Let \(\Sigma_{\mathrm{in}}:=\CovbtIn{0}{t}\) be the covariance within the convex region. We now bound the remaining terms in the decomposition. First, note that for any vector \(z\), \(z z^\top \preceq \|z\|^2 \mathbf{I}\). Combined with the property that the covariance matrix of any random vector \(Z\) satisfies \(\mathrm{Cov}(Z) = \mathbb{E}[Z Z^\top] - \mathbb{E}[Z]\mathbb{E}[Z]^\top \preceq \mathbb{E}[Z^\top Z] \mathbf{I}\), we obtain
    \begin{align*}
        \CovbtOut{0}{t}
        &=\mathrm{Cov}_{[0|t]}[y-x^*\mid x_t, y \notin \BallTau{}] \\
        &\preceq \mathbb{E}_{[0|t]}\bigl[ \|y-x^*\|^2 \mid x_t, y \notin \BallTau{} \bigr] \mathbf{I}
        = \EbtOutTnorm{0}{t} \mathbf{I}.
    \end{align*}
    Second, since \(y \in \BallTau{}=\{y:\|y-x^*\|<D_\tau\}\) implies \(\|\mathbb{E}[y-x^*\mid x_t, y \in \BallTau{}]\| \le D_\tau\), the distance between the conditional means satisfies:
    \[
        \|\mu_{\mathrm{out}}-\mu_{\mathrm{in}}\|
        = \bigl\| \mathbb{E}_{[0|t]}[y-x^*\mid x_t, y \notin \BallTau{}] - \mathbb{E}_{[0|t]}[y-x^*\mid x_t, y \in \BallTau{}] \bigr\|
        \le \EbtOutnorm{0}{t} + D_\tau.
    \]
    Consequently, the cross term in the variance decomposition is bounded by
    \[
    \resizebox{0.96\textwidth}{!}{$
    \begin{aligned}
        &\PbtIn{0}{t} \PbtOut{0}{t}
        \bigl(\mu_{\mathrm{out}}-\mu_{\mathrm{in}}\bigr)
        \bigl(\mu_{\mathrm{out}}-\mu_{\mathrm{in}}\bigr)^\top \\
        &\qquad \preceq \PbtOut{0}{t} (\EbtOutnorm{0}{t} + D_\tau)^2 \mathbf{I}
        \preceq \PbtOut{0}{t}
        (\EbtOutTnorm{0}{t} + 2 D_\tau \EbtOutnorm{0}{t} + D_\tau^2) \mathbf{I}.
    \end{aligned}$}
    \]
    where the second inequality uses Jensen's inequality \((\EbtOutnorm{0}{t})^2 \le \EbtOutTnorm{0}{t}\). Combining these bounds, we arrive at
    \begin{align}
        \CovbtAll{0}{t} &\preceq \PbtIn{0}{t} \Sigma_{\mathrm{in}} + \PbtOut{0}{t} \momout \mathbf{I}, \nonumber \\
        \momout &:= 2 \EbtOutTnorm{0}{t} + 2 D_\tau \EbtOutnorm{0}{t} + D_\tau^2. \label{eq:cov_decomp_bound}
    \end{align}

    \medskip
    \noindent\textbf{Step 3: Reduction of strong convexity to a bound on tail moments.}
    Substituting the bound from \cref{eq:cov_decomp_bound} into the strong convexity requirement \cref{eq:second_order_alpha_strong}, it is sufficient to show that
    \[
        t \mathbf{I} - \PbtIn{0}{t} \Sigma_{\mathrm{in}} - \PbtOut{0}{t} \momout \mathbf{I}
        \succeq \frac{C_\alpha t^2}{t+\lambda/\alpha} \mathbf{I}.
    \]
    Applying \cref{lem:convex_region_covariance_bound}, which states \(\Sigma_{\mathrm{in}} \preceq \frac{t(\lambda/\alpha)}{t+\lambda/\alpha} \mathbf{I}\), and noting that \(\PbtIn{0}{t} \le 1\), we find it sufficient to require
    \[
        \left( t - \frac{t(\lambda/\alpha)}{t+\lambda/\alpha} \right) \mathbf{I} - \PbtOut{0}{t} \momout \mathbf{I}
        \succeq \frac{C_\alpha t^2}{t+\lambda/\alpha} \mathbf{I}.
    \]
    Since \(t - \frac{t(\lambda/\alpha)}{t+\lambda/\alpha} = \frac{t^2}{t+\lambda/\alpha}\), this simplifies to the following sufficient condition:
    \begin{equation}
        \PbtOut{0}{t} \momout \le \frac{(1-C_\alpha) t^2}{t+\lambda/\alpha}.
        \label{eq:Pout_Mout_constraint}
    \end{equation}
    Thus, proving \(\tfrac{C_\alpha\lambda}{t+\lambda/\alpha}\)-strong convexity reduces to bounding (i) the conditional tail mass \(\PbtOut{0}{t}\) and (ii) the outside-moment term \(\momout\) so that \cref{eq:Pout_Mout_constraint} holds.

    \medskip
    \noindent\textbf{Step 4: Upper bound on \(\momout\).}
    We introduce unified "effective mean/variance" notation that mirrors the product-of-Gaussians parameters used in Step~5.
    Define
    \begin{equation}
        \mu_{\mathrm{in}}'
        \;:=\;
        \frac{\frac{\lambda}{\beta}}{\frac{\lambda}{\beta}+2t}\,x_t
        +\frac{t}{\frac{\lambda}{\beta}+2t}\,x^*,
        \qquad
        (\sigma_{\mathrm{in}}')^2
        \;:=\;
        \frac{t\frac{\lambda}{\beta}}{\frac{\lambda}{\beta}+2t},
    \label{eq:mu_in_prime}
    \end{equation}
    and
    \begin{equation}
        \mu_{\mathrm{out}}'
        \;:=\;
        \frac{\tau^2}{\tau^2+2t}\,x_t+\frac{t}{\tau^2+2t}\,x^*,
        \qquad
        (\sigma_{\mathrm{out}}')^2:=\frac{t\tau^2}{t+\tau^2}.
    \label{eq:mu_out_prime}
    \end{equation}
    Note that \((\sigma_{\mathrm{out}}')^2=\tau'^2\), where
    \[
        \tau'^2 \;:=\; \frac{t\tau^2}{2t+\tau^2}.
    \]
    Define the scalar shrinkage radius
    \[
        r_t' \;:=\; \|\mu_{\mathrm{out}}' - x^*\|
        \;=\; \frac{\tau^2}{2t+\tau^2}\,\|x_t-x^*\|.
    \]
    Also as in \cref{tab:notation}, we have that
    \[
        D_s \;:=\; \tau\sqrt{\log\!\Bigl(\tfrac{1}{\Pout{}}\Bigr)}.
    \]
    Lemma~\ref{lem:sub_gauss_region_moment_bound} gives
    \begin{align*}
        \EbtOutTnorm{0}{t}
        & \leq \max\!\left((D_s+2r_t')^2,\; D_s^2 + 3 \tau'^2 + 4 \tau' r_t' \right), \\
        \EbtOutnorm{0}{t}
        & \leq \max\!\left(D_s+2r_t',\; D_s + 3\tau' \right).
    \end{align*}
    To simplify the algebra, we upper bound the maxima by slightly looser expressions.
    Since \(D_s\ge \tau'\) and \(r_t'\ge 0\), we have
    \(4\tau'r_t' \le 4D_s r_t' \le 4D_s r_t' + 4(r_t')^2\), hence
    \[
        D_s^2 + 3\tau'^2 + 4\tau'r_t'
        \le (D_s+2r_t')^2 + 3\tau'^2,
    \]
    which implies
    \begin{align*}
        \EbtOutTnorm{0}{t} & \le (D_s+2r_t')^2 + 3\tau'^2, \\
        \EbtOutnorm{0}{t}  & \le D_s + 2r_t' + 3\tau'.
    \end{align*}
    Substituting these into \(\momout=2\EbtOutTnorm{0}{t}+2D_\tau\EbtOutnorm{0}{t}+D_\tau^2\), we obtain
    \begin{align}
        \momout
        & \leq 2\bigl[(D_s + 2 r_t' )^2 + 3\tau'^2\bigr]
            + 2 D_\tau \bigl[D_s + 2 r_t' + 3 \tau'\bigr] + D_\tau^2 \nonumber \\
        & = 2D_s^2 + 8 D_s r_t' + 8 (r_t')^2 + 6\tau'^2
            + 2 D_\tau D_s + 4 D_\tau r_t' + 6 D_\tau \tau' + D_\tau^2.
        \label{eq:M_out_bound}
    \end{align}
    In particular, \(\momout\) is controlled by \((D_\tau,\tau,\Pout{})\) and \(\|x_t-x^*\|\) through \(r_t'\).

    \medskip
   \paragraph{Simplifications under the two distance regimes.}

    \smallskip
    \paragraph{Case 1 (\(\|x_t-x^*\|\le \tfrac{1}{3}D_\tau\)).}
    Using \(r_t'\le \|x_t-x^*\|\le \tfrac{1}{3}D_\tau\) and \(\tau'\le \tau\), the bound \cref{eq:M_out_bound} implies
    \begin{align}
        \momout
        &\le 2D_s^2 + 8 D_s r_t' + 8 (r_t')^2 + 6\tau'^2
        + 2 D_\tau D_s + 4 D_\tau r_t' + 6 D_\tau \tau' + D_\tau^2 \nonumber\\
        &\le 2 D_s^2 + 5 D_s D_\tau + 5 D_\tau^2
        \;\le\; 12 \max(D_\tau^2, D_s^2),
        \label{eq:M_out_bound_case1_simpl}
    \end{align}
    for all sufficiently large \(d\) (since \(D_\tau\sim \Theta(d)\) while \(\tau=\Theta(1)\)).

    \smallskip
    \paragraph{Case 2 (\(\|x_t-x^*\|> \tfrac{1}{3}D_\tau\)).}
    Recall \cref{eq:M_out_bound}. Under the expansion-rate bounds used later (so that \(r_t'\le C_E D_\tau\) and \(\tau'\le\tau\)), we have
    \begin{align}
        \momout
        &\leq 2 D_s^2 + 8 C_E D_s D_\tau +8 C_E^2 D_\tau^2 + 6\tau^2
        + 2 D_s D_\tau + 4 C_E D_\tau^2 + 6 D_\tau \tau + D_\tau^2 \nonumber\\
        &\le 2 D_s^2 + 3 D_s D_\tau + 4 D_\tau^2
        \;\le\; 9 \max(D_\tau^2, D_s^2),
        \label{eq:M_out_bound_case2_simpl}
    \end{align}
    for all sufficiently large \(d\) (since \(C_E\sim \Theta(\sqrt{\log d/d})\), \(D_\tau^2\sim \Theta(d)\), and \(\tau=\Theta(1)\)).

    The next step is to bound \(\PbtOut{0}{t}\) and verify that \(\PbtOut{0}{t}\momout\) satisfies \cref{eq:Pout_Mout_constraint}.

    \noindent\textbf{Step 5: Upper bound on \(\PbtOut{0}{t}\).}
    Since \(p_{0|t}(y\mid x_t)\propto p_0(y)\,p_{t\mid 0}(x_t\mid y)\), define the unnormalized masses
    \[
        Z_{\mathrm{in}}(x_t):=\int_{y\in\BallTau{}} p_0(y)\,p_{t\mid 0}(x_t\mid y)\,dy,
        \qquad
        Z_{\mathrm{out}}(x_t):=\int_{y\notin\BallTau{}} p_0(y)\,p_{t\mid 0}(x_t\mid y)\,dy.
    \]
    Then \(\PbtOut{0}{t}/\PbtIn{0}{t}=Z_{\mathrm{out}}(x_t)/Z_{\mathrm{in}}(x_t)\). We next upper bound \(Z_{\mathrm{out}}\) and lower bound \(Z_{\mathrm{in}}\) to obtain bounds on \(\PbtOut{0}{t}\) and \(\PbtIn{0}{t}\).
    \medskip
    
    \noindent\textbf{Upper bound on \(Z_{\mathrm{out}}(x_t)\).}
    By definition of \(\BtauDist{x_t}:=\min_{y\notin\BallTau{}}\|y-x_t\|\), for all \(y\notin\BallTau{}\),
    \[
        p_{t|0}(x_t\mid y)=(2\pi t)^{-d/2}\exp\!\left(-\frac{\|x_t-y\|^2}{2t}\right)
        \le (2\pi t)^{-d/2}\exp\!\left(-\frac{\BtauDist{x_t}^2}{2t}\right).
    \]
    Therefore, using \(\int_{y\notin\BallTau{}}p_0(y)\,dy=\Pout{}\),
    \[
        Z_{\mathrm{out}}(x_t)
        \le \Pout{}\,(2\pi t)^{-d/2}\exp\!\left(-\frac{\BtauDist{x_t}^2}{2t}\right).
    \]

    \medskip
    \noindent\textbf{Lower bound on \(Z_{\mathrm{in}}(x_t)\).}
    Since \(f\) is \(\alpha\)-strongly convex and \(\beta\)-smooth on \(\BallTau{}\), and \(\nabla f(x^*)=0\), for all \(y\in\BallTau{}\),
    \[
        \frac{\alpha}{2}\|y-x^*\|^2 \le f(y)-f(x^*) \le \frac{\beta}{2}\|y-x^*\|^2.
    \]
    Recalling \(p_0(y)=p_0(x^*)\exp(-(f(y)-f(x^*))/\lambda)\), this implies for all \(y\in\BallTau{}\),
    \[
        p_0(x^*)\exp\!\left(-\frac{\beta}{2\lambda}\|y-x^*\|^2\right)
        \le p_0(y)
        \le p_0(x^*)\exp\!\left(-\frac{\alpha}{2\lambda}\|y-x^*\|^2\right).
    \]
    Using \(\Pin{}=\int_{y\in\BallTau{}}p_0(y)\,dy\) and \(\int_{\mathbb{R}^d}\exp(-\frac{\alpha}{2\lambda}\|u\|^2)\,du=(2\pi\lambda/\alpha)^{d/2}\), we obtain
    \[
        p_0(x^*) \ge \Pin{}\left(\frac{\alpha}{2\pi\lambda}\right)^{d/2}.
    \]
    For \(y\in\BallTau{}\), we have \(p_0(y)\ge p_0(x^*)\exp\!\left(-\frac{\beta}{2\lambda}\|y-x^*\|^2\right)\), and
    \(p_{t|0}(x_t\mid y)=\mathcal{N}(y\mid x_t,tI)\).
    By the product-of-Gaussians identity (with \(\mu_{\mathrm{in}}'\) and \((\sigma_{\mathrm{in}}')^2\) as defined in \cref{eq:mu_in_prime}), we have
    \[
        \mathcal{N}(y\mid x_t,tI)\exp\!\left(-\frac{\beta}{2\lambda}\|y-x^*\|^2\right)
        = \exp\!\left(-\frac{\beta}{2(Lt+\lambda)}\|x_t-x^*\|^2\right)\left(\frac{Lt+\lambda}{\lambda}\right)^{-d/2}
        \mathcal{N}(y\mid \mu_{\mathrm{in}}',(\sigma_{\mathrm{in}}')^2 I).
    \]
    Integrating the above bound over \(y\in\BallTau{}\), we obtain
    \begin{align*}
        Z_{\mathrm{in}}(x_t)
        &=\int_{y\in\BallTau{}} p_0(y)\,p_{t\mid 0}(x_t\mid y)\,dy \\
        &\geq
        p_0(x^*)\exp\!\left(-\frac{\beta}{2(\beta t + \lambda)} \|x^* - x_t\|^2\right)
        \left(\frac{\beta t + \lambda}{\lambda}\right)^{-d/2} \\
        &\quad \cdot
        \int_{y\in\BallTau{}}
        \mathcal{N}\left(y \mid \mu_{\mathrm{in}}', (\sigma_{\mathrm{in}}')^2 I\right) dy.
    \end{align*}
    Combining the previous display with \(p_0(x^*) \ge \Pin{}\left(\frac{\alpha}{2\pi\lambda}\right)^{d/2}\), we obtain
    \[
        Z_{\mathrm{in}}(x_t)
        \ge \Pin{}\left(\frac{\alpha}{2\pi}\right)^{d/2}(\beta t + \lambda)^{-d/2}
        \exp\!\left(-\frac{\beta}{2(\beta t + \lambda)} \|x^* - x_t\|^2\right)
        \int_{y\in\BallTau{}} \mathcal{N}\left(y \mid \mu_{\mathrm{in}}', (\sigma_{\mathrm{in}}')^2 I\right) dy.
    \]
    Therefore,
    \begin{align}
        \frac{\PbtOut{0}{t}}{\PbtIn{0}{t}}
        &= \frac{\int_{y \notin \BallTau{}} p_0(y) p_{t\mid 0}(x_t \mid y)\,dy}{\int_{y \in \BallTau{}} p_0(y) p_{t\mid 0}(x_t \mid y)\,dy} \nonumber\\
        &\leq \frac{\Pout{}\exp\left(-\frac{1}{2t} \BtauDist{x_t}^2\right)}{\Pin{}
            \exp\left(-\frac{\beta}{2(\beta t + \lambda)} \|x^* - x_t\|^2\right)
            \left(\frac{t\alpha }{\beta t + \lambda}\right)^{d/2}
            \int_{y\in\BallTau{}} \mathcal{N}\left(y \mid \mu_{\mathrm{in}}', (\sigma_{\mathrm{in}}')^2 I\right) dy.}
        \label{eq:PbtOut_bound}
\end{align}

\noindent\textbf{Lower bound on $\int_{y\in\BallTau{}} \mathcal{N}\left(y \mid \mu_{\mathrm{in}}', (\sigma_{\mathrm{in}}')^2 I\right) dy$.}
    We now lower bound the Gaussian mass term $\int_{y\in\BallTau{}} \mathcal{N}\left(y \mid \mu_{\mathrm{in}}', (\sigma_{\mathrm{in}}')^2 I\right) dy$ appearing in \cref{eq:PbtOut_bound}.
    As stated in the theorem, $x_t\in \mathcal{R}_{SC}(t)$, i.e.
\begin{equation}
    \|x^* - x_t\|  \leq C_E \min({\sqrt{\frac{\frac{\lambda}{\beta} +t}{\frac{\lambda}{\beta} }} ,   \sqrt{\frac{\tau^2 +t}{\tau^2}}})D_\tau\label{eq:CE_bound}
\end{equation}
This also implies the corresponding bound on the effective mean \(\mu_{\mathrm{in}}'\):
\[
    \|\mu_{\mathrm{in}}' - x^*\|
    = \frac{\frac{\lambda}{\beta}}{t+\frac{\lambda}{\beta}}\|x_t-x^*\|
    \leq \|x_t-x^*\|
    \leq C_E D_\tau.
\]
Moreover,
\begin{equation}
    \|\mu_{\mathrm{in}}' - x^*\|
    = \|x_t - x^*\|\frac{\frac{\lambda}{\beta}}{t+\frac{\lambda}{\beta}}
    \leq C_E \sqrt{\frac{\frac{\lambda}{\beta}}{\frac{\lambda}{\beta}+t }}\,D_\tau
    \leq C_E D_\tau,
    \label{eq:x_t_prime_bound}
\end{equation}
Most importantly, from our assumption, we have 
\begin{equation}
\frac{\beta}{Lt + \lambda}\| x_t - x^*\|^2 \leq C_E^2 \tfrac{\beta}{\lambda} D_\tau^2 \label{eq:x_t_bound_ce}
\end{equation}
Note that
\[
    \int_{\BallTau{}} \mathcal{N}\left(y \mid \mu_{\mathrm{in}}', (\sigma_{\mathrm{in}}')^2 I\right) dy
    = P\!\left(\|Z+\mu_{\mathrm{in}}'-x^*\|\le D_\tau\right)
    \ge P\!\left(\|Z\|\le D_\tau-\|\mu_{\mathrm{in}}'-x^*\|\right),
\]
where \(Z\sim \mathcal{N}(0,(\sigma_{\mathrm{in}}')^2 I)\). In particular, if we choose \(C_E\le \tfrac{1}{3}\), then \cref{eq:x_t_prime_bound} implies \(\|\mu_{\mathrm{in}}'-x^*\|\le \tfrac{1}{3}D_\tau \), and hence \(D_\tau-\|\mu_{\mathrm{in}}'-x^*\|\ge \tfrac{2}{3}D_\tau\). Therefore,
\begin{align}
    \int_{\BallTau{}} \mathcal{N}\left(y \mid \mu_{\mathrm{in}}', (\sigma_{\mathrm{in}}')^2 I\right) dy
    &\ge P_{Z\sim \mathcal{N}(0,(\sigma_{\mathrm{in}}')^2 I)}\!\left(\|Z\|\le \tfrac{2}{3}D_\tau\right)\nonumber\\
    &\ge P_{Z\sim \mathcal{N}(0,I)}\!\left(\|Z\|^2\le \frac{4 D_\tau^2}{9(\sigma_{\mathrm{in}}')^2}\right)
    \ge P_{Z\sim \mathcal{N}(0,I)}\!\left(\|Z\|^2\le \frac{4 D_\tau^2}{9}\frac{\beta}{\lambda}\right) \nonumber\\
    &= \frac{\gamma(\frac{d}{2},\frac{2D_\tau^2}{9}\frac{\beta}{\lambda})}{\Gamma(\frac{d}{2})}.
    \label{eq:Dx_prob_lb}
\end{align}
We assume that $\lambda$ is upper bounded by a fixed constant $\lambda_{\max}>0$ (a practical modeling choice since $\lambda\to\infty$ makes $p_0$ nearly uniform and removes learning signal). Then the incomplete-gamma argument satisfies
\[
    \frac{2D_\tau^2}{9}\frac{\beta}{\lambda}
    \;\ge\;
    \frac{2D_\tau^2}{9}\frac{\beta}{\lambda_{\max}}.
\]
We will require:
\begin{equation}
    \frac{2D_\tau^2}{9}\frac{\beta}{\lambda_{\max}}\sim \Theta(d) \geq \tfrac{d}{2}
    \label{eq:key_Dtau_scaling}
\end{equation}
Under \cref{eq:key_Dtau_scaling}, the lower-tail probability $\frac{\gamma(\frac{d}{2},\frac{2D_\tau^2}{9}\frac{\beta}{\lambda})}{\Gamma(\frac{d}{2})}$ is bounded below by a universal constant (e.g., $\ge \tfrac{1}{2}$ for all sufficiently large $d$) by standard Chernoff bounds for $\Gamma(\tfrac{d}{2},1)$.

    Plugging \cref{eq:Dx_prob_lb} into \cref{eq:PbtOut_bound} and using $\PbtIn{0}{t}\leq 1$, and as given in the assumption that $\Pout{}\leq \Pin{}$, we have
\begin{align}
   \PbtOut{0}{t} & \leq  2 \Pout{} \frac{\Gamma(\tfrac{d}{2})}{\gamma(\tfrac{d}{2},\tfrac{2D_\tau^2 \beta}{9 \lambda})} \frac{\exp\left(-\frac{1}{2t} \BtauDist{x_t}^2\right)}{
\exp\left(-\frac{\beta}{2(\beta t + \lambda)} \|x^* - x_t\|^2\right) } (\frac{Lt + \lambda}{\alpha t})^{d/2} \nonumber \\
&\leq 4 \Pout{} \frac{\exp\left(-\frac{1}{2t} \BtauDist{x_t}^2\right)}{
\exp\left(-\frac{1}{2} C_E^2 \frac{\beta}{\lambda} D_\tau^2 \right) } (\frac{Lt + \lambda}{\alpha t})^{d/2}
\label{eq:PbtOut_bound_3}
\end{align}
where the second inequality comes from \cref{eq:x_t_bound_ce}.

\medskip
\noindent\emph{A second (geometric) bound.}
Define the inside-distance
\[
    \BinDist{x_t}:=\min_{y\in\BallTau{}}\|y-x_t\|.
\]
Then, by lower bounding the inside mass with \(\exp(-\BinDist{x_t}^2/(2t))\) and upper bounding the outside mass with \(\exp(-\BtauDist{x_t}^2/(2t))\), we also have
\begin{align}
    \frac{\PbtOut{0}{t}}{\PbtIn{0}{t}}
    &\leq \frac{\Pout{}  \exp\left(-\frac{1}{2t}\BtauDist{x_t}^2\right)}{\Pin{} \exp\left(-\frac{1}{2t} \BinDist{x_t}^2\right)}.
    \label{eq:PbtOut_bound_2}
\end{align}

    \medskip
    \noindent\textbf{Case analysis for $\PbtOut{0}{t}$.}

    \smallskip
    \paragraph{Case 1: $\|x_t - x^*\| \leq \tfrac{1}{3}D_\tau$.}
    In this case, we have two bounds. First, using \cref{eq:PbtOut_bound_3}, we obtain
\begin{align}
    \PbtOut{0}{t} \leq 4 \Pout{} \frac{\exp\left(- \tfrac{2 D_\tau^2}{9t}\right)}{
        \exp\left(-\frac{1}{2} C_E^2 \frac{\beta}{\lambda} D_\tau^2 \right) } (\frac{Lt + \lambda}{\alpha t})^{d/2} 
 \label{eq:Pbtout_upper_bound_small_1}
 \end{align}
    Second, using \cref{eq:PbtOut_bound_2}, we obtain
\begin{align}
    \PbtOut{0}{t} 
 &  \leq  \frac{\Pout{}  \exp\left(-\frac{1}{2t}\BtauDist{x_t}^2\right)}{\Pin{} \exp\left(-\frac{1}{2t} \BinDist{x_t}^2\right)}  \nonumber \\
 & \leq 2 \frac{\Pout{}  \exp\left(-\frac{2 D_\tau^2}{9t}\right)}{ P_0(\| x- x^*\| \leq \tfrac{1}{3}D_\tau) \exp\left(-\frac{2D_\tau^2}{9} \right)} 
  \label{eq:Pbtout_upper_bound_1}
 \end{align}
    where we have used $\BtauDist{x_t}^2 \geq \frac{4}{9} D_\tau^2$. Since $D_\tau \sim \Theta(d)$ while $\alpha$ and $\beta$ are constants, $P(\| x- x^*\| \leq \tfrac{1}{3}D_\tau) \geq 1/2 \Pin{}$ holds for all sufficiently large $d$.
    Therefore,
 \begin{equation}
    \PbtOut{0}{t}  \leq  4 \Pout{} \label{eq:Pbtout_upper_bound_small_2}
 \end{equation}
    \paragraph{Case 2: $\|x_t - x^*\| > \tfrac{1}{3}D_\tau$.}
    In this case, we upper bound the probability as follows:
 \begin{align}
    \PbtOut{0}{t} & \leq  \frac{ \Pout{}}{\Pin{} \exp(-\tfrac{1}{2t}(D_\tau + \|x_t-x^*\|)^2)}  \\
    &\leq \frac{2 \Pout{}}{\exp(-\tfrac{8}{t}(\|x_t-x^*\|)^2)} \\
    &\leq 2\Pout{}\exp(8\frac{\tfrac{\lambda}{\beta}+t}{\tfrac{\lambda}{\beta}t}C_E^2D_\tau^2) \\
    & \leq2\Pout{}\exp(8\tfrac{\lambda}{\beta}C_E^2D_\tau^2),
 \label{eq:Pbtout_upper_bound_large_1}
 \end{align}
    where the second inequality uses $\|x_t-x^*\|\ge \tfrac{1}{3}D_\tau$ (so $D_\tau+\|x_t-x^*\|\le 4\|x_t-x^*\|$), and the third inequality uses the convex shrinkage/region condition \cref{eq:CE_bound}.

    \medskip
    \noindent\textbf{Step 6: Putting the bounds together.}
    Next, we plug in the upper bounds for $\momout$ from \cref{eq:M_out_bound_case1_simpl,eq:M_out_bound_case2_simpl} and for $\PbtOut{0}{t}$ from \cref{eq:Pbtout_upper_bound_small_1,eq:Pbtout_upper_bound_small_2,eq:Pbtout_upper_bound_large_1} under different conditions into the sufficient condition \cref{eq:Pout_Mout_constraint} to achieve the strong convexity condition. 
More specifically, we organize the discussion with the following hierarchy:
\begin{itemize}
    \item \textbf{Primary split (distance):} Case 1/2 by whether $\|x_t-x^*\|\le \tfrac{1}{3}D_\tau$ or $\|x_t-x^*\|\geq \tfrac{1}{3}D_\tau$.
    \item \textbf{Secondary split (tail scale):} Case 1.1/1.2/2.1/2.2 by whether $D_s\ge D_\tau$ or $D_s< D_\tau$.
    \item \textbf{Tertiary split (time, only when needed):} further split by regimes of $t$ to simplify the sufficient conditions.
\end{itemize}

\paragraph{Case 1: $\|x_t - x^*\| \leq \tfrac{1}{3}D_\tau$.}
In this regime we use the simplified bound \cref{eq:M_out_bound_case1_simpl}.
\paragraph{Case 1.1: $D_s \ge D_\tau$.}
The following inequality is sufficient for \cref{eq:Pout_Mout_constraint} to hold:
\[
    12 D_s^2  \min(1, 4\Pout{}\frac{ \exp\left(-\frac{2D_\tau^2}{9t} \right)}{ 
\exp\left(-\frac{1}{2} C_E^2 \frac{\beta}{\lambda} D_\tau^2\right) } (\frac{Lt + \lambda}{\alpha t})^{d/2}) \leq \frac{(1-C_\alpha)t^2}{t+\frac{\lambda}{\alpha}} 
\]
where we have used \eqref{eq:Pbtout_upper_bound_small_1} and the trivial upper bound of $1$ on any probability. 
We further simplify the above equation to a sufficient condition, by plugging in  $D_s =\tau \sqrt{\log\frac{1}{\Pout{}}} $
the fact that  $x \log(1/x) < \sqrt{x}, \forall 0<x<1$.
\begin{align}
    48 \tau^2 \sqrt{\Pout{}}   \frac{ \exp\left(-\frac{2D_\tau^2}{9t} \right)}{ 
\exp\left(-\frac{1}{2} C_E^2 \frac{\beta}{\lambda} D_\tau^2 \right) } (\frac{Lt + \lambda}{\alpha t})^{d/2} \leq \frac{(1-C_\alpha)t^2}{t+\frac{\lambda}{\alpha}}\label{eq:finalboundcase1} 
\end{align}

Rearranging~\eqref{eq:finalboundcase1} yields the equivalent sufficient condition
\begin{equation}\label{eq:finalbound-subcase11}
    48 \tau^2 \sqrt{\Pout{}} \leq 
    \frac{(1-C_\alpha)t^2}{t+\frac{\lambda}{\alpha}} 
    \frac{\exp\!\left(-\frac{1}{2} C_E^2 \frac{\beta}{\lambda} D_\tau^2 \right)}{ \exp\!\left(-\frac{2D_\tau^2}{9t} \right)} 
    \left(\frac{\alpha t}{Lt + \lambda}\right)^{\!d/2}.
\end{equation}
A sufficient condition for the above is
\begin{equation}
    48 \tau^2 \sqrt{\Pout{}} \leq (1-C_\alpha) t
    (\frac{t}{t+\frac{\lambda}{\alpha}})^{\tfrac{d}{2}+1} \kappa_0^{-\tfrac{d}{2}}
    \frac{\exp\!\left(-\frac{1}{2} C_E^2 \frac{\beta}{\lambda} D_\tau^2 \right)}{ \exp\!\left(-\frac{2D_\tau^2}{9t} \right)},
\end{equation}
where $\kappa_0 = \frac{\beta}{\alpha}$ is the condition number. Since $\exp\!\left(-\frac{1}{2} C_E^2 \frac{\beta}{\lambda} D_\tau^2 \right)\le 1$, it suffices to lower bound the right-hand side. We focus on bounding the following auxiliary function (local to Case~1.1):

For $t>0$, define
\[
\mathcal{F}_{1.1}(t)
\;:=\; t\left(\frac{t}{t+\frac{\lambda}{\alpha}}\right)^{\tfrac{d}{2}+1}
\exp\!\left(\frac{2D_\tau^2}{9t}\right).
\]

\emph{Subcase 1.1.a: $0 < t < 2\kappa_0\sqrt{t}$.}
In this regime, we have the uniform lower bound
\[
\mathcal{F}_{1.1}(t) \;\ge\; c_{1.1,\mathrm{lb}},
\]
where
\[
C_{1.1} := \frac{2D_\tau^2}{9}, 
\qquad
a_{1.1} := \frac{d}{4} + \frac{3}{2} = \frac{d+6}{4},
\]
and
\[
c_{1.1,\mathrm{lb}}
:= (3\kappa_0)^{-\left(\frac{d}{2}+1\right)}
    \left(\frac{e\,C_{1.1}}{a_{1.1}}\right)^{a_{1.1}}
= (3\kappa_0)^{-\left(\frac{d}{2}+1\right)}
    \left(\frac{8e D_\tau^2}{9(d+6)}\right)^{\tfrac{d+6}{4}}.
\]

With the assumption $\lambda \leq 2 \alpha$, and $t < 2\kappa_0\sqrt{t}$, we can derive that 
\[
\kappa_0\sqrt{t} \;\ge\; \kappa_0\frac{\lambda}{\beta} \;\ge\; \frac{\lambda}{\alpha},
\]
where we used $\kappa_0 = \tfrac{\beta}{\alpha}$ in the last inequality.
Together with $t < 2\kappa_0\sqrt{t}$ this yields
\[
t + \frac{\lambda}{\alpha}
\;\le\;
2\kappa_0\sqrt{t} + \kappa_0\sqrt{t}
= 3\kappa_0\sqrt{t}.
\]
Hence
\[
\frac{t}{t+\frac{\lambda}{\alpha}}
\;\ge\;
\frac{t}{3\kappa_0\sqrt{t}}
= \frac{\sqrt{t}}{3\kappa_0}.
\]
Plugging this into $\mathcal{F}_{1.1}(t)$ gives
\[
\mathcal{F}_{1.1}(t)
\;\ge\;
t\left(\frac{\sqrt{t}}{3\kappa_0}\right)^{\tfrac{d}{2}+1}
\exp\!\left(\frac{2D_\tau^2}{9t}\right)
= (3\kappa_0)^{-\left(\frac{d}{2}+1\right)}
    t^{a_{1.1}}\exp\!\left(\frac{C_{1.1}}{t}\right),
\]
with $a_{1.1}$ and $C_{1.1}$ as defined above.

Define
\[
h_{1.1}(t) := t^{a_{1.1}}\exp\!\left(\frac{C_{1.1}}{t}\right), \qquad t>0.
\]
Then
\[
\log h_{1.1}(t) = a_{1.1} \log t + \frac{C_{1.1}}{t},
\quad
\frac{\mathrm{d}}{\mathrm{d}t}\log h_{1.1}(t)
= \frac{a_{1.1}}{t} - \frac{C_{1.1}}{t^2}
= \frac{a_{1.1} t - C_{1.1}}{t^2}.
\]
Thus $\log h_{1.1}$ has a unique critical point at
\[
t_{1.1}^* = \frac{C_{1.1}}{a_{1.1}} > 0,
\]
and
\[
\frac{\mathrm{d}^2}{\mathrm{d}t^2}\log h_{1.1}(t)
= -\frac{a_{1.1}}{t^2} + \frac{2C_{1.1}}{t^3},
\quad
\frac{\mathrm{d}^2}{\mathrm{d}t^2}\log h_{1.1}(t_{1.1}^*)
= \frac{C_{1.1}}{(t_{1.1}^*)^3} > 0.
\]
Hence $t_{1.1}^*$ is the global minimizer of $h_{1.1}$ on $(0,\infty)$, and
\[
h_{1.1}(t) \;\ge\; h_{1.1}(t_{1.1}^*)
= \left(\frac{C_{1.1}}{a_{1.1}}\right)^{a_{1.1}}
    \exp\!\left(\frac{C_{1.1}}{t_{1.1}^*}\right)
= \left(\frac{C_{1.1}}{a_{1.1}}\right)^{a_{1.1}} e^{a_{1.1}}
= \left(\frac{e\,C_{1.1}}{a_{1.1}}\right)^{a_{1.1}},
\quad \forall\,t>0.
\]
Combining this with the previous inequality yields
\[
\mathcal{F}_{1.1}(t)
\;\ge\;
(3\kappa_0)^{-\left(\frac{d}{2}+1\right)}
\left(\frac{ 4 e\,C_{1.1}}{d+6}\right)^{\tfrac{d+6}{4}}
= c_{1.1,\mathrm{lb}},
\]
In conclusion, the sufficient condition becomes
\[
48 \tau^2 \sqrt{\Pout{}} \leq (1-C_\alpha) (3\kappa_0)^{-\left(\frac{d}{2}+1\right)}
\left(\frac{ 4 e\,C_{1.1}}{d+6}\right)^{\tfrac{d+6}{4}} \kappa_0^{-\tfrac{d}{2}} \exp(-\frac{1}{2} C_E^2 \frac{\beta}{\lambda} D_\tau^2).
\]
By taking the maximum of $\Pout{}$ (which occurs at $\Pout{} = 1/e$), we obtain the sufficient condition
\[
33 \tau^2  \leq (1-C_\alpha) (3\kappa_0^2)^{-\left(\frac{d}{2}+1\right)}
\left(\frac{ 4 e\,C_{1.1}}{d+6}\right)^{\tfrac{d+6}{4}} \exp(-\frac{1}{2} C_E^2 \frac{\beta}{\lambda} D_\tau^2).
\]
In other words, for all sufficiently large $d$, the sufficient condition is satisfied if
\begin{equation}
    (3\kappa_0^2)^{-\left(\frac{d}{2}\right)}
    \left(\frac{ 8 e\,D_\tau^2}{9d}\right)^{\tfrac{d}{4}}
    \exp\!\left(-\frac{1}{2} C_E^2 \frac{\beta}{\lambda} D_\tau^2\right)
    > 33\tau^2
    \label{eq:case11_subcase1_sufficient_growth}
\end{equation}

\emph{Subcase 1.1.b: $t > 4 \kappa_0^2$.}
For $t > 4 \kappa_0^2$, we use another bound on $\PbtOut{0}{t}$ from \cref{eq:Pbtout_upper_bound_small_2}, which leads to the following condition:
\[
    48 \tau^2 \log \tfrac{1}{\Pout{}} \Pout{}\leq (1-C_\alpha) \frac{(4 \kappa_0^2)^2}{4 \kappa_0^2 + 2 \kappa_0^2 }.
\]
By taking the maximum of $\Pout{}\log(1/\Pout{})$ (which is $1/e$), we obtain the sufficient condition
\begin{equation}
   18 \tau^2 \leq (1-C_\alpha) \frac{(4 \kappa_0^2)^2}{4 \kappa_0^2 + 2 \kappa_0^2 }
    \label{eq:case11_subcase1_b_sufficient}
\end{equation}
where the factor $1/e$ comes from the fact that $\Pout{} \log(1/\Pout{}) \leq 1/e$ for all $\Pout{} \in (0,1)$.

\paragraph{Case 1.2: $D_s < D_\tau$.}
In this case, the sufficient condition is given by 
\[
    12 D_\tau^2  \min(1, 4\Pout{}\frac{ \exp\left(-\frac{2D_\tau^2}{9t} \right)}{ 
\exp\left(-\frac{1}{2} C_E^2 \frac{\beta}{\lambda} D_\tau^2\right) } (\frac{Lt + \lambda}{\alpha t})^{d/2}) \leq \frac{(1-C_\alpha)t^2}{t+\frac{\lambda}{\alpha}}.
\]
We consider two subcases based on which bound from \cref{eq:Pbtout_upper_bound_small_1,eq:Pbtout_upper_bound_small_2} is used:
\[
    48 D_\tau^2 \Pout{}\frac{ \exp\left(-\frac{2D_\tau^2}{9t} \right)}{ 
\exp\left(-\frac{1}{2} C_E^2 \frac{\beta}{\lambda} D_\tau^2\right) } (\frac{Lt + \lambda}{\alpha t})^{d/2} \leq \frac{(1-C_\alpha)t^2}{t+\frac{\lambda}{\alpha}},
\]
and
\[
    12 D_\tau^2  \min(1, 4\Pout{}) \leq \frac{(1-C_\alpha)t^2}{t+\frac{\lambda}{\alpha}}.
\]

\emph{Subcase 1.2.a: $0 < t < 2\kappa_0\sqrt{t}$.}
Similarly, 
\begin{equation}
    48 D_\tau^2 \Pout{} \leq (1-C_\alpha) t
    (\frac{t}{t+\frac{\lambda}{\alpha}})^{\tfrac{d}{2}+1} \kappa_0^{-\tfrac{d}{2}}
    \frac{\exp\!\left(-\frac{1}{2} C_E^2 \frac{\beta}{\lambda} D_\tau^2 \right)}{ \exp\!\left(-\frac{2D_\tau^2}{9t} \right)},
\end{equation}
And all the calculation for $\mathcal{F}_{1.1}(t)$ is the same as in Subcase~1.1.a, so the sufficient condition is given by
\[
    48 D_\tau^2 \Pout{} \leq (1-C_\alpha) (3\kappa_0^2)^{-\left(\frac{d}{2}+1\right)}
\left(\frac{ 4 e\,C_{1.1}}{d+6}\right)^{\tfrac{d+6}{4}} \exp(-\frac{1}{2} C_E^2 \frac{\beta}{\lambda} D_\tau^2) 
\]
In other words, for all sufficiently large $d$, the sufficient condition is satisfied if
\begin{equation}
    (1-C_\alpha)(3\kappa_0^2)^{-\left(\frac{d}{2}\right)}
    \left(\frac{ 8 e\,D_\tau^2}{9d}\right)^{\tfrac{d}{4}}
    \exp\!\left(-\frac{1}{2} C_E^2 \frac{\beta}{\lambda} D_\tau^2\right)
    > 48 D_\tau^2 \Pout{}
    \label{eq:case12_growth_condition}
\end{equation}
by taking the maximum of $\Pout{}$.

\emph{Subcase 1.2.b: $t > 4\kappa_0^2$.}
In this regime, we have the sufficient condition
\begin{equation}
    48 D_\tau^2 \Pout{} \leq (1-C_\alpha) \frac{(4\kappa_0^2)^2}{4\kappa_0^2 + 2\kappa_0^2 } \label{eq:pout_bound_case2}
\end{equation}

\paragraph{Case 2: $\|x_t - x^*\| > \tfrac{1}{3}D_\tau$.}
In this regime, we use the simplified bound \cref{eq:Pbtout_upper_bound_large_1}.
From the expansion rate condition, we have $ C_E \sqrt{\frac{t + \tfrac{\lambda}{\beta}}{\tfrac{\lambda}{\beta}}} \geq \tfrac{1}{3}$. Therefore,
\[
    t  + \tfrac{\lambda}{\beta} \geq \tfrac{1}{9 C_E^2}\tfrac{\lambda}{\beta}.
\]
With our assumption $\lambda \leq 2 \alpha$ and $C_E \sim \Theta(\sqrt{\frac{\log (d)}{d}} )$ as in the theorem statement, we can show that $\frac{t^2}{t+\tfrac{\lambda}{\beta}} \geq \tfrac{1}{10 C_E^2}\tfrac{\lambda}{\beta}$ for all sufficiently large $d$.
\paragraph{Case 2.1: $D_s \ge D_\tau$.}
In this case we have $\max(D_\tau^2,D_s^2)=D_s^2$, so it suffices that
\begin{equation}
    9 D_s^2  \min(1, 2\Pout{}\exp(8\tfrac{\lambda}{\beta}C_E^2D_\tau^2)) \leq (1-C_\alpha)\tfrac{1}{10 C_E^2}\tfrac{\lambda}{\beta} \label{eq:PbOut_bound_3}
\end{equation}
Plugging in $D_s =\tau \sqrt{\log\frac{1}{\Pout{}}}$, \eqref{eq:PbOut_bound_3} is equivalent to
\begin{equation}\label{eq:PbOut_bound_3_plugged}
    9 \tau^2 \log\!\frac{1}{\Pout{}}\;\min\!\left(1,\,2\Pout{}\exp(8\tfrac{\lambda}{\beta}C_E^2D_\tau^2)\right)
    \;\le\; (1-C_\alpha)\tfrac{1}{10 C_E^2}\tfrac{\lambda}{\beta}
\end{equation}
Here the $t$-dependence has already been absorbed into the lower bound on $\frac{t^2}{t+\lambda/\beta}$, so we do not need an additional $t$-subcase split. We now split into two regimes based on the value of $\Pout{}$ (via the $\min(\cdot,\cdot)$).

    \emph{Regime A: $\Pout{}\ge \frac{1}{2}\exp(-8\tfrac{\lambda}{\beta}C_E^2D_\tau^2)$.}
    In this regime, the minimum equals $1$ and a sufficient condition is
    \begin{equation}\label{eq:PbOut_bound_3_caseA}
     9\tau^2   \log\!\frac{1}{\Pout{}}
        \;\le\; (1-C_\alpha)\tfrac{1}{10 C_E^2}\tfrac{\lambda}{\beta}
    \end{equation}
    Moreover, the case assumption itself implies
    \[
        \log\!\frac{1}{\Pout{}}
        \;\le\; 8\tfrac{\lambda}{\beta}C_E^2D_\tau^2
    \]
    Therefore, a purely-parameter sufficient condition ensuring~\eqref{eq:PbOut_bound_3_caseA} for all $\Pout{}$ in this regime is
    \begin{equation}\label{eq:PbOut_bound_3_caseA_constants}
       720 \tau^2 C_E^4 D_\tau^2
       \leq  (1-C_\alpha)
    \end{equation}
    \emph{Regime B: $\Pout{}\le \frac{1}{2}\exp(-8\tfrac{\lambda}{\beta}C_E^2D_\tau^2)$.}
    In this regime, the minimum equals $2\Pout{}\exp(8\tfrac{\lambda}{\beta}C_E^2D_\tau^2)$ and \cref{eq:PbOut_bound_3_plugged} reduces to
    \begin{equation}
        18\tau^2 \Pout{}\log\!\frac{1}{\Pout{}}
        \;\le\; \frac{(1-C_\alpha)\tfrac{1}{10 C_E^2}\tfrac{\lambda}{\beta}}{\exp(8\tfrac{\lambda}{\beta}C_E^2D_\tau^2)}.
    \end{equation}

    This yields the sufficient condition 
    \begin{equation}
        \label{eq:PbOut_bound_3_caseB}
            67 \tau^2 C_E^2  \exp(8\tfrac{\lambda}{\beta}C_E^2D_\tau^2) \leq (1-C_\alpha) \tfrac{\lambda}{\beta}
    \end{equation}

\paragraph{Case 2.2: $D_s < D_\tau$.}
Likewise, here we have the following condition to satisfy
\[
 9 D_\tau^2 \min(1, 2\Pout{}\exp(8\tfrac{\lambda}{\beta}C_E^2D_\tau^2)) \leq (1-C_\alpha)\tfrac{1}{10 C_E^2}\tfrac{\lambda}{\beta}
\]
The following two conditions are sufficient to satisfy the above inequality
\[
 18 D_\tau^2 \Pout{} \exp(8\tfrac{\lambda}{\beta}C_E^2D_\tau^2) \leq (1-C_\alpha)\tfrac{1}{10 C_E^2}\tfrac{\lambda}{\beta}
\]
Because of \cref{eq:pout_bound_case2}, we have
\[
    D_\tau^2 \Pout{}
    \;\le\;
    (1-C_\alpha)\,\frac{(4\kappa_0^2)^2}{48(4\kappa_0^2 + 4\kappa_0^3)}
    \;=\;
    (1-C_\alpha)\,\frac{\kappa_0^2}{12(1+\kappa_0)}.
\]
Plugging this into the left-hand side yields
\[
    18 D_\tau^2 \Pout{} \exp(8\tfrac{\lambda}{\beta}C_E^2D_\tau^2)
    \;\le\;
    (1-C_\alpha)\,\frac{3}{2}\,\frac{\kappa_0^2}{1+\kappa_0}\,
    \exp(8\tfrac{\lambda}{\beta}C_E^2D_\tau^2),
\]
so it suffices that
\begin{equation}
    \frac{3}{2}\,\frac{\kappa_0^2}{1+\kappa_0}\,
    \exp(8\tfrac{\lambda}{\beta}C_E^2D_\tau^2)
    \;\le\;
    \tfrac{1}{10 C_E^2}\tfrac{\lambda}{\beta}
    \label{eq:case22_parameter_sufficiency}
\end{equation}

In summary, \eqref{eq:PbOut_bound_3} is ensured by either the constant-parameter condition~\eqref{eq:PbOut_bound_3_caseA_constants} (Regime~A),
or the small-$\Pout{}$ condition~\eqref{eq:PbOut_bound_3_caseB} (Regime~B).

\medskip
\noindent\textbf{Step 7: Collecting the boxed conditions and an explicit final regime.}
We now show that the boxed sufficient conditions appearing in the case analysis can be met simultaneously by an explicit parameter choice.
Throughout, treat $(\beta,\alpha,\lambda_{\max},C_\alpha)$ and hence $\kappa_0:=\beta/\alpha$ as fixed constants (independent of $d$), and assume $\lambda\le \lambda_{\max}$.

\paragraph{(7.1) Choosing $D_\tau^2=\Theta(d)$ and deriving $\Pout{}\le 1/d$.}
Pick a constant $b>0$ and set
\[
    D_\tau^2 := b d.
\]
To satisfy the key scaling requirement \cref{eq:key_Dtau_scaling}, it suffices that
\[
    b \;\ge\; \frac{9}{4}\frac{\lambda_{\max}}{\beta} = \frac{9}{2}\frac{\alpha}{\beta},
\]
since then $\frac{2D_\tau^2}{9}\frac{\beta}{\lambda_{\max}} = \frac{2b}{9}\frac{\beta}{\lambda_{\max}}\,d \ge \tfrac{d}{2}$.

Next, the boxed condition \cref{eq:pout_bound_case2} implies
\[
    48 D_\tau^2 \Pout{} \le (1-C_\alpha)\frac{(4\kappa_0^2)^2}{4\kappa_0^2+2\kappa_0^2}
    = (1-C_\alpha)\frac{8}{3}\kappa_0^2,
\]
so
\[
    \Pout{} \le \frac{1-C_\alpha}{18}\frac{\kappa_0^2}{D_\tau^2}
    = \frac{1-C_\alpha}{18}\frac{\kappa_0^2}{b}\cdot \frac{1}{d}.
\]
Therefore, if we additionally choose
\[
    b \;\ge\; \frac{1-C_\alpha}{18}\kappa_0^2,
\]
then the same boxed condition \cref{eq:pout_bound_case2} yields the clean tail-mass bound
\[
    \Pout{} \le \frac{1}{d}.
\]

\paragraph{(7.2) An explicit $\tau$ upper bound from $\kappa_0$.}
From the boxed condition \cref{eq:case11_subcase1_b_sufficient} (and simplifying its right-hand side),
\[
    18\tau^2 \le (1-C_\alpha)\frac{(4\kappa_0^2)^2}{4\kappa_0^2+2\kappa_0^2}
    = (1-C_\alpha)\frac{8}{3}\kappa_0^2,
\]
it suffices to impose
\[
    \tau \;\le\; \tau_{\max}
    \;:=\; \frac{2}{3\sqrt{3}}\sqrt{1-C_\alpha}\,\kappa_0.
\]

\paragraph{(7.3) Choosing $C_E$ from the boxed $C_E$-constraints.}
We now \emph{derive} a sufficient scaling for $C_E$ directly from the boxed inequalities
\cref{eq:PbOut_bound_3_caseB,eq:case22_parameter_sufficiency}.
Let $u:=C_E^2$ and $C_a:=8\frac{\lambda}{\beta}D_\tau^2$. Then:
\begin{itemize}
    \item From \cref{eq:PbOut_bound_3_caseB} we have
    \[
        67\tau^2\,u\,e^{C_a u} \le (1-C_\alpha)\frac{\lambda}{\beta}
        \quad\Longrightarrow\quad
        u\,e^{C_a u} \le \frac{1-C_\alpha}{67\tau^2}\frac{\lambda}{\beta}.
    \]
    \item From \cref{eq:case22_parameter_sufficiency} we have
    \[
        \frac{3}{2}\frac{\kappa_0^2}{1+\kappa_0}\,e^{C_a u} \le \frac{1}{10u}\frac{\lambda}{\beta}
        \quad\Longrightarrow\quad
        u\,e^{C_a u} \le \frac{1+\kappa_0}{15\kappa_0^2}\frac{\lambda}{\beta}.
    \]
\end{itemize}
Therefore it suffices to choose $u$ such that
\[
    u\,e^{C_a u}
    \;\le\;
    \frac{\lambda}{\beta}\,
    \min\!\left\{\frac{1-C_\alpha}{67\tau^2},\;\frac{1+\kappa_0}{15\kappa_0^2}\right\}.
\]
Under $D_\tau^2=b d$ (so $C_a=\Theta(d)$) and $\lambda\le\lambda_{\max}$, one concrete sufficient choice is
\[
    C_E^2
    \leq
    \frac{\beta}{16\lambda_{\max}\,D_\tau^2}\,\log d
    \;=\;
    \frac{\beta}{16\lambda_{\max}b}\,\frac{\log d}{d},
\]
for which $C_a u \le \frac{1}{2}\log d$ and hence
\[
    u\,e^{C_a u}
    \;\le\;
    \frac{\beta}{16\lambda_{\max}b}\,\frac{\log d}{d}\,d^{1/2}
    \;=\;
    \frac{\beta}{16\lambda_{\max}b}\,\frac{\log d}{\sqrt{d}}
    \xrightarrow[d\to\infty]{} 0.
\]
Thus both boxed $C_E$-constraints \cref{eq:PbOut_bound_3_caseB,eq:case22_parameter_sufficiency} hold for all sufficiently large $d$.

\paragraph{(7.4) Compatibility with the remaining boxed conditions.}
At this point, the only remaining requirements are constraints on $(\Pout{},D_\tau,C_E)$ that ensure the boxed inequalities used in Case~1.2 and Case~2.2.\footnote{With $\Pout{}\le 1/d$ and $D_\tau^2=b d$, we have $D_s=\tau\sqrt{\log(1/\Pout{})}=\tau\sqrt{\log d}$, so $D_s<D_\tau$ for all sufficiently large $d$. Hence only the $D_s<D_\tau$ branches (Case~1.2 and Case~2.2) are relevant asymptotically.}
We list them explicitly:
\begin{itemize}
    \item \emph{(Tail-mass scaling.)} The boxed condition \cref{eq:pout_bound_case2} implies
    \[
        \Pout{} \le \frac{1-C_\alpha}{18}\frac{\kappa_0^2}{D_\tau^2}.
    \]
    Under $D_\tau^2=b d$, choosing $b\ge \frac{1-C_\alpha}{18}\kappa_0^2$ yields $\Pout{}\le 1/d$.

    \item \emph{(Key $D_\tau$ scaling.)} The boxed condition \cref{eq:key_Dtau_scaling} is enforced by taking $D_\tau^2=b d$ with
    \[
        b \ge \frac{9}{4}\frac{\lambda_{\max}}{\beta}.
    \]

    \item \emph{(Expansion-rate scale $C_E$.)} Choose $C_E$ as in Paragraph~(iii), i.e.
    \[
        C_E^2 \leq \frac{\beta}{16\lambda_{\max}b}\,\frac{\log d}{d},
    \]
    which in particular implies $C_E^2=\Theta((\log d)/d)$.
\end{itemize}

\noindent With these choices, the remaining boxed conditions are satisfied by choosing $b$ above an explicit (parameter-only) lower bound:
\begin{itemize}
    \item \emph{(Case~1.2.)} The boxed growth condition \cref{eq:case12_growth_condition} holds provided the exponential-in-$d$ term has positive rate, i.e.
    \[
        \log\!\left(\frac{8eb}{9}\right) > 2\log(3\kappa_0^2)
        \quad\Longleftrightarrow\quad
        b > \frac{81}{8e}\kappa_0^4.
    \]
\end{itemize}

\paragraph{Conclusion.}
Choose constants $b$ and $\tau$, and pick $C_E=C_E(d)$, such that
\[
    b \ge \max\!\left\{\frac{9}{4}\frac{\lambda_{\max}}{\beta},\;\frac{1-C_\alpha}{18}\kappa_0^2,\;\frac{81}{8e}\kappa_0^4\right\},
    \qquad
    \tau \le \tau_{\max}.
\]
Let $D_\tau^2=b d$ and assume $\Pout{}\le 1/d$ (which is implied by \cref{eq:pout_bound_case2} under the chosen $b$). Finally, choose $C_E$ to satisfy the explicit upper bound from Paragraph~(iii),
\[
    C_E^2 \le \frac{\beta}{16\lambda_{\max}\,D_\tau^2}\,\log d
    = \frac{\beta}{16\lambda_{\max}b}\,\frac{\log d}{d}.
\]
Then all boxed sufficient conditions used in the Case~1.2 and Case~2.2 analysis hold for all sufficiently large $d$. Therefore the sufficient condition \cref{eq:Pout_Mout_constraint} holds, completing the proof of \cref{theorem:SCBound}.

\end{proof}

\subsection{Proof for \texorpdfstring{\cref{theorem:bias_bound}}{Theorem bias bound}}
Here we provide the proof for \cref{theorem:bias_bound}, which is the main theorem in this paper and is built on top of the proof of \cref{theorem:SCBound}.
\OptBound*

\begin{proof}
    \noindent\textbf{Step 1: Setup and Reduction via Strong Convexity.}
    By the strong convexity of $g(x;t)$ established in \cref{theorem:SCBound}, the function $g(x;t)$ is $\frac{C_\alpha\lambda}{t + \frac{\lambda}{\alpha}}$-strongly convex within the region $\mathcal{R}_{SC}(t)$. Since $x_t^*$ is the minimizer of $g(x;t)$ and $x^*$ lies within the strongly convex region, the distance to the optimizer is bounded by the gradient norm at the original optimizer $x^*$:
    \[
        \|x_t^* - x^*\| \leq \frac{1}{\frac{C_\alpha\lambda}{t + \frac{\lambda}{\alpha}}} \|\nabla g(x^*; t)\| = \frac{t + \frac{\lambda}{\alpha}}{C_\alpha \lambda} \|\nabla g(x^*; t)\|.
    \]

    Using Tweedie's formula, the gradient at $x^*$ is given by
    \[
        \nabla g(x^*; t) = \frac{\lambda}{t} (x^* - \mathbb{E}[y|x^*]).
    \]
    Without loss of generality, we assume $x^*=0$ in this proof. Consequently, the problem reduces to bounding the conditional expectation $\|\mathbb{E}[y|x^*]\|$.  Therefore, we have the reduction:
    \[
        \|x_t^* - x^*\| \leq \frac{t + \frac{\lambda}{\alpha}}{C_\alpha t} \|\mathbb{E}[y|x^*]\|.
    \]

    \noindent\textbf{Step 2: Decomposition.}
    We decompose the expectation into contributions from inside and outside the ball $B_\tau$:
    \begin{align}
        \|\mathbb{E}[y|x^*]\| &\leq  \underbrace{\|\mathbb{E}[y | y \in B_\tau, x^*]\| \cdot P(y \in B_\tau | x^*)}_{\text{Term I: Local Conditional Mean}} + \underbrace{\|\mathbb{E}[y | y \notin B_\tau, x^*]\| \cdot P(y \notin B_\tau | x^*)}_{\text{Term II: Tail Conditional Mean}}.
        \label{eq:decomposition}
    \end{align}
    And for the bias bound, we have 
    \begin{align}
        \|x_t^* - x^*\| &\leq  \left( \underbrace{\frac{t + \frac{\lambda}{\alpha}}{C_\alpha t} \|\mathbb{E}[y | y \in B_\tau, x^*]\| \cdot P(y \in B_\tau | x^*)}_{\text{Term I: Local Asymmetry}} + \underbrace{`\frac{t + \frac{\lambda}{\alpha}}{C_\alpha t} \|\mathbb{E}[y | y \notin B_\tau, x^*]\| \cdot P(y \notin B_\tau | x^*)'}_{\text{Term II: Tail Contribution}}\right)
        \label{eq:bias_decomposition}
    \end{align}

\noindent\textbf{Step 3: Bounding Term I (Local Asymmetry).}
We now bound the contribution to the bias coming from the \emph{local} region $B_\tau=\{y:\|y-x^*\|\le D_\tau\}$, where $f$ is $\alpha$-strongly convex and $\beta$-smooth.
Intuitively, if $f$ were perfectly symmetric around $x^*$, then the posterior mean under the Gaussian smoothing would remain centered and this term would vanish.
Thus, Term~I captures the worst-case \emph{directional asymmetry} of the local landscape permitted by the curvature bounds $(\alpha,\beta)$.
Formally, it suffices to control the directional conditional mean $\mathbb{E}[\langle y,\hat e\rangle\mid x^*,\, y\in B_\tau]$ uniformly over unit vectors $\hat e$; we will upper bound this quantity by replacing $p_0(y)\propto e^{-f(y)/\lambda}$ on $B_\tau$ with extremal envelopes consistent with the local curvature constraints.

\paragraph{(3.1) Worst-case envelope inside $B_\tau$.}
On $B_\tau=\{\|y\|\le D_\tau\}$, by $\alpha$-strong convexity and $\beta$-smoothness around
$x^*$ (with $\nabla f(x^*)=0$), we have
\[
\frac{\alpha}{2}\|y\|^2 \le f(y)-f(0) \le \frac{\beta}{2}\|y\|^2 .
\]
Since $p_0(y)\propto \exp(-f(y)/\lambda)$, this implies the pointwise envelope
\[
\exp\!\Big(-\frac{\beta}{2\lambda}\|y\|^2\Big) \;\le\; \frac{p_0(y)}{p_0(0)} \;\le\;
\exp\!\Big(-\frac{\alpha}{2\lambda}\|y\|^2\Big),\qquad y\in B_\tau.
\]
Define $\tilde p_0^{\max}(y):=\exp(-\frac{\beta}{2\lambda}\|y\|^2)$ and
$\tilde p_0^{\min}(y):=\exp(-\frac{\alpha}{2\lambda}\|y\|^2)$. We omit $p_0(0)$ in the following calculations as it is a constant appearing in both numerator and denominator.

\paragraph{(3.2) Directional conditional mean bound.}
We calculate the conditional expectation separately for two hemispheres: one where $\langle y, \hat{e} \rangle > 0$ and the other where $\langle y, \hat{e} \rangle  < 0$.
    \begin{align}
        \mathbb{E}_{0|t}[\langle y, \hat{e} \rangle|x^*, y \in B_{\tau}] & = \frac{\int_{B_{\tau}} \langle y, \hat{e} \rangle p_0(y) p_{t\mid 0}(x^*\mid y) dy}{\int_{B_{\tau}} p_0(y) p_{t\mid 0}(x^*\mid y) dy} \nonumber                                                                                    \\
  & = \frac{\int_{B_{\tau}, \langle y, \hat{e} \rangle  > 0} \langle y, \hat{e} \rangle p_0(y) p_{t\mid 0}(x^*\mid y) dy + \int_{B_{\tau}, \langle y, \hat{e} \rangle  < 0} \langle y, \hat{e} \rangle p_0(y) p_{t\mid 0}(x^*\mid y) dy}{\int_{B_{\tau}, \langle y, \hat{e} \rangle  > 0} p_0(y) p_{t\mid 0}(x^*\mid y) dy + \int_{B_{\tau}, \langle y, \hat{e} \rangle  < 0} p_0(y) p_{t\mid 0}(x^*\mid y) dy} \nonumber \\
  &\leq \frac{\int_{B_{\tau}, \langle y, \hat{e} \rangle  > 0} \langle y, \hat{e} \rangle \pzeromax(y) p_{t\mid 0}(x^*\mid y) dy + \int_{B_{\tau}, \langle y, \hat{e} \rangle  < 0} \langle y, \hat{e} \rangle \pzeromin(y) p_{t\mid 0}(x^*\mid y) dy}{\int_{B_{\tau}, \langle y, \hat{e} \rangle  > 0} \pzeromax(y) p_{t\mid 0}(x^*\mid y) dy + \int_{B_{\tau}, \langle y, \hat{e} \rangle  < 0} \pzeromin(y) p_{t\mid 0}(x^*\mid y) dy}
  \label{eq:conditional_expectation_inside}
    \end{align}
    The inequality in \eqref{eq:conditional_expectation_inside} follows from the pointwise envelope
$\tilde p_0^{\max}(y)\ge p_0(y)/p_0(0)\ge \tilde p_0^{\min}(y)$ on $B_\tau$:
on the half-space $\{y_1>0\}$ the integrand $\langle y,\hat e\rangle=y_1$ is nonnegative, so replacing $p_0$ by the upper envelope increases the numerator;
on $\{y_1<0\}$ the integrand is nonpositive, so replacing $p_0$ by the lower envelope increases the numerator (makes it less negative).
In the denominator the integrand is nonnegative everywhere, hence using $\tilde p_0^{\max}$ on $\{y_1>0\}$ and $\tilde p_0^{\min}$ on $\{y_1<0\}$ yields an upper bound on the ratio.

\paragraph{(3.3) Monotonicity in $D_\tau$ and the limit $D_\tau\to\infty$.}
We next show that the upper bound in \eqref{eq:conditional_expectation_inside} is monotone in the ball radius $D_\tau$.
This allows us to take $D_\tau\to\infty$ and reduce the calculation to standard Gaussian integrals.

\begin{equation}
    \text{RHS of \eqref{eq:conditional_expectation_inside}} =  \frac{\int_{r=0}^{D_\tau} \left[\int_{\|y\| = r,\langle y, \hat{e} \rangle > 0} \langle y, \hat{e} \rangle  \pzeromax(y) p_{t|0}(x^*|y) dy + \int_{\|y\| = r,\langle y, \hat{e} \rangle < 0} \langle y, \hat{e} \rangle  \pzeromin(y) p_{t|0}(x^*|y) dy\right] dr }{\int_{r=0}^{D_\tau} \left[\int_{\|y\| = r,\langle y, \hat{e} \rangle > 0} \pzeromax(y) p_{t|0}(x^*|y) dy + \int_{\|y\| = r,\langle y, \hat{e} \rangle < 0} \pzeromin(y) p_{t|0}(x^*|y) dy\right] dr }\label{eq:cond_exp_withD}
\end{equation}

To this end, express both numerator and denominator in \eqref{eq:conditional_expectation_inside} in polar coordinates.
At each radius $r$, define the corresponding “shell ratio” as in \eqref{eq:inner_mono}.
Then \eqref{eq:cond_exp_withD} is precisely the ratio of the integrals of these shell contributions over $r\in[0,D_\tau]$.
    \begin{equation}
        \frac{\int_{\|y\| = r,\langle y, \hat{e} \rangle > 0} \langle y, \hat{e} \rangle  \pzeromax(y) p_{t|0}(x^*|y) dy + \int_{\|y\| = r,\langle y, \hat{e} \rangle < 0} \langle y, \hat{e} \rangle  \pzeromin(y) p_{t|0}(x^*|y) dy}{\int_{\|y\| = r,\langle y, \hat{e} \rangle > 0} \pzeromax(y) p_{t|0}(x^*|y) dy + \int_{\|y\| = r,\langle y, \hat{e} \rangle < 0} \pzeromin(y) p_{t|0}(x^*|y) dy}.
    \label{eq:inner_mono}
\end{equation}

For a fixed radius $r>0$, the shell ratio in \eqref{eq:inner_mono} is increasing in $r$.
Indeed, after rotating so that $\hat e=e_1$, the integrands depend on $y$ only through
$y_1$ and $\|y\|=r$, and the resulting one-dimensional form reduces (up to positive
multiplicative factors) to
\[
h(y):=\frac{y e^{-y^2/a}-y e^{-y^2/b}}{e^{-y^2/a}+e^{-y^2/b}}
= y\,\tanh\!\Big(\frac{y^2}{2}\Big(\frac{1}{b}-\frac{1}{a}\Big)\Big),
\qquad a>b>0,
\]

This ratio is increasing in $y$ (one can verify this by reducing to the 1D form
$y\mapsto \frac{y e^{-y^2/a}-y e^{-y^2/b}}{e^{-y^2/a}+e^{-y^2/b}}
= y\tanh\!\big(y^2(\frac{1}{b}-\frac{1}{a})/2\big)$ with $a>b>0$). Consequently, writing \eqref{eq:cond_exp_withD} as
$\frac{\int_0^{D_\tau} u(r)\,dr}{\int_0^{D_\tau} v(r)\,dr}$ with $u(r)/v(r)$ increasing and
$v(r)>0$, we obtain that the ratio is increasing in $D_\tau$ (a standard “ratio of integrals”
monotonicity argument).

Therefore, the RHS of \eqref{eq:conditional_expectation_inside} is monotone increasing in
$D_\tau$, and we may upper bound it by taking the limit $D_\tau\to\infty$.
In the next step we evaluate the resulting Gaussian half-space integrals by bounding the
numerator and denominator separately.

\paragraph{(3.4) Closed-form bound via Gaussian integrals.}
Taking $D_\tau\to\infty$, the terms in \eqref{eq:conditional_expectation_inside} become Gaussian half-space integrals with quadratic exponents.
We evaluate the numerator and denominator separately and then combine them to obtain a closed-form upper bound.
\begin{align*}
    \int_{B_{\tau}, \langle y, \hat{e} \rangle  > 0} \langle y, \hat{e} \rangle \pzeromax(y) p_{t\mid 0}(x^*\mid y) dy & = \int_{B_{\tau}, \langle y, \hat{e} \rangle  > 0} \langle y, \hat{e} \rangle \pzeromax(y) p_{t\mid 0}(x^*\mid y) dy                                                  \\        
& = \int_{B_{\tau}, \langle y, \hat{e} \rangle  > 0} \langle y, \hat{e} \rangle \exp\left(-\left(\frac{\beta}{2\lambda} + \frac{1}{2t}\right) \|y\|^2\right) dy
\end{align*}
Using $B_\tau^+$ to represent $ B_{\tau}\cap\{y: \langle y, \hat{e} \rangle  > 0\} $, we have:
\begin{align*}
    \int_{B_\tau^+} \langle y, \hat{e} \rangle  \exp{((-\frac{\beta}{2\lambda} - \frac{1}{2t}) \|y \|^2)} dy
    =  \frac{1}{-\frac{\beta}{\lambda} - \frac{1}{t}}\int_{B_\tau^+} \nabla\left(\exp{((-\frac{\beta}{2\lambda} - \frac{1}{2t}) \|y\|^2)} \right)^\top\hat{e}   dy.
\end{align*}
Applying Gauss's theorem yields:
\begin{align*}
    \frac{1}{-\frac{\beta}{\lambda} - \frac{1}{t}}\int_{B_\tau^+} \nabla\left(\exp{((-\frac{\beta}{2\lambda} - \frac{1}{2t}) \|y\|^2)}\right)^\top \hat{e}    dy
    = \frac{1}{-\frac{\beta}{\lambda} - \frac{1}{t}} \iint_{\partial B_\tau^+} \exp{((-\frac{\beta}{2\lambda} - \frac{1}{2t}) \|y\|^2)} (\hat{e}^\top  d\mathbf{S} )
\end{align*}
The surface integral splits into two parts: the hemispherical surface $\partial B_{\tau}^+ \cap \{ y:\Vert y\Vert = D_\tau\}$ and the flat surface $\{y:\langle y, \hat{e} \rangle =0\}\cap B_r$. The integral over the hemisphere vanishes as $D_\tau \to \infty$:
\begin{align*}
        &\lim_{D_\tau \rightarrow \infty}
        \iint_{\partial B_{\tau}^+\cap \{ y:\Vert y\Vert = D_\tau\}}
        \exp{\left(\left(-\frac{\beta}{2\lambda} - \frac{1}{2t}\right) \|y\|^2\right)}
        \hat{e}^\top d\mathbf{S} \\
        &\quad =
        \lim_{D_\tau \rightarrow \infty}
        \exp{\left(\left(-\frac{\beta}{2\lambda} - \frac{1}{2t}\right) \|D_\tau\|^2\right)}
        \frac{\pi^{(d-1)/2}}{\Gamma((d+1)/2)}D_\tau^{d-1} \hat{e}
        =0
\end{align*}
For the integral over the flat surface $\{y:\langle y, \hat{e} \rangle =0\}\cap B_r$:
\begin{align*}
    \iint_{\langle y, \hat{e} \rangle  = 0, y \in B_{\tau}} \exp{((-\frac{\beta}{2\lambda} - \frac{1}{2t}) \|y\|^2)} \hat{e}^\top  d\mathbf{S}  = \frac{(2\pi)^{(d-1)/2}}{(\frac{\beta}{\lambda} + \frac{1}{t})^{(d-1)/2}\Gamma((d-1)/2)}\gamma(\frac{d-1}{2}, D_\tau^2(\frac{\beta}{2\lambda} + \frac{1}{2t}))
\end{align*}
Take the limit, we have that
\begin{align*}
    \lim_{D_\tau \rightarrow \infty} \iint_{\langle y, \hat{e} \rangle  = 0, y \in B_{\tau}} \exp{((-\frac{\beta}{2\lambda} - \frac{1}{2t}) \|y\|^2)} \hat{e}  d\mathbf{S}  = \frac{(2\pi)^{(d-1)/2}}{(\frac{\beta}{\lambda} + \frac{1}{t})^{(d-1)/2}}
\end{align*}
As for the denominator of \eqref{eq:conditional_expectation_inside}, we have that
The gaussian integral can be calculated that
\begin{align*}
    \lim_{D_\tau \rightarrow \infty} \int_{B_{\tau}} \exp{((-\frac{\beta}{2\lambda} - \frac{1}{2t}) \|y\|^2)} dy  = \frac{(2\pi)^{(d)/2}}{(\frac{\beta}{\lambda}+\frac{1}{t})^{(d)/2}}
\end{align*}

It is straightforward to show that for the other side of the ball, we just have to replace $\frac{\beta}{\lambda}$ with $\frac{\alpha}{\lambda}$. Hereby, we have the final bound for \cref{eq:conditional_expectation_inside}.

\paragraph{(3.5) Final bound for the conditional expectation inside the ball.} As we have shown that the conditional expectation is monotonically increasing in $D_\tau$, we can take the limit of $D_\tau$ to infinity for equation \cref{eq:conditional_expectation_inside} to get upper bound. In what follows we separately upper bound the numerator and denominator of \eqref{eq:conditional_expectation_inside}. We first calculate the terms in the numerator:
\begin{align*}
    \mathbb{E}_{0|t}[\langle y, \hat{e} \rangle |x^*, y \in B_{\tau}] \leq \frac{\frac{1}{-\frac{\beta}{\lambda} - \frac{1}{t}} \frac{(2\pi)^{(d-1)/2}}{(\frac{\beta}{\lambda} + \frac{1}{t})^{(d-1)/2}} -\frac{1}{-\frac{\alpha}{\lambda} - \frac{1}{t}} \frac{(2\pi)^{(d-1)/2}}{(\frac{\alpha}{\lambda} + \frac{1}{t})^{(d-1)/2}}}{\frac{(2\pi)^{(d)/2}}{(\frac{\beta}{\lambda} + \frac{1}{t})^{(d)/2}} + \frac{(2\pi)^{(d)/2}}{(\frac{\alpha}{\lambda} + \frac{1}{t})^{(d)/2}}}
\end{align*}
Collecting the limiting numerator/denominator terms from the two half-spaces (with parameters $\beta$ and $\alpha$, respectively) and simplifying yields the following expression.
For notational convenience, define
\[
T(\alpha,\lambda,t):=\frac{\frac{1}{t}+\frac{\beta}{\lambda}}{\frac{1}{t}+\frac{\alpha}{\lambda}}.
\]
And we can further simplify it to
\[
    \frac{1}{\sqrt{2\pi}}
    \;\;\frac{-
    \;\frac{1}{\bigl(\tfrac1t + \frac{\beta}{\lambda}\bigr)^{\frac{d+1}{2}}}
    \;+\;
    \frac{1}{\bigl(\tfrac1t + \frac{\alpha}{\lambda}\bigr)^{\frac{d+1}{2}}}
    }{
    \frac{1}{\bigl(\tfrac1t + \frac{\alpha}{\lambda}\bigr)^{\frac{d}{2}}}
    \;+\;
    \frac{1}{\bigl(\tfrac1t + \frac{\beta}{\lambda}\bigr)^{\frac{d}{2}}}
    } = \frac{1}{\sqrt{2\pi}}\frac{T(\alpha,\lambda,t)^{\tfrac{d+1}{2}}-1}{T(\alpha,\lambda,t)^{\tfrac{d}{2}}+1}\frac{1}{\sqrt{\tfrac{1}{t}+\tfrac{\beta}{\lambda}}}.
\]
\paragraph{(3.6) A uniform simplification.}
The closed-form expression above can be further simplified into a dimension-free bound in terms of the local condition number $\kappa_0=\beta/\alpha$.
To this end, define $\phi(T):=\frac{T^{\frac{d+1}{2}}-1}{T^{\frac{d}{2}}+1}$ and consider
\[
L_2 := \phi(T)\sqrt{\tfrac{1}{t}+\tfrac{\alpha}{\lambda}}.
\]
We next upper bound $\phi(T)$ by $\sqrt{T}-1$.

\begin{equation}
    \frac{1}{C_\alpha}\frac{\lambda}{\alpha} \left(\frac{\alpha}{\lambda}+\frac{1}{t}\right)\left[
        \frac{1}{\sqrt{2\pi}}
        \underbrace{\frac{T(\alpha,\lambda,t)^{\tfrac{d+1}{2}}-1}{T(\alpha,\lambda,t)^{\tfrac{d}{2}}+1}}_{:=\phi(T)}\frac{1}{\sqrt{\tfrac{1}{t}+\tfrac{\beta}{\lambda}}}
        \right]
    \label{eq:cond_num_expec_bound}
\end{equation}
Here we consider $\phi(T) = \frac{T^{\tfrac{d+1}{2}}-1}{T^{\tfrac{d}{2}}+1}$ (matching the definition in \eqref{eq:cond_num_expec_bound}) and $L_2 = \phi(T) \sqrt{\tfrac{1}{t}+\tfrac{\alpha}{\lambda}}$.
We calculate the derivative of $\phi(T)$:
\[
    \frac{d\phi(T)}{dT} = \frac{\frac{d+1}{2} T^{\frac{d-1}{2}}(T^{\tfrac{d}{2}}+1) - \frac{d}{2} T^{\frac{d}{2}-1}(T^{\tfrac{d+1}{2}}-1)}{(T^{\tfrac{d}{2}}+1)^2}.
\]
After simplification, this yields:
\[
    \frac{d}{dT} \log \phi(T) \leq \frac{1}{2T}.
\]
This leads to
\[
    \phi(T) \leq \sqrt{T} - 1
\]
Substituting $\phi(T)\le \sqrt{T}-1$ and using $\sqrt{T}=\sqrt{\frac{\frac{1}{t}+\frac{\beta}{\lambda}}{\frac{1}{t}+\frac{\alpha}{\lambda}}}$ yields
\[
L_2 \le \sqrt{\tfrac{1}{t}+\tfrac{\beta}{\lambda}}-\sqrt{\tfrac{1}{t}+\tfrac{\alpha}{\lambda}}
= \frac{\frac{\beta-\alpha}{\lambda}}{\sqrt{\tfrac{1}{t}+\tfrac{\beta}{\lambda}}+\sqrt{\tfrac{1}{t}+\tfrac{\alpha}{\lambda}}}
\le \frac{\alpha}{\lambda}\frac{\kappa_0-1}{2\sqrt{\tfrac{1}{t}+\tfrac{\alpha}{\lambda}}}.
\]
Therefore, for \cref{eq:cond_num_expec_bound}, we have 
\[
    \frac{1}{C_\alpha}\frac{\lambda}{\alpha} \left(\frac{\alpha}{\lambda}+\frac{1}{t}\right)\left[
        \frac{1}{\sqrt{2\pi}}
        \frac{T(\alpha,\lambda,t)^{\tfrac{d+1}{2}}-1}{T(\alpha,\lambda,t)^{\tfrac{d}{2}}+1}\frac{1}{\sqrt{\tfrac{1}{t}+\tfrac{\beta}{\lambda}}}
        \right] \leq \frac{\kappa_0 -1}{ 2 C_\alpha} \sqrt{\tfrac{1}{2\pi(\tfrac{1}{t}+\tfrac{\alpha}{\lambda})}}.
\]
Note that as $t \to 0$, this term grows in the order of $\sqrt{t}$, and as $t \to \infty$, it is bounded by $\frac{\kappa_0 -1}{ 2 C_\alpha} \sqrt{\tfrac{\lambda}{2\pi \alpha}}$.

    \noindent\textbf{Step 4: Bounding Term II (Tail Contribution).}

From \cref{eq:Pout_Mout_constraint} together with \cref{eq:M_out_bound_case1_simpl,eq:M_out_bound_case2_simpl} in the proof of \cref{thm:single_stage}, 
    \[
        \PbtOut{0}{t}|_{x_t = x^*} 9 \max(D_s^2, D_\tau^2) \leq \frac{(1-C_\alpha) t^2}{t+ \frac{\lambda}{\alpha}},
    \]
    where $D_s$ is as defined in \cref{tab:notation}.
    Therefore, we have that
    \[
        \PbtOut{0}{t}|_{x_t = x^*} (D_\tau + \tau) \leq \frac{(1-C_\alpha) t^2}{t+ \frac{\lambda}{\alpha}} \frac{D_\tau + \tau}{9 D_\tau^2}  \leq \frac{(1-C_\alpha) t^2}{t + 
        \frac{\lambda}{\alpha}} \frac{1}{4 D_\tau}
    \]
    And given that $\PbtOut{0}{t}|_{x_t = x^*} \leq 1$, we have that
    \[
        \PbtOut{0}{t}|_{x_t = x^*} (D_\tau + \tau) \leq (D_\tau + \tau) 
    \]
    As the conditional expectation $\mathbb{E}[\|x-x^*\| |x_t = x^*, x \notin B_{\tau}]$ is bounded by $D_\tau + \frac{\tau^2}{2 D_\tau}$, we have the total conditional 
    expectation bound for the second term
    \begin{align*}
        &\frac{t + \frac{\lambda}{\alpha}}{C_\alpha t}
        \|\mathbb{E}[y | y \notin B_\tau, x^*]\| \cdot P(y \notin B_\tau | x^*) \\
        &\quad \leq \frac{t+\frac{\lambda}{\alpha}}{C_\alpha t}
        \PbtOut{0}{t}|_{x_t = x^*} (D_\tau + \tau) \\
        &\quad \leq
        \min\left\{\frac{(1-C_\alpha) t}{C_\alpha} \frac{1}{4 D_\tau},
        (D_\tau + \tau) \frac{t+\frac{\lambda}{\alpha}}{C_\alpha t}\right\}
    \end{align*}

    \noindent\textbf{Step 5: Combination.}
    Combining Terms I and II from the decomposition \eqref{eq:decomposition}, we have the total conditional expectation:
    \begin{align*}
        \|x_t^* - x^*\|
        &\leq \frac{\kappa_0 -1}{ 2 C_\alpha}
        \sqrt{\tfrac{1}{2\pi(\tfrac{1}{t}+\tfrac{\alpha}{\lambda})}} \\
        &\quad + \min\left\{\frac{(1-C_\alpha) t}{C_\alpha} \frac{1}{4 D_\tau},
        (D_\tau + \tau) \frac{t+\frac{\lambda}{\alpha}}{C_\alpha t}\right\}.
    \end{align*}
    
\end{proof}

\subsection{Bias compared to convex radius}
\begin{remark}
    Though the convex radius seems expanding with order of $O(\sqrt{t})$, and the bias seems expanding with order of $O(t)$, the convex radius is actually much larger than the bias because of the \cref{theorem:convex_region_lower_bound}. One can show that only when $t$ is large enough, the bias is larger than $\frac{1}{4} D_\tau$. Namely at least 
    \[
    t \frac{1-C_\alpha}{C_\alpha} \frac{1}{4 D_\tau} \geq \frac{1}{8} D_\tau
    \]
    which implies
    \[
    t \geq \frac{1}{2} D_\tau^2 \frac{C_\alpha}{1-C_\alpha}
    \]
    Pluging this into \cref{eq:convex_radius}. One can show that the convex radius
    \[
        \mathcal{R}_{SC}(t) \geq C_E \sqrt{\frac{\frac{1}{2} \frac{C_\alpha}{1-C_\alpha}}{\frac{\lambda}{\beta} }} D_\tau^2 \sim \Theta(\sqrt{d \log d})
    \]
    Recall the sufficient condition for $C_E$ as in theorem statement \cref{theorem:SCBound} and \cref{eq:conv_suff}. 
    well on the other hand, the bias is bounded by $(D_\tau + \tau)\frac{2}{ C_\alpha} \sim \Theta(\sqrt{d})$. Therefore, we have that 
    \[
    \mathcal{R}_{SC}(t) \gg \|x_t - x^*\|
    \]
    holds for all $t$ given $d$ is sufficiently large.
\end{remark}
\subsection{Other Auxilliary Results}
\begin{restatable}[Covex Region Lower Bound]{theorem}{ConvexRegionLowerBound}
    \label{theorem:convex_region_lower_bound}
    Let all the conditions in \cref{theorem:SCBound} hold. The convex region is lower bounded by the following inequality for all t for sufficiently large $d$.
    \begin{equation}
        \mathcal{R}_{SC}(t) \geq \frac{1}{4} D_\tau 
    \end{equation}
\end{restatable}
At the beginning stage, when $t$ is small, the landscape of $g(x;t)$ is approximately the same as the landscape of $f(x)$. Therefore, the convex region should be very similar to the convex region of $f(x)$. Here we improve \cref{theorem:SCBound} by the above theorem showing that when $t$ is small, the radius of the convex region does not start from 0 but instead, should be at least of the same order as $D_\tau$.
\begin{proof}
Recall \cref{eq:PbtOut_bound_2}, we can get something similar to \cref{eq:Pbtout_upper_bound_large_1}.
\begin{align}
    \PbtOut{0}{t}  \leq 2 \frac{\Pout{}  \exp\left(-\frac{9 D_\tau^2}{32t}\right)}{ P_0(\| x- x^*\| \leq \tfrac{1}{4}D_\tau) \exp\left(-\frac{ D_\tau^2}{8} \right)} 
 \end{align}
 As in the theorem statement \cref{eq:conv_suff}, while $\alpha,\beta$ are still constants. $P(\| x- x^*\| \leq \tfrac{1}{4}D_\tau) \geq 1/2 \Pin{}$ holds trivially when $d$ is large. Then
\[
\PbtOut{0}{t} \leq 4  \Pout{} \exp\left(-\frac{5 D_\tau^2}{32t}\right)
\]
The sufficient condition for the convex radius now beomes 
\[
4  \Pout{} \exp\left(-\frac{5 D_\tau^2}{32t}\right) \momout \leq \frac{C_\alpha t^2}{t+ \lambda / \alpha}
\]
Moreover, with \cref{eq:M_out_bound_case1_simpl,eq:M_out_bound_case2_simpl}, this leads to
\[
48 \max(D_\tau^2, D_s^2) \Pout{} \leq \frac{(1-C_\alpha)t^2}{t+ \lambda / \alpha} \exp\left(\frac{ D_\tau^2}{32t}\right)
\]
Now we first focus on lower bound the RHS of the above inequality.
\[
\mathcal{F}_{1.2}(t)
\;:=\; t\left(\frac{t}{t+\frac{\lambda}{\alpha}}\right)
\exp\!\left(\frac{5 D_\tau^2}{32\,t}\right),
\qquad t>0.
\]

Write
\[
\log \mathcal{F}_{1.2}(t)
= 2\log t - \log\!\left(t+\frac{\lambda}{\alpha}\right) + \frac{5 D_\tau^2}{32\,t}.
\]
Differentiate and set to zero:
\[
\frac{d}{dt}\log \mathcal{F}_{1.2}(t)
= \frac{2}{t}
- \frac{1}{t+\frac{\lambda}{\alpha}}
- \frac{5 D_\tau^2}{32\,t^2}
=0.
\]
Multiplying by \(t^2\!\left(t+\frac{\lambda}{\alpha}\right)\) gives the quadratic
\[
t^2+\left(\frac{2\lambda}{\alpha}-\frac{5 D_\tau^2}{32}\right)t
-\frac{D_\tau^2}{32}\frac{\lambda}{\alpha}=0,
\]
whose unique positive root is
\[
t^\star
=\frac{1}{2}\left(
\frac{5 D_\tau^2}{32}-\frac{2\lambda}{\alpha}
+\sqrt{\left(\frac{D_\tau^2}{32}\right)^2
+4\left(\frac{\lambda}{\alpha}\right)^2}
\right).
\]
This \(t^\star\) is the (global) minimizer of \(\mathcal{F}_{1.1}(t)\) over \(t>0\).

Let
\[
s:=\sqrt{\left(\frac{5 D_\tau^2}{32}\right)^2+4\left(\frac{\lambda}{\alpha}\right)^2}.
\]
Then \(t^\star+\frac{\lambda}{\alpha}=\frac{1}{2}\left(\frac{D_\tau^2}{32}+s\right)\), and
\[
\mathcal{F}_{1.2}(t^\star)
=
\frac{\left(\frac{5 D_\tau^2}{32}-\frac{2\lambda}{\alpha}+s\right)^2}
{2\left(\frac{5 D_\tau^2}{32}+s\right)}
\exp\!\left(
\frac{\frac{5 D_\tau^2}{16}}
{\frac{5 D_\tau^2}{32}-\frac{2\lambda}{\alpha}+s}
\right).
\]

Given our assumption that $D_\tau^2 \sim \Theta(d)$, and $\lambda/\beta \leq 2\kappa_0$, we have that
\[
\frac{5 D_\tau^2}{32}\gg \frac{\lambda}{\alpha},\quad
t^\star
=
\frac{5 D_\tau^2}{32}-\frac{\lambda}{\alpha}
+O\!\left(\frac{(\lambda/\alpha)^2}{D_\tau^2}\right),
\]
and consequently
\[
\mathcal{F}_{1.1}(t^\star)
=
e\left(\frac{5 D_\tau^2}{32}-\frac{\lambda}{\alpha}\right)
+O\!\left(\frac{(\lambda/\alpha)^2}{D_\tau^2}\right) \geq 0.42 D_\tau^2
\]
Putting pieces together, we get the sufficient condition for the convex radius as follows
\[
115 \max(D_\tau^2, D_s^2) \Pout{} \leq (1-C_\alpha) D_\tau^2
\]
Recall the sufficient condition for $\Pout{}$ and $D_\tau^2$ as in theorem statement \cref{theorem:SCBound}, this holds trivially for sufficiently large $d$.
\end{proof}
\begin{restatable}[Local Condition Number Property]{theorem}{LocCond}
    \label{theorem:local_cond}
    Let \cref{assum:local_convexity} and \cref{assum:subgaussian} hold. The local condition number inside the strongly convex region $\mathcal{R}_{SC}(t)$ is bounded that  
    \begin{equation}
        \kappa(t) \leq \frac{\alpha t + \lambda}{\beta t + \lambda} \frac{1}{C_\alpha} \kappa_0 + \frac{t \lambda / \beta}{t + \lambda / \beta}\frac{1}{9 D_\tau^2}
        \label{eq:local_cond_bound}
    \end{equation}
\end{restatable}

\begin{proof}
    Recall that we have \cref{eq:Pout_Mout_constraint} holds. Therefore, 
    \begin{align}
    \PbtOut{0}{t}  & \leq \frac{1}{9 \max(D_s^2, D_\tau^2)} \frac{(1-C_\alpha)t^2}{t + \lambda/\alpha}
    \end{align}
    With the variance fraction, we get that 
    \[
    \CovbtAll{0}{t} \geq \PbtIn{0}{t} \CovbtIn{0}{t} \geq (1-\PbtOut{0}{t}) \frac{t \lambda / \beta}{t+ \lambda/\beta}
    \]
    Recall that \cref{eq:Pout_Mout_constraint} holds. Therefore, we have that
    \[
    \PbtOut{0}{t} \momout \leq \frac{(1-C_\alpha)t^2}{t+ \lambda / \alpha}
    \]
    and given \cref{eq:second_order_tweedie_prop}
    \[
    \nabla^2 g(x;t) = \frac{\lambda}{t^2}(t - \CovbtAll{0}{t})
    \]
    The condition number can be bounded by 
    \begin{align*}
    \kappa_t  \leq \frac{t - (1 - \PbtOut{0}{t}) \frac{t \lambda / \beta}{t+ \lambda/\beta}}{\frac{C_\alpha t^2}{t+ \lambda / \alpha}} 
             \leq \frac{\alpha t + \lambda}{\beta t + \lambda} \frac{1}{C_\alpha} \kappa_0 + \frac{t \lambda / \beta}{t + \lambda / \beta}\frac{1}{9 D_\tau^2}
    \end{align*}
\end{proof}

\begin{lemma}[Variance Fraction]
    \label{lemma:variance_fraction}
    The total variance of a random variable $X$ (scalar or vector) can be decomposed based on a partition of the sample space. Consider a partition into a set $\mathbb{A}$ and its complement $\mathbb{A}^\complement$.  Let $P(\mathbb{A})$ be the probability that an outcome is in $\mathbb{A}$, and $P(\mathbb{A}^\complement) = 1 - P(\mathbb{A})$. Let $\mu_{\mathbb{A}} = \mathbb{E}[X | X \in \mathbb{A}]$ and $\mu_{\mathbb{A}^\complement} = \mathbb{E}[X | X \in \mathbb{A}^\complement]$. Then the decomposition is:
\begin{align*}
        \operatorname{Var}(X) &= \underbrace{P(\mathbb{A}) \operatorname{Var}(X | X \in \mathbb{A}) + P(\mathbb{A}^\complement) \operatorname{Var}(X | X \in \mathbb{A}^\complement)}_{\text{Expected Conditional Variance (Variance within groups)}} \\
                             &\quad + \underbrace{P(\mathbb{A})P(\mathbb{A}^\complement) \left( \mu_{\mathbb{A}} - \mu_{\mathbb{A}^\complement} \right) \left( \mu_{\mathbb{A}} - \mu_{\mathbb{A}^\complement} \right)^T}_{\text{Variance of Conditional Expectations (Variance between groups)}}.
\end{align*}
    This lemma is often used to break down the overall variability of $X$ into components attributable to variability within specified subgroups and variability between these subgroups.
\end{lemma}
\begin{proof}
    Let \(p := P(\mathbb{A})\) so that \(P(\mathbb{A}^\complement)=1-p\).  
    Define the conditional means
    \[
    \mu_{\mathbb{A}} := \mathbb{E}[X \mid X \in \mathbb{A}], 
    \qquad
    \mu_{\mathbb{A}^\complement} := \mathbb{E}[X \mid X \in \mathbb{A}^\complement].
    \]
    By the law of total expectation,
    \[
    \mathbb{E}[X] \;=\; p\,\mu_{\mathbb{A}} + (1-p)\,\mu_{\mathbb{A}^\complement}.
    \]
    
    Next, decompose the second moment via the same partition:
    \[
    \mathbb{E}[XX^{T}]
     \;=\; p\,\mathbb{E}[XX^{T} \mid X \in \mathbb{A}]
          + (1-p)\,\mathbb{E}[XX^{T} \mid X \in \mathbb{A}^\complement].
    \]
    Within each conditional expectation insert
    \[
    \operatorname{Var}(X \mid X \in \mathbb{A})
      \;=\;
      \mathbb{E}[XX^{T} \mid X \in \mathbb{A}] - \mu_{\mathbb{A}}\mu_{\mathbb{A}}^{T},
    \]
    and its analogue for \(\mathbb{A}^\complement\), to find
    \[
    \mathbb{E}[XX^{T}]
      \;=\;
      p\Bigl(\operatorname{Var}(X \mid X \in \mathbb{A}) + \mu_{\mathbb{A}}\mu_{\mathbb{A}}^{T}\Bigr)
      + (1-p)\Bigl(\operatorname{Var}(X \mid X \in \mathbb{A}^\complement) + \mu_{\mathbb{A}^\complement}\mu_{\mathbb{A}^\complement}^{T}\Bigr).
    \]
    
    Subtracting \(\mathbb{E}[X]\,\mathbb{E}[X]^{T}\) yields
    \[
    \operatorname{Var}(X)
      \;=\;
      p\,\operatorname{Var}(X \mid X \in \mathbb{A})
      + (1-p)\,\operatorname{Var}(X \mid X \in \mathbb{A}^\complement)
      + p(1-p)
        \bigl(\mu_{\mathbb{A}} - \mu_{\mathbb{A}^\complement}\bigr)
        \bigl(\mu_{\mathbb{A}} - \mu_{\mathbb{A}^\complement}\bigr)^{T}.
    \]
    Recalling \(p = P(\mathbb{A})\) completes the decomposition.
    \end{proof}

\newcommand{\phiStd}{\phi}
\newcommand{\PhiStd}{\Phi}

\begin{lemma}[Integrals of a product of two Gaussians truncated above \(a\)]
    \label{lem:trunc_prod_gaussians}
    Let \(\phi,\Phi\) denote the standard normal pdf and CDF. Fix \(\tau^{2}>0\), \(t>0\), \(\mu\in\mathbb R\), and \(a\in\mathbb R\). Define
    \[
    \eta \;=\;\frac{\tau^{2}\mu}{\tau^{2}+t},\qquad
    \sigma^{2}\;=\;\frac{\tau^{2}t}{\tau^{2}+t},\qquad
    k\;=\;\frac{1}{\sqrt{2\pi(\tau^{2}+t)}}\,\exp\!\Bigl(-\frac{\mu^{2}}{2(\tau^{2}+t)}\Bigr),
    \]
    and the standardized cutoff \(z := (a-\eta)/\sigma\). Set
    \[
    \Phi_{\mu}:=\eta+\sigma\,\frac{\phi(z)}{1-\Phi(z)}, \qquad
    \Phi_{\sigma}:=\sigma^{2}\!\left[1+z\,\frac{\phi(z)}{1-\Phi(z)}-\left(\frac{\phi(z)}{1-\Phi(z)}\right)^{2}\right].
    \]
    Then
    \begin{align*}
    \text{\emph{(i)}}\quad &\int_{a}^{\infty} N(x \mid 0,\tau^{2})\,N(x \mid \mu,t)\,dx \;=\; k\,[1-\Phi(z)], \\[2pt]
    \text{\emph{(ii)}}\quad &\int_{a}^{\infty} x\,N(x \mid 0,\tau^{2})\,N(x \mid \mu,t)\,dx \;=\; k\,[1-\Phi(z)]\,\Phi_{\mu}, \\[2pt]
    \text{\emph{(iii)}}\quad &\int_{a}^{\infty} x^{2}\,N(x \mid 0,\tau^{2})\,N(x \mid \mu,t)\,dx \;=\; k\,[1-\Phi(z)]\,\bigl[\Phi_{\sigma}+(\Phi_{\mu})^{2}\bigr].
    \end{align*}
\end{lemma}

\begin{proof}
A standard completion-of-squares argument gives a Gaussian-in-\(x\) representation for the product:
\begin{equation}
N(x\mid 0,\tau^{2})\,N(x\mid \mu,t)
\;=\;
k\,N(x\mid \eta,\sigma^{2}),
\label{eq:prod-two-normals}
\end{equation}
with \(\eta,\sigma^{2},k\) as stated. Consequently, for any integrable test function \(h\),
\begin{equation}
\int_{a}^{\infty} h(x)\,N(x\mid 0,\tau^{2})\,N(x\mid \mu,t)\,dx
\;=\;
k\int_{a}^{\infty} h(x)\,N(x\mid \eta,\sigma^{2})\,dx.
\label{eq:factorization}
\end{equation}

Let \(X\sim N(\eta,\sigma^{2})\) and write \(z=(a-\eta)/\sigma\). Then
\[
\int_{a}^{\infty} N(x\mid \eta,\sigma^{2})\,dx
=
\Pr(X>a)
=
1-\Phi(z),
\]
which combined with \eqref{eq:factorization} for \(h\equiv 1\) yields (i).

For the first moment, using \(x=\eta+\sigma u\) with \(u=(x-\eta)/\sigma\),
\[
\int_{a}^{\infty} x\,N(x\mid \eta,\sigma^{2})\,dx
=
\eta \int_{a}^{\infty} N(x\mid \eta,\sigma^{2})\,dx
+
\sigma \int_{z}^{\infty} u\,\phi(u)\,du.
\]
Since \(\int_{z}^{\infty} u\,\phi(u)\,du=\phi(z)\), we obtain
\[
\int_{a}^{\infty} x\,N(x\mid \eta,\sigma^{2})\,dx
=
\eta\,[1-\Phi(z)]+\sigma\,\phi(z)
=
[1-\Phi(z)]\!\left(\eta+\sigma\frac{\phi(z)}{1-\Phi(z)}\right)
=
[1-\Phi(z)]\,\Phi_{\mu}.
\]
Plugging this into \eqref{eq:factorization} with \(h(x)=x\) gives (ii).

For the second moment, similarly expand \(x^{2}=(\eta+\sigma u)^{2}=\eta^{2}+2\eta\sigma u+\sigma^{2}u^{2}\) to get
\[
\int_{a}^{\infty} x^{2}\,N(x\mid \eta,\sigma^{2})\,dx
=
\eta^{2}[1-\Phi(z)]
+2\eta\sigma\!\int_{z}^{\infty} u\,\phi(u)\,du
+\sigma^{2}\!\int_{z}^{\infty} u^{2}\phi(u)\,du.
\]
We already have \(\int_{z}^{\infty} u\,\phi(u)\,du=\phi(z)\), and integration by parts gives
\(\int_{z}^{\infty} u^{2}\phi(u)\,du = z\,\phi(z) + [1-\Phi(z)]\). Hence
\[
\int_{a}^{\infty} x^{2}\,N(x\mid \eta,\sigma^{2})\,dx
=
(\eta^{2}+\sigma^{2})[1-\Phi(z)]
+\sigma\phi(z)\,(2\eta+\sigma z).
\]
It is convenient to express this in terms of the truncated-normal mean and variance. With
\(\lambda(z):=\phi(z)/(1-\Phi(z))\), we have \(\Phi_{\mu}=\eta+\sigma\lambda(z)\) and
\(\Phi_{\sigma}=\sigma^{2}\bigl(1+z\lambda(z)-\lambda(z)^{2}\bigr)\), so that
\[
\Phi_{\sigma}+(\Phi_{\mu})^{2}
=
\eta^{2}+\sigma^{2}+2\eta\sigma\lambda(z)+\sigma^{2}z\lambda(z).
\]
Multiplying by \(1-\Phi(z)\) yields exactly the previous expression for
\(\int_{a}^{\infty} x^{2}N(x\mid \eta,\sigma^{2})\,dx\). Substituting into
\eqref{eq:factorization} with \(h(x)=x^{2}\) gives (iii).
\end{proof}

\section{Convergence Analysis}
\label{sec:single_stage_convergence_proof}
\subsection{Zeroth-Order Gradient Bounds Proof}
\label{sec:zeroth_order_gradient_bounds_proof}
\begin{restatable}[Gradient Estimator Bounds]{theorem}{GradientBounds}
    \label{thm:gradient-bounds}
    Given a zeroth-order gradient estimator~\cref{eq:zero_order_gradient_approx},
    its bias and variance are bounded as: 
    \begin{align}
        \mathbb{E}[\|\nabla_x g^{(0)}(x;t) - \nabla_x g(x;t)\|_2]   & \leq \frac{\lambda d}{N} \left(M_{-\frac{1}{2}} t^{-\frac{1}{2}} + M_0 \frac{L}{\lambda} + M_\frac{1}{2} \left(\frac{L}{\lambda}\right)^2 t^{\frac{1}{2}}\right) \label{eq:bias-bound} \\
        \mathbb{E}[\|\nabla_x g^{(0)}(x;t) - \nabla_x g(x;t)\|_2^2] & \leq \frac{\lambda^2 d}{N} \left(V_{-1} t^{-1} + V_0 \left(\frac{L}{\lambda}\right)^2 + V_1 \left(\frac{L}{\lambda}\right)^4 t\right) \label{eq:variance-bound}
    \end{align}
    where $p \in (1, +\infty)$ and $M_{-\frac{1}{2}}, M_0, M_\frac{1}{2}, V_{-1}, V_0, V_1$ are positive constants that are independent of $t, N, \lambda, L$ and depend on the problem dimension $n$.
\end{restatable}
\begin{lemma}[Moment Bounds for Lipschitz Functions]\label{lem:moment-bounds}
    Let $x \sim \mathcal{N}(0, t)$ be a random variable and $f: \mathbb{R} \to \mathbb{R}$ be an $L$-Lipschitz function. Then the following bounds hold for the raw moments and central moments:
    \begin{align}
        \mathbb{E}[f(x)^{2n}]                                        & \leq C^{(f)}_{n} L^{2n} t^n \label{eq:raw-moment-f}                                \\
        m_{2n}[f(x)] = \mathbb{E}[|f(x) - \mathbb{E}[f(x)]|^{2n}]    & \leq C^{(f)}_{n} L^{2n} t^n \label{eq:central-moment-f}                            \\
        \mathbb{E}[(x f(x))^{2n}]                                    & \leq C^{(xf)}_{n} t^n + C^{(xf)}_{2n} L^{2n} t^{2n} \label{eq:raw-moment-xf}       \\
        m_{2n}[xf(x)] = \mathbb{E}[|xf(x) - \mathbb{E}[xf(x)]|^{2n}] & \leq C^{(xf)}_{n} t^{n} + C^{(xf)}_{2n} L^{2n} t^{2n} \label{eq:central-moment-xf}
    \end{align}
    where $C^{(f)}_{n}$ and $C^{(xf)}_{n}$ are constants depending only on $n$.
\end{lemma}

\begin{proof}
    We begin with the triangle inequality: for any $x, y \in \mathbb{R}$ and $n \in \mathbb{N}$,
    \begin{equation}
        (x+y)^n \leq 2^{n-1}(x^n + y^n)
    \end{equation}

    Applying this to the central moment of $xf(x)$:
    \begin{align}
        \mathbb{E}[|xf(x) - \mathbb{E}[xf(x)]|^{2n}] & \leq 2^{2n-1} (\mathbb{E}[|xf(x)|^{2n}] + |\mathbb{E}[xf(x)]|^{2n})
    \end{align}

    For the first term, we use the fact that $f$ is $L$-Lipschitz, which means $|f(x) - f(0)| \leq L|x|$. This implies:
    \begin{align}
        |f(x)|   & \leq |f(0)| + L|x|                                                        \\
        |x f(x)| & \leq |x| \cdot |f(x)| \leq |x| \cdot (|f(0)| + L|x|) = |f(0)| |x| + L x^2
    \end{align}

    Therefore:
    \begin{align}
        (x f(x))^{2n}             & \leq 2^{2n-1} (|f(0)|^{2n} |x|^{2n} + L^{2n} x^{4n})                         \\
        \mathbb{E}[|x f(x)|^{2n}] & \leq 2^{2n-1} (|f(0)|^{2n} \mathbb{E}[|x|^{2n}] + L^{2n} \mathbb{E}[x^{4n}])
    \end{align}

    Since $x \sim \mathcal{N}(0, t)$, we know that $\mathbb{E}[|x|^{2n}] = C_n t^n$ and $\mathbb{E}[x^{4n}] = C_{2n} t^{2n}$ for some constants $C_n, C_{2n}$ depending only on $n$. Thus:
    \begin{align}
        \mathbb{E}[|x f(x)|^{2n}] & \leq 2^{2n-1} (|f(0)|^{2n} C_n t^n + L^{2n} C_{2n} t^{2n}) \\
                                  & = C_{n}' t^{n} + C_{2n}' L^{2n} t^{2n}
    \end{align}
    where we've absorbed the constants into new constants $C_{n}'$ and $C_{2n}'$.

    For the second term, we have:
    \begin{align}
        \mathbb{E}[|x f(x)|] & \leq \mathbb{E}[|f(0)| |x|] + \mathbb{E}[L x^2]          \\
                             & = |f(0)| \mathbb{E}[|x|] + L \mathbb{E}[x^2]             \\
                             & = |f(0)| K_{\frac{1}{2}} t^{\frac{1}{2}} + L \cdot K_1 t
    \end{align}

    where $K_{\frac{1}{2}} = \sqrt{\frac{2}{\pi}}$ and $K_1 = 1$ for the standard normal distribution scaled by $\sqrt{t}$.

    Raising this to the power of $2n$:
    \begin{align}
        \mathbb{E}[|x f(x)|]^{2n} & \leq 2^{2n-1} ((|f(0)| K_{\frac{1}{2}} t^{\frac{1}{2}})^{2n} + (L \cdot K_1 t)^{2n}) \\
                                  & = 2^{2n-1} (|f(0)|^{2n} K_{\frac{1}{2}}^{2n} t^{n} + L^{2n} K_1^{2n} t^{2n})         \\
                                  & = C^{(2)}_{n} t^{n} + C^{(2)}_{2n} L^{2n} t^{2n}
    \end{align}

    Therefore:
    \begin{align}
        \mathbb{E}[|xf(x) - \mathbb{E}[xf(x)]|^{2n}] & \leq 2^{2n-1} (\mathbb{E}[|xf(x)|^{2n}] + |\mathbb{E}[xf(x)]|^{2n})                                                 \\
                                                     & \leq 2^{2n-1} ((C^{(1)}_{n} t^{n} + C^{(1)}_{2n} L^{2n} t^{2n}) + (C^{(2)}_{n} t^{n} + C^{(2)}_{2n} L^{2n} t^{2n})) \\
                                                     & = 2^{2n-1} ((C^{(1)}_{n} + C^{(2)}_{n}) t^{n} + (C^{(1)}_{2n} + C^{(2)}_{2n}) L^{2n} t^{2n})                        \\
                                                     & = C^{(xf)}_{n} t^n + C^{(xf)}_{2n} L^{2n} t^{2n}
    \end{align}
    where we've combined all constants into final constants $C^{(xf)}_{n}$ and $C^{(xf)}_{2n}$ as stated in equation \eqref{eq:central-moment-xf}.

    The proofs for the other bounds follow similar reasoning, applying the Lipschitz property of $f$ and the moment properties of the normal distribution.
\end{proof}

\GradientBounds*

\begin{proof}[Proof of \cref{thm:gradient-bounds}]
    We apply the non-asymptotic moment bound of self-normalized importance sampling estimators~\cite{agapiouImportanceSamplingIntrinsic2017}. When the following quantity is bounded:
    \begin{align}
        \mathcal{C}_{\rm MSE} & := \frac{3}{\pi(g)^2}m_2[\phi g] + \frac{3}{\pi(g)^4} \pi(|\phi g|^{2r})^\frac{1}{r}{\Gamma}_{2s}^\frac{1}{s} m_{2s}[\phi g]^\frac{1}{s}            \\
                              & \quad + \frac{3}{\pi(g)^{2(1+\frac{1}{p})}}\pi(|\phi|^{2p})^\frac{1}{p}{\Gamma}_{2q(1+\frac{1}{p})}^\frac{1}{q}m_{2q(1+\frac{1}{p})}[g]^\frac{1}{q}
    \end{align}

    where the constants ${\Gamma}_t > 0, t \geq 2$, satisfy $\Gamma_t^\frac{1}{t} \leq t-1$ and the two pairs of parameters $r,s$, and $p,q$ are conjugate pairs of indices satisfying $r,s,p,q \in (1,\infty)$ and $r^{-1}+s^{-1}=1$, $p^{-1}+q^{-1}=1$.

    The bias and MSE of the importance sampling estimator are bounded by:
    \begin{align}
        \left| \mathbb{E}\left[ \mu^N(\phi) - \mu(\phi) \right] \right|   & \leq \frac{1}{N}
        \Biggl(\frac{2}{\pi(g)^2}m_2[g]^\frac{1}{2}m_2[\phi g]^\frac{1}{2}+2\mathcal{C}_{\rm MSE}^\frac{1}{2}\frac{\pi(g^2)^\frac{1}{2}}{\pi(g)}\Biggr), \label{eq:bias-IS} \\
        \mathbb{E}\left[ \left( \mu^N(\phi) - \mu(\phi) \right)^2 \right] & \leq \frac{1}{N} \mathcal{C}_{\rm MSE} \label{eq:mse-IS}
    \end{align}

    In our case, the test function is $\phi(x) = x$ and the target function is $g(x) = \exp\left(-\frac{J(x)}{\lambda}\right)$. Given that $J(x)$ is $L$-Lipschitz, the function $g$ is $\frac{L}{\lambda}$-Lipschitz.

    Note that $\pi(g)$ is the normalizer, which is constant up to a multiplicative factor. Thus, we focus on the moment terms like $m_2[\phi g]$, $m_2[g]$, $\pi(|\phi g|^{2d})$, and $\pi(|g|^{2q})$.

    For the first term in $\mathcal{C}_{\rm MSE}$, applying the central moment lemma from \cref{lem:moment-bounds}:
    \begin{align}
        m_2[\phi g] \leq A_1 t d + A_2 \left(\frac{L}{\lambda}\right)^2 t^2 d
    \end{align}

    For the second term in $\mathcal{C}_{\rm MSE}$:
    \begin{align}
        \pi(|\phi g|^{2r})^{\frac{1}{r}} & \leq \left(B_r t^r d^r + B_{2r} \left(\frac{L}{\lambda}\right)^{2r} t^{2r} d^r\right)^{\frac{1}{r}} \\
                                         & \leq (B_1 t + B_2 \left(\frac{L}{\lambda}\right)^{2} t^{2})d                                        \\
        m_{2e}[g]^{\frac{1}{e}}          & \leq B_3 \left(\frac{L}{\lambda}\right)^{2} t d
    \end{align}

    For the third term in $\mathcal{C}_{\rm MSE}$:
    \begin{align}
        \pi(|\phi|^{2p})^{\frac{1}{p}}         & \leq S_1 t d                                                                                  \\
        m_{2q(1+\frac{1}{p})}[g]^{\frac{1}{q}} & \leq S_{1+\frac{1}{p}} \left(\frac{L}{\lambda}\right)^{2 + \frac{2}{p}} t^{1 + \frac{1}{p}} d
    \end{align}

    Combining these results, we get (for the estimation of $\mathbb{E}_{y \sim p_{0|t}}[y|x]$):
    \begin{align}
        \left| \mathbb{E}\left[ \mu^N(\phi) - \mu(\phi) \right] \right|   & \leq
        \frac{d}{N} \left(E_{\frac{1}{2}} t^{\frac{1}{2}} + E_1 \frac{L}{\lambda} t + E_{1+\frac{1}{2p}} \left(\frac{L}{\lambda}\right)^{1 + \frac{1}{p}} t^{1+\frac{1}{2p}} + E_\frac{3}{2} \left(\frac{L}{\lambda}\right)^2 t^{\frac{3}{2}}\right) \label{eq:combined-bias}                                          \\
        \mathbb{E}\left[ \left( \mu^N(\phi) - \mu(\phi) \right)^2 \right] & \leq \frac{d}{N} \left(F_1 t + F_2 \left(\frac{L}{\lambda}\right)^2 t^2 + F_{2+\frac{1}{p}} \left(\frac{L}{\lambda}\right)^{2 + \frac{2}{p}} t^{2 + \frac{1}{p}} + F_3 \left(\frac{L}{\lambda}\right)^4 t^3\right) \label{eq:combined-mse}
    \end{align}

    Note that
    \begin{align}
        \mathbb{E}[|\nabla_x g^{(0)}(x;t) - \nabla_x g(x;t)|]   & = \frac{\lambda}{t} \mathbb{E}[|\mu^N(\phi) - \mu(\phi)|] \label{eq:grad-transform-bias}      \\
        \mathbb{E}[|\nabla_x g^{(0)}(x;t) - \nabla_x g(x;t)|^2] & = \frac{\lambda^2}{t^2} \mathbb{E}[|\mu^N(\phi) - \mu(\phi)|^2] \label{eq:grad-transform-mse}
    \end{align}

    Substituting these relationships from equations \eqref{eq:combined-bias} and \eqref{eq:combined-mse} into \eqref{eq:grad-transform-bias} and \eqref{eq:grad-transform-mse}, and simplifying the exponents of $t$, we arrive at the final result as stated in equations \eqref{eq:bias-bound} and \eqref{eq:variance-bound}:
    \begin{align}
        \mathbb{E}[|\nabla_x g^{(0)}(x;t) - \nabla_x g(x;t)|]   & \leq \frac{\lambda d}{N} \left(M_{-\frac{1}{2}} t^{-\frac{1}{2}} + M_0 \frac{L}{\lambda} + M_{\frac{1}{2p}} \left(\frac{L}{\lambda}\right)^{1 + \frac{1}{p}} t^\frac{1}{2p} + M_\frac{1}{2} \left(\frac{L}{\lambda}\right)^2 t^{\frac{1}{2}}\right) \\
        \mathbb{E}[|\nabla_x g^{(0)}(x;t) - \nabla_x g(x;t)|^2] & \leq \frac{\lambda^2 d}{N} \left(V_{-1} t^{-1} + V_0 \left(\frac{L}{\lambda}\right)^2 + V_{\frac{1}{p}} \left(\frac{L}{\lambda}\right)^{2 + \frac{2}{p}} t^\frac{1}{p} + V_1 \left(\frac{L}{\lambda}\right)^4 t\right)
    \end{align}

    where $M_{-\frac{1}{2}}, M_0, M_\frac{1}{2}, M_{\frac{1}{2p}}, M_\frac{1}{2}, V_{-1}, V_0, V_1, V_{\frac{1}{p}}$ are positive constants.

    Given that $p \in (1, +\infty)$, under worst case, $p \to 1$, plugging in $p = 1$ into the above bound, term $M_{\frac{1}{2p}}$ and $V_{\frac{1}{p}}$ can be merged into $M_\frac{1}{2}$ and $V_1$ respectively.

    \begin{align}
        \mathbb{E}[|\nabla_x g^{(0)}(x;t) - \nabla_x g(x;t)|]   & \leq \frac{\lambda d}{N} \left(M_{-\frac{1}{2}} t^{-\frac{1}{2}} + M_0 \frac{L}{\lambda} + M_\frac{1}{2} \left(\frac{L}{\lambda}\right)^2 t^{\frac{1}{2}}\right) \\
        \mathbb{E}[|\nabla_x g^{(0)}(x;t) - \nabla_x g(x;t)|^2] & \leq \frac{\lambda^2 d}{N} \left(V_{-1} t^{-1} + V_0 \left(\frac{L}{\lambda}\right)^2 + V_1 \left(\frac{L}{\lambda}\right)^4 t\right)
    \end{align}

\end{proof}

\subsection{Single-Stage Convergence with Fixed Smoothing Parameter}

\begin{theorem}[Convergence of SGD with bounded bias and variance]
    \label{thm:single_stage_fixed_t}
    Let $f:\mathbb{R}^d \to \mathbb{R}$ be an $\alpha$-strongly convex and $\beta$-smooth function. Let the step size satisfy $\eta \leq \frac{\alpha}{4 \beta^2}$. Assume a gradient estimator $\nabla \hat{f}(x_k) = \nabla f(x_k) + b_k + w_k$ with bounded bias $\|b_k\| \le K$ and bounded variance $\mathbb{E}[\|w_k\|^2] \le \sigma^2$.
    Then, with probability at least $(1 - \delta)$, the error satisfies
    \begin{align*}
        \| x_k - x^* \|^2 \le (1 - \frac{\eta \alpha}{2})^k \| x_0 - x^* \|^2 + (\frac{\eta}{\alpha} + 2\eta^2) \left( \frac{1}{\delta} \frac{4(K^2 + \sigma^2)}{\eta \alpha} \right)
    \end{align*}
\end{theorem}

\begin{proof}
    Let $\Delta_k = x_k - x^*$. Then,
    \begin{align*}
        \| \Delta_{k+1} \|^2 & = \| \Delta_k - \eta \nabla f(x_k) - \eta b_k - \eta w_k \|^2                                                                                                \\
                             & = \| \Delta_k \|^2 - 2 \eta \langle \Delta_k, \nabla f(x_k) \rangle - 2 \eta \langle \Delta_k, b_k + w_k \rangle + \eta^2 \| \nabla f(x_k) + b_k + w_k \|^2.
    \end{align*}

    For the gradient term, using $\alpha$-strong convexity:
    \begin{align*}
        -2 \eta \langle \Delta_k, \nabla f(x_k) \rangle \le -2 \eta \alpha \| \Delta_k \|^2.
    \end{align*}

    For bias and noise term, first using Cauchy-Schwarz inequality:
    \begin{align*}
        -2 \eta \langle \Delta_k, b_k + w_k \rangle \le 2 \eta \| \Delta_k \| \| b_k + w_k \|.
    \end{align*}

    Then using Young's inequality ($2ab \le \epsilon a^2 + b^2/\epsilon$, let $\epsilon = \alpha$):
    \begin{align*}
        2 \eta \| \Delta_k \| \| b_k + w_k \| \le \eta \alpha \| \Delta_k \|^2 + \frac{\eta}{\alpha} \| b_k + w_k \|^2.
    \end{align*}

    For the quadratic term, using $(a+b)^2 \le 2a^2 + 2b^2$ and $\beta$-smoothness $\|\nabla f(x_k) \| \leq \beta \| \Delta_k \|$:
    \begin{align*}
        \eta^2 \| \nabla f(x_k) + b_k + w_k \|^2 \le 2\eta^2 \beta^2 \| \Delta_k \|^2 + 2\eta^2 \| b_k + w_k \|^2.
    \end{align*}

    Organize the terms, we get recurrence:
    \begin{align*}
        \| \Delta_{k+1} \|^2 \le (1 - 2\eta \alpha + \eta \alpha + 2\eta^2 \beta^2) \| \Delta_k \|^2 + (2\eta^2 + \frac{\eta}{\alpha}) \| b_k + w_k \|^2.
    \end{align*}

    For contraction coefficient $\rho$, due to $2 \eta \beta^2 \le \frac{\alpha}{2}$, we have:
    \begin{align*}
        \rho = 1 - 2\eta \alpha + \eta \alpha + 2\eta^2 \beta^2 \le 1 - \frac{\eta \alpha}{2}.
    \end{align*}

    The recurrence is now:
    \begin{align*}
        \| \Delta_{k+1} \|^2 \le \rho \| \Delta_k \|^2 + (\frac{\eta}{\alpha} + 2\eta^2) \| b_k + w_k \|^2.
    \end{align*}

    Unrolling the recurrence from $k$ down to 0, we get:
    \begin{align*}
        \| \Delta_k \|^2 \le \rho^k \| \Delta_0 \|^2 + (\frac{\eta}{\alpha} + 2\eta^2) \sum_{i=0}^{k-1} \rho^{k-1-i} \| b_i + w_i \|^2.
    \end{align*}

    For the bias and variance term, its expected value is bounded by:
    \begin{align*}
        \mathbb{E}[\| b_i + w_i \|^2] \le \mathbb{E}[2\| b_i \|^2 + 2\| w_i \|^2] \le 2\mathbb{E}[\| b_i \|^2] + 2\mathbb{E}[\| w_i \|^2] \le 2K^2 + 2\sigma^2.
    \end{align*}

    Then the expected value of the sum $Z = \sum_{i=0}^{k-1} \rho^{k-1-i} \| b_i + w_i \|^2$ is bounded by:
    \begin{align*}
        \mathbb{E}[Z] & = \sum_{i=0}^{k-1} \rho^{k-1-i} \mathbb{E}[\| b_i + w_i \|^2]                        \\
                      & \leq \sum_{i=0}^{\infty} \rho^{i} (2K^2 + 2\sigma^2)                                 \\
                      & = \frac{2(K^2 + \sigma^2)}{1 - \rho}                                                 \\
                      & = \frac{2(K^2 + \sigma^2)}{\eta \alpha / 2} = \frac{4(K^2 + \sigma^2)}{\eta \alpha}.
    \end{align*}

    With Markov's inequality, we have:
    \begin{align*}
        P(Z \ge \epsilon) \le \frac{\mathbb{E}[Z]}{\epsilon} = \frac{4(K^2 + \sigma^2)}{\eta \alpha \epsilon}.
    \end{align*}

    With probability at least $(1 - \delta)$, we have:
    \begin{align*}
        Z \le \frac{1}{\delta} \frac{4(K^2 + \sigma^2)}{\eta \alpha}.
    \end{align*}

    Substitute the bound for $Z$ back into the unrolled equation, we get:
    \begin{align*}
        \| \Delta_k \|^2 \le \rho^k \| \Delta_0 \|^2 + \frac{1}{\delta} (\frac{\eta}{\alpha} + 2\eta^2) \frac{4(K^2 + \sigma^2)}{\eta \alpha}.
    \end{align*}

\end{proof}

\SingleStageConvergence*

\begin{proof}

    Plugging in gradient estimator bounds and step size $\eta = \frac{\alpha}{4\beta^2}$, we get:
    \begin{align*}
        \| x_k - x_k^* \|^2 \le (1 - \frac{1}{4 \kappa_t^2})^k \| x_0 - x^* \|^2 + \frac{4(K_t^2 + \sigma_t^2)}{\delta \ \alpha_t} (\frac{t}{2\lambda} + \frac{1}{\alpha_t})
    \end{align*}

    Use triangle inequality, we have:
    \begin{align*}
        \| x_k - x^* \|^2 \le \| x_k - x_k^* \|^2 + \| x_k^* - x^* \|^2
    \end{align*}

    Plugging in the bound for $\| x_k - x_k^* \|^2$, we get:
    \begin{align*}
        \| x_k - x^* \|^2 \le (1 - \frac{1}{4 \kappa_t^2})^k \| x_0 - x^* \|^2 + \frac{4(K_t^2 + \sigma_t^2)}{\delta \ \alpha_t} (\frac{t}{2\lambda} + \frac{1}{\alpha_t}) + \| x_k^* - x^* \|^2
    \end{align*}
\end{proof}
\subsection{Multi-Stage Convergence with Varying Smoothing Parameter}
\label{sec:multi_stage_convergence_proof}

\begin{theorem}[Global Convergence of Multi-Stage Algorithm]
    \label{thm:multi_stage_convergence}
    Let $\{g(x; t)\}_{t \ge 0}$ be a family of objective functions with minimizer $x^*_t$ such that $\|x^*_t - x^*\|^2 \le k_b t$.
    Assume gradient estimator bias $\|b(x; t)\| \le K_m$ and variance $\mathbb{E}[\|w(x; t)\|^2] \le \sigma_m^2$.
    Let $\rho_m = 1 - \frac{\alpha_m^2}{8 \beta_m^2}$. Define the effective contraction $\tilde{\rho}_m = (1+\frac{1-\rho_m}{4})^2 \rho_m < 1$.
    With the following feasible condition for the noise and bias:
    \begin{align}
        \sigma_m^2 + K_m^2 & \le \frac{4E_0 r_\text{min} \beta_1^2 \delta}{3M} \\
        \frac{k_g}{k_b}    & \le K_0
    \end{align}
    where $\rho'_m = \tilde{\rho}_m + \frac{\mathcal{P}_m k_b}{k_g}, \epsilon_m = \frac{1-\rho_m}{4}, \mathcal{P}_m = (1+\epsilon_m)\rho_m(1+\epsilon_m^{-1}) + (1+\epsilon_m^{-1}), E_0 = \frac{512\,\kappa_0^4 + 56\,\kappa_0^2 + 1}{8192\,\kappa_0^6 + 256\,\kappa_0^4}, K_0 = \frac{\kappa_0^2 (524288 \kappa_0^6 + 57344 \kappa_0^4 + 2304 \kappa_0^2 + 32)}{512 \kappa_0^4 + 56\kappa_0^2 + 1}$, $\kappa_0$ is the condition number of the initial objective function $f(x) = g(x; t = 0)$.
    Then if initial iterate $\|x_0 - x^*\|^2 \le r_\text{min} + k_g t_0$, the sequence $\{x_m\}_{m=0}^M$ satisfies $\| x_m - x^* \|^2 \le r_\text{min} + k_g t_m$ for all $m$ with probability $1-\delta$.
\end{theorem}

\begin{proof}
    By the Union Bound over $M$ stages, it suffices to show that $\|x_{m+1} - x^*\|^2 \le r_\text{min} + k_g t_{m+1}$ given $\|x_m - x^*\|^2 \le r_\text{min} + k_g t_m$.
    Using Young's Inequality twice with $\epsilon_m = \frac{1-\rho_m}{4}$, with probability at least $1-\frac{\delta}{M}$:
    \begin{align*}
        \|x_{m+1} - x^*\|^2 & \le (1+\epsilon_m) \|x_{m+1} - x^*_{m}\|^2 + (1+\epsilon_m^{-1}) \|x^*_{m} - x^*\|^2                         \\
                            & \le (1+\epsilon_m) \left[ \rho_m \|x_m - x^*_{m}\|^2 + E_m \right] + (1+\epsilon_m^{-1}) \|x^*_{m} - x^*\|^2
    \end{align*}
    where $E_m = \frac{2M}{\delta}(\frac{\eta_m}{\alpha_m} + 2 \eta_m^2) (K_m^2 + \sigma_m^2)$.
    We expand $\|x_m - x^*_{m}\|^2 \le (1+\epsilon_m)\|x_m - x^*\|^2 + (1+\epsilon_m^{-1})\|x^* - x^*_{m}\|^2$. Substituting this back and grouping terms:
    \begin{align*}
        \|x_{m+1} - x^*\|^2 & \le \underbrace{(1+\epsilon_m)^2 \rho_m}_{\tilde{\rho}_m} \|x_m - x^*\|^2 + (1+\epsilon_m)E_m + \underbrace{\left[ (1+\epsilon_m)\rho_m(1+\epsilon_m^{-1}) + (1+\epsilon_m^{-1}) \right]}_{\mathcal{P}_m} \|x^*_{m} - x^*\|^2
    \end{align*}
    To make sure the next stage still stay within $\|x_{m+1} - x^*\|^2 \le r_\text{min} + k_g t_{m+1}$, we require:
    \begin{align*}
        \tilde{\rho}_m (r_\text{min} + k_g t_m) + (1+\epsilon_m)E_m + \mathcal{P}_m k_b t_{m} \le r_\text{min} + k_g t_{m+1}
    \end{align*}
    Rearranging terms, we get:
    \begin{align*}
        t_{m+1} \ge \underbrace{(\tilde{\rho}_m + \frac{\mathcal{P}_m k_b}{k_g}) t_m}_{\text{contraction factor}} - \underbrace{\frac{1}{k_g} (1 -\tilde{\rho}_m) r_{\min}}_{\text{bias term}} + \underbrace{\frac{1}{k_g} (1 + \epsilon_m) E_m}_{\text{noise term}}
    \end{align*}
    To make sure $t_m$ is decreasing, let
    \begin{align*}
        (1 + \epsilon_m) E_m - (1 -\tilde{\rho}_m) r_{\min} & \le 0 \\
        \tilde{\rho}_m + \frac{\mathcal{P}_m k_b}{k_g}      & < 1
    \end{align*}
    For constration factor, plugging in bounds for step size:
    \begin{align*}
        \rho_m = 1 - \frac{\alpha_m^2}{8 \beta_m^2} = 1 - \frac{1}{8 \kappa_m^2}
    \end{align*}
    where $\kappa_m = \frac{\beta_m}{\alpha_m} \in [1, \kappa_0]$, where $\kappa_0$ is the condition number of the initial objective function.
    Given $\epsilon_m = \frac{1-\rho_m}{4}$, $\tilde{\rho}_m = (1+\epsilon_m)^2 \rho_m = (1 + \epsilon_m)^2 (1 - 4 \epsilon_m)$, $\mathcal{P}_m = (1+\epsilon_m)\rho_m(1+\epsilon_m^{-1}) + (1+\epsilon_m^{-1}) = 2 \epsilon_m^{-1} + 9 \epsilon_m + 4 \epsilon_m^2 + 7$, we have contraction factor is bounded by:
    \begin{align*}
        \frac{7}{8}             & \leq \rho_m \leq 1 - \frac{1}{8 \kappa_0^2}                                                     \\
        \frac{1}{32 \kappa_0^2} & \leq \epsilon_m \leq \frac{1}{32}                                                               \\
        \frac{7623}{8192}       & \leq  \tilde{\rho}_m \leq   (1 + \frac{1}{32 \kappa_0^2})^2(1 - \frac{4}{32 \kappa_0^2})        \\
        \frac{18249}{256}       & \leq  \mathcal{P}_m \leq 64 \kappa_0^2 + 7 + \frac{9}{32 \kappa_0^2} + \frac{1}{256 \kappa_0^4}
    \end{align*}
    For contration factor to be less than 1, we need convex expansion factor to be greater than:
    \begin{align*}
        k_g > \max_{m} \frac{\mathcal{P}_m k_b}{1 - \tilde{\rho}_m} \geq \frac{\kappa_0^2 (524288 \kappa_0^6 + 57344 \kappa_0^4 + 2304 \kappa_0^2 + 32) k_b}{512 \kappa_0^4 + 56\kappa_0^2 + 1}
    \end{align*}
    For bias term to be negative, we need:
    \begin{align*}
        E_m \leq \min_m \frac{(1 - \tilde{\rho}_m) r_\text{min}}{1 + \epsilon_m}
        = \frac{512\,\kappa_0^4 + 56\,\kappa_0^2 + 1}{8192\,\kappa_0^6 + 256\,\kappa_0^4} r_\text{min} = E_0 r_\text{min}
    \end{align*}
    Plugging $E_m$ back, we get:
    \begin{align*}
        K_m^2 + \sigma_m^2 \leq \min_m \frac{\delta}{2M}  \frac{E_0 r_\text{min}}{\frac{\eta_m}{\alpha_m} + 2 \eta_m^2}
    \end{align*}
    where $\frac{\eta_m}{\alpha_m} = \frac{1}{4\beta_m^2} \leq \frac{1}{4\beta_1^2}$, where $\beta_1$ is the smoothness of the smoothed objective function at $t_1$.
    Step size is bounded by $\eta_m = \frac{\alpha_m}{4\beta_m^2} = \frac{1}{4 \kappa_m \beta_m} \leq \frac{1}{4 \beta_1}$.
    Plugging in step size and smoothness, we get:
    \begin{align*}
        K_m^2 + \sigma_m^2 \leq \frac{4 E_0 r_\text{min} \beta_1^2 \delta}{3M}
    \end{align*}
\end{proof}

\DIDAConvergence*

\begin{proof}
    Consider the convex radius bound in \cref{theorem:SCBound}:
    \begin{align*}
        \mathcal{R}_{SC}(t)
        :=  \left\{x\in \mathbb{R}^d \;\middle|\;
        \|x - x^*\|
        \leq  C_E\min\!\left(
        \sqrt{\frac{t+\tfrac{\lambda}{L}}{\tfrac{\lambda}{L}}}\;,\;
        \sqrt{\frac{t+\tau^2}{\tau^2}}
        \right)D_\tau
        \right\}
    \end{align*}
    We can identify minimum convex radius as $r_\text{min} = C_E^2 D_\tau^2$ and convex radius expansion speed as $k_g = C_E^2 \min \{\frac{L^2}{\lambda^2}, \frac{1}{\tau^4}\}$.

    Consider bias bound in \cref{theorem:bias_bound}:
    \begin{align*}
        \left\| x_t^* - x^* \right\| \leq \min\left\{\frac{(1-C_\alpha) t}{C_\alpha} \frac{1}{4 D_\tau}, (D_\tau
        + \tau) \frac{t+\frac{\lambda}{\alpha}}{C_\alpha t}\right\} + \frac{\kappa -1}{ 2 C_\alpha} \sqrt{\tfrac{1}{2\pi(\tfrac{1}{t}+\tfrac{\alpha}{\lambda})}}.
    \end{align*}
    given $t \in [t_F, t_0]$, we have:
    \begin{align*}
        \left\| x_t^* - x^* \right\|
         & \leq \frac{(1-C_\alpha) t}{4 C_\alpha D_\tau} + \frac{\kappa -1}{ 2 C_\alpha} \sqrt{\frac{1}{2\pi/t}}
         &                                                                                                                             & \text{\footnotesize(Select 1st term of $\min$; drop $\frac{\alpha}{\lambda} \geq 0$)} \\
         & = \frac{1-C_\alpha}{4 C_\alpha D_\tau} t + \frac{\kappa -1}{ 2 C_\alpha \sqrt{2\pi}} \sqrt{t}
         &                                                                                                                             & \text{\footnotesize(Simplify)}                                                        \\
         & \leq \frac{1-C_\alpha}{4 C_\alpha D_\tau} t + \frac{\kappa -1}{ 2 C_\alpha \sqrt{2\pi}} \left( \frac{t}{\sqrt{t_F}} \right)
         &                                                                                                                             & \text{\footnotesize(Since $t \geq t_F \implies \sqrt{t} \leq \frac{t}{\sqrt{t_F}}$)}  \\
         & = \underbrace{\left( \frac{1-C_\alpha}{4 C_\alpha D_\tau} + \frac{\kappa -1}{ 2 C_\alpha \sqrt{2\pi t_F}} \right)}_{k_b} t
         &                                                                                                                             & \text{\footnotesize(Definition of $k_b$)}
    \end{align*}

    Next, consider gradient estimator bound in \cref{thm:gradient-bounds}:
    When $t_m < t_{M_0}$, the gradient estimator is dominated by the proposal sampling variance $t_m$.
    Since $K=O(1/N), \sigma^2=O(1/N)$, the MSE is dominated by the variance term when $N$ is large, i.e. $K^2 \ll \sigma^2$. Thus, the required sample size $N$ is bounded by:
    \begin{align*}
        K_m^2 + \sigma_m^2 \leq 2\sigma_m^2 \leq 2 \frac{\lambda_m^2 d}{N} \left( V_{-1} t_m^{-1} + V_0 (\frac{L}{\lambda_m})^2 + V_1 (\frac{L}{\lambda_m})^4 t_m \right)
    \end{align*}

    Let $D = 2 \frac{\lambda_m^2 d}{N} \left( V_{-1} t_m^{-1} + V_0 (\frac{L}{\lambda_m})^2 + V_1 (\frac{L}{\lambda_m})^4 t_m \right)$.
    Let $D \leq \frac{4 E_0 C_E^2 D_\tau^2 \beta_1^2 \delta}{3M}$, we get:
    \begin{align*}
        N \geq \frac{3 \lambda_m^2 d M}{2 E_0 C_E^2 D_\tau^2 \beta_1^2 \delta} \left( V_{-1} t_m^{-1} + V_0 (\frac{L}{\lambda_m})^2 + V_1 (\frac{L}{\lambda_m})^4 t_m \right)
    \end{align*}

    When $\lambda_m = \lambda$, to make sure $N$ works for all $m$, we need to ensure:
    \begin{align*}
        N      & \geq \max \{N(t_0), N(t_c)\},                                                                                                                                     \\
        N(t_0) & = \frac{3 \lambda_0^2 d M}{2 E_0 C_E^2 D_\tau^2 \beta_1^2 \delta} \left( V_{-1} t_0^{-1} + V_0 (\frac{L}{\lambda_0})^2 + V_1 (\frac{L}{\lambda_0})^4 t_0 \right), \\
        N(t_c) & = \frac{3 \lambda_c^2 d M}{2 E_0 C_E^2 D_\tau^2 \beta_1^2 \delta} \left( V_{-1} t_c^{-1} + V_0 (\frac{L}{\lambda_c})^2 + V_1 (\frac{L}{\lambda_c})^4 t_c \right)
    \end{align*}

    When $\lambda_m = L \sqrt{t_m}$, we have:
    \begin{align*}
        D & \le 2\frac{\lambda_m^2 d}{N} \left(V_{-1} t_m^{-1} + V_0 \left(\frac{L}{\lambda_m}\right)^2 + V_1 \left(\frac{L}{\lambda_m}\right)^4 t_m\right) \\
          & = 2\frac{L^2 t_m d}{N} \left(V_{-1} t_m^{-1} + V_0 t_m^{-1} + V_1 t_m^{-2} t_m\right)                                                           \\
          & = 2\frac{L^2 d}{N} (V_{-1} + V_0 + V_1)
    \end{align*}
    Thus, we have uniform bound for $N$:
    \begin{align*}
        N \ge \frac{3 L^2 d M}{2 E_0 C_E^2 D_\tau^2 \beta_1^2 \delta} (V_{-1} + V_0 + V_1)
    \end{align*}

    With above $N$, the sequence $\{x_m\}_{m=0}^{M_0}$ satisfies $\| x_m - x^* \|^2 \le r_\text{min} + k_g t_m$ for all $m$ with probability $1-\delta$, where $r_\text{min} = C_E^2 D_\tau^2$ and $k_g = C_E^2 \min \{\frac{L^2}{\lambda^2}, \frac{1}{\tau^4}\}$.

    Once $t_m \le t_{M_0}$, it will be fixed to $t_F < t_{M_0}$ to run local search by $M - M_0$ steps. Applying \cref{thm:single_stage_fixed_t}, we have:
    \begin{align*}
        \| x_M - x^* \|^2 \le \|x_F^* - x^*\|^2 + (1 - \frac{1}{4 \kappa_F^2})^{M - M_0} (C_E^2 D_\tau^2 + k_g t_{M_0}) + \frac{4(K_F^2 + \sigma_F^2)}{\delta \ \alpha_F} (\frac{t_F}{2\lambda_F} + \frac{1}{\alpha_F})
    \end{align*}
    where $x_F^*$ is the minimizer of the final sampling kernel $t_F$, $K_F = \frac{\lambda_F}{N} \left(M_{-\frac{1}{2}} t_F^{-\frac{1}{2}} + M_0 \frac{L_F}{\lambda_F} + M_\frac{1}{2} \left(\frac{L_F}{\lambda_F}\right)^2 t_F^{\frac{1}{2}}\right)$ and $\sigma_F^2 = \frac{\lambda_F^2}{N} \left(V_{-1} t_F^{-1} + V_0 \left(\frac{L_F}{\lambda_F}\right)^2 + V_1 \left(\frac{L_F}{\lambda_F}\right)^4 t_F\right)$.

    Finally, plugging convex radius - bias bound:
    \begin{align*}
        \frac{k_g}{k_b} \leq K_0
    \end{align*}
    where $K_0 = \frac{\kappa_0^2 (524288 \kappa_0^6 + 57344 \kappa_0^4 + 2304 \kappa_0^2 + 32)}{512 \kappa_0^4 + 56\kappa_0^2 + 1}$, $k_g = C_E^2 \min \{\frac{L^2}{\lambda^2}, \frac{1}{\tau^4}\}$ and $k_b = \frac{1-C_\alpha}{4 C_\alpha D_\tau} + \frac{\kappa -1}{ 2 C_\alpha \sqrt{2\pi t_F}}$.
\end{proof}
\section{Experiment Details}
\begin{table}[ht]
    \centering
    \resizebox{\textwidth}{!}{%
    \begin{tabular}{@{}l @{\hspace{0.5em}} >{\raggedright\arraybackslash}p{1.7cm} cccccc@{}}
        \toprule
         & Task/Env. & DIDA(ours) & MBD~\citep{panModelBasedDiffusionTrajectory2024} & SA~\citep{colemanIsotropicEffectiveEnergy1993} & MPPI~\citep{williamsInformationTheoreticModelPredictive2018} & CEM~\citep{rubinsteinCrossEntropyMethod2004} & CMA-ES~\citep{akimotoTheoreticalFoundationCMAES2012} \\
        \midrule
        \multirow{9}{*}{\makebox[0pt][r]{\rotatebox[origin=c]{90}{\parbox{2.2cm}{\centering Blackbox \\ Optimization.}}}}
         & Ackley (d=200)     & \textbf{$\mathbf{3.1 \pm \scriptstyle 0.1}$}      & $6.8 \pm \scriptstyle 0.3$               & $14.0 \pm \scriptstyle 0.1$      & $14.2 \pm \scriptstyle 0.1$      & $14.3 \pm \scriptstyle 0.1$      & $14.2 \pm \scriptstyle 0.1$      \\
         & Ackley (d=400)     & \textbf{$\mathbf{4.4 \pm \scriptstyle 0.2}$}      & $8.8 \pm \scriptstyle 0.2$               & $14.4 \pm \scriptstyle 0.0$      & $14.6 \pm \scriptstyle 0.0$      & $14.7 \pm \scriptstyle 0.1$      & $14.6 \pm \scriptstyle 0.0$      \\
         & Ackley (d=800)     & \textbf{$\mathbf{6.0 \pm \scriptstyle 0.1}$}      & $9.9 \pm \scriptstyle 0.1$               & $14.6 \pm \scriptstyle 0.1$      & $14.8 \pm \scriptstyle 0.0$      & $14.9 \pm \scriptstyle 0.0$      & $14.8 \pm \scriptstyle 0.0$      \\
         & Levy (d=200)       & \textbf{$\mathbf{11.8 \pm \scriptstyle 2.0}$}     & $210.6 \pm \scriptstyle 17.9$            & $744.3 \pm \scriptstyle 23.7$    & $744.3 \pm \scriptstyle 23.7$    & $744.3 \pm \scriptstyle 23.7$    & $744.3 \pm \scriptstyle 23.7$   \\
         & Levy (d=400)       & \textbf{$\mathbf{53.6 \pm \scriptstyle 5.0}$}     & $628.4 \pm \scriptstyle 29.1$            & $1567.4 \pm \scriptstyle 28.2$   & $1567.4 \pm \scriptstyle 28.2$   & $1567.4 \pm \scriptstyle 28.2$   & $1567.4 \pm \scriptstyle 28.2$   \\
         & Levy (d=800)       & \textbf{$\mathbf{202.5 \pm \scriptstyle 11.3}$}   & $1508.4 \pm \scriptstyle 43.7$           & $3212.5 \pm \scriptstyle 24.8$   & $3212.5 \pm \scriptstyle 24.8$   & $3212.5 \pm \scriptstyle 24.8$   & $3212.5 \pm \scriptstyle 24.8$   \\
         & Rastrigin (d=200)  & \textbf{$\mathbf{1703.3 \pm \scriptstyle 65.0}$}  & $2823.3 \pm \scriptstyle 74.3$           & $3652.6 \pm \scriptstyle 38.0$   & $3648.8 \pm \scriptstyle 43.4$   & $3644.2 \pm \scriptstyle 32.0$   & $3648.8 \pm \scriptstyle 43.4$   \\
         & Rastrigin (d=400)  & \textbf{$\mathbf{3782.1 \pm \scriptstyle 80.3}$}  & $6224.6 \pm \scriptstyle 101.5$          & $7478.4 \pm \scriptstyle 76.0$   & $7478.4 \pm \scriptstyle 76.0$   & $7478.4 \pm \scriptstyle 76.0$   & $7478.4 \pm \scriptstyle 76.0$  \\
         & Rastrigin (d=800)  & \textbf{$\mathbf{8337.7 \pm \scriptstyle 132.9}$} & $12947.9 \pm \scriptstyle 48.3$          & $15231.6 \pm \scriptstyle 116.9$ & $15231.6 \pm \scriptstyle 116.9$ & $15231.6 \pm \scriptstyle 116.9$ & $15231.6 \pm \scriptstyle 116.9$ \\
        \midrule
        \multirow{8}{*}{\makebox[0pt][r]{\rotatebox[origin=c]{90}{\parbox{2.0cm}{\centering Trajectory \\ Optimization}}}}
         & ant                & $0.032 \pm \scriptstyle 0.080$                    & $\mathbf{0.073 \pm \scriptstyle 0.058}$ & $0.834 \pm \scriptstyle 0.067$   & $0.748 \pm \scriptstyle 0.047$   & $0.649 \pm \scriptstyle 0.101$   & $0.879 \pm \scriptstyle 0.177$   \\
         & halfcheetah        & $\mathbf{0.414 \pm \scriptstyle 0.042}$           & $0.906 \pm \scriptstyle 0.008$           & $0.997 \pm \scriptstyle 0.006$   & $0.924 \pm \scriptstyle 0.024$   & $0.998 \pm \scriptstyle 0.008$   & $0.995 \pm \scriptstyle 0.006$   \\
         & hopper             & $\mathbf{0.623 \pm \scriptstyle 0.007}$           & $0.749 \pm \scriptstyle 0.010$           & $0.924 \pm \scriptstyle 0.008$   & $0.855 \pm \scriptstyle 0.030$   & $0.861 \pm \scriptstyle 0.006$   & $0.929 \pm \scriptstyle 0.010$   \\
         & \makecell[l]{humanoid\\run}        & $\mathbf{0.298 \pm \scriptstyle 0.059}$           & $0.356 \pm \scriptstyle 0.031$           & $0.998 \pm \scriptstyle 0.009$   & $0.928 \pm \scriptstyle 0.083$   & $0.973 \pm \scriptstyle 0.008$   & $0.989 \pm \scriptstyle 0.017$   \\
         & \makecell[l]{humanoid\\standup}    & $\mathbf{0.781 \pm \scriptstyle 0.025}$           & $0.875 \pm \scriptstyle 0.000$           & $0.876 \pm \scriptstyle 0.000$   & $0.883 \pm \scriptstyle 0.002$   & $0.876 \pm \scriptstyle 0.000$   & $0.876 \pm \scriptstyle 0.000$   \\
         & \makecell[l]{humanoid\\track}      & $\mathbf{0.845 \pm \scriptstyle 0.009}$           & $0.914 \pm \scriptstyle 0.013$           & $1.022 \pm \scriptstyle 0.008$   & $1.047 \pm \scriptstyle 0.051$   & $1.015 \pm \scriptstyle 0.002$   & $1.022 \pm \scriptstyle 0.008$   \\
         & pushT              & $\mathbf{0.834 \pm \scriptstyle 0.034}$           & $0.847 \pm \scriptstyle 0.017$           & $1.026 \pm \scriptstyle 0.030$   & $0.937 \pm \scriptstyle 0.024$   & $0.960 \pm \scriptstyle 0.032$   & $1.028 \pm \scriptstyle 0.031$   \\
         & walker2d           & $\mathbf{0.352 \pm \scriptstyle 0.062}$           & $0.756 \pm \scriptstyle 0.011$           & $0.849 \pm \scriptstyle 0.001$   & $0.745 \pm \scriptstyle 0.016$   & $0.850 \pm \scriptstyle 0.001$   & $0.848 \pm \scriptstyle 0.001$   \\
        \bottomrule
    \end{tabular}%
    }
    \caption{Full optimized cost comparison of DIDA against other optimization algorithms on blackbox optimization and trajectory optimization tasks. Results are averaged over multiple runs, with standard deviations reported.}
    \label{tab:full_results_appendix}
\end{table}
In addition to the coverage analysis presented in \cref{fig:checkerboard_coverage}, we further investigate the geometric landscape of the probability distribution learned by the diffusion model. This analysis is performed across various noise levels (timesteps \(t\)) by examining the Hessian of the model's log-probability density function, \(\nabla^2 \log p_t(x)\). The Hessian describes the local curvature of this landscape, offering insights into its structure.

\begin{figure}[ht]
    \centering
    \includegraphics[width=\textwidth]{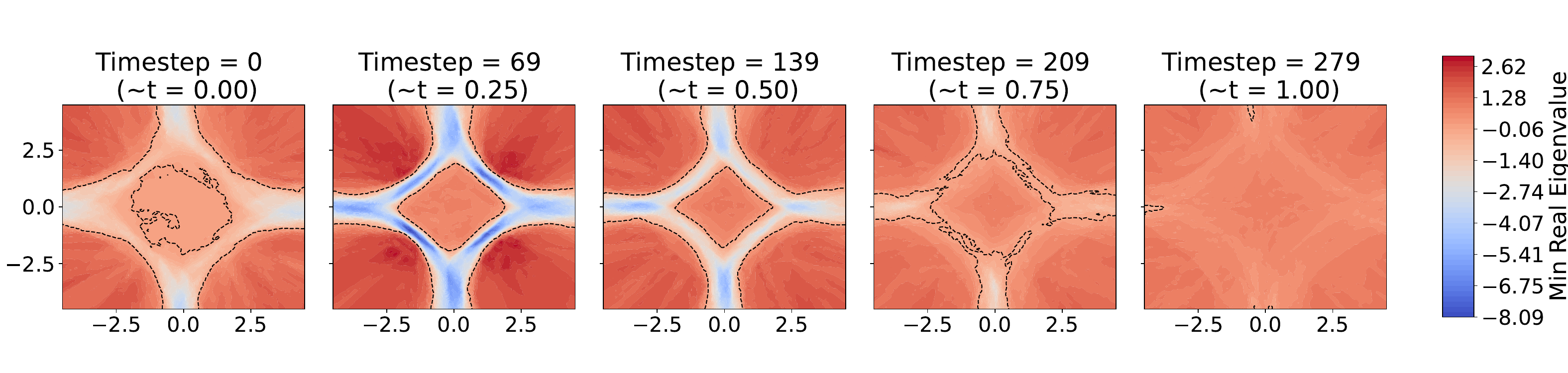}
    \caption{The smoothed landscape of diffusion model with different noise level $t$. When $t$ increases, the landscape becomes smoother.}
    \label{fig:diffusion_landscape}
\end{figure}
The minimum eigenvalue of the Hessian is plotted in \cref{fig:diffusion_landscape} to illustrate the convexity of the landscape. The subtitles in the figure correspond to results at different timesteps \(t\).

\newpage

\end{document}